%% file: Paper_ImagesOnSurfaces.tex
\documentclass[11pt]{article}
\usepackage[left=2.6cm,right=2.6cm,top=3cm,bottom=3cm]{geometry}
\usepackage{graphicx}
\usepackage{cite}
\usepackage[cmex10]{amsmath}
\usepackage{amsthm}
\usepackage{amssymb}
\usepackage{latexsym} 
\usepackage{tikz}
\usetikzlibrary{decorations.markings,shapes,arrows,positioning}

\usepackage{booktabs}
\usepackage{url}
\usepackage{placeins}

\newcommand{\M}{\mathcal{M}}
\newtheorem{definition}{Definition}
\newtheorem{theorem}{Theorem}

\begin{document}

\title{Segmentation and Restoration of Images on Surfaces by Parametric Active Contours with Topology Changes}
\author{Heike~Benninghoff
\thanks{Deutsches Zentrum f\"ur Luft- und Raumfahrt (DLR), 82234 We\ss ling, Germany, email: heike.benninghoff@dlr.de.} 
~and~Harald~Garcke
\thanks{Fakult\"at f\"ur Mathematik, Universit\"at Regensburg, 93040 Regensburg, Germany, email: harald.garcke@ur.de.}}

\date{}

\maketitle

\begin{abstract}
In this article, a new method for segmentation and restoration of images on two-dimensional surfaces is given. Active contour models for image segmentation are extended to images on surfaces. The evolving curves on the surfaces are mathematically described using a parametric approach. For image restoration, a diffusion equation with Neumann boundary conditions is solved in a postprocessing step in the individual regions. Numerical schemes are presented which allow to efficiently compute segmentations and denoised versions of images on surfaces. Also topology changes of the evolving curves are detected and performed using a fast sub-routine. Finally, several experiments are presented where the developed methods are applied on different artificial and real images defined on different surfaces. 
\end{abstract}

\vspace{1cm}

\input{introduction.tex}
\input{segmentation_restoration.tex}

\input{num_approx.tex}
\input{results.tex}
\input{conclusion.tex}

\setlength{\parskip}{0cm}

\bibliographystyle{plain}
\bibliography{literatur}   

\vfill

\end{document}

%% file: introduction.tex
\section{Introduction}
\label{intro}
We study the problem of segmentation and restoration of images defined on two-dimensional surfaces. The Mumford-Shah functional \cite{Mumford89} can be extended and reformulated for image data given on surfaces. Segmentation aims at dividing an image on a surface in characteristic regions, for example in regions of similar gray-value or similar color. The objective of image restoration is to  denoise the original image while preserving sharp edges in the image. 

Active contours \cite{Kass88} originally developed  for planar images can also be used to segment images on surfaces. We let one or more time-dependent curves $\Gamma(t)$, $t\in[0,T]$, evolve in time on a surface $\M \subset \mathbb{R}^3$ according to a flow such that the curves are attracted by edges or region boundaries. 
 
Existing studies and works on evolution of curves on surfaces and on active contours for images on surfaces differ in the way how the surface and the curves are  described mathematically. A geometric scale space for images on parametric surfaces is introduced in \cite{Kimmel97} using level sets. The image is handled implicitly by considering its iso-gray levels. In \cite{Spira07}, Spira and Kimmel consider flows of curves on parametric surfaces and perform edge detection using  a variant of the geodesic active contours method \cite{Caselles97} for images on surfaces.  In \cite{Spira07}, the authors restrict on surfaces which have a global parameterization. Evolution equations are solved with the level set method \cite{OsherSethian88} by considering the pre-image of the curve resulting in a planar curve in the parameterization plane. However, this approach has intrinsic disadvantages since the pre-image of the curve is used. In \cite{Krueger08}, drawbacks of this approach concerning the scaling behavior are discussed. Another drawback is the fact that the method does not allow to incorporate a balloon term \cite{Krueger08}. 

In \cite{Cheng02}, a different method to represent the surface is proposed, where the surface is modeled by the zero level set of a three-dimensional function. The zero level set of a second three-dimensional function, which is time-dependent, is used additionally. The curve is represented by the intersection of the two zero level sets. This approach is used in \cite{Krueger08} and \cite{Zhou13} for image segmentation with geodesic active contours. A drawback of the method is that a one-dimensional curve evolution problem is extended to a three-dimensional problem. To reduce the effort,  a new narrow band technique is given in \cite{Krueger08}. 

In \cite{Tian09}, image segmentation is performed using the Chan-Vese \cite{Chan01} model for images on surfaces with level set methods in combination with a so-called closest point method. One iteration step of the Chan-Vese model in a small 3D neighborhood of the surface is computed followed by an interpolation step.

Flows of the form $V_n = \kappa_\M + F$ are studied in \cite{Mikula06}, where $V_n$ is the velocity contribution normal to the curve, but tangential to the surface, $\kappa_\M$ is the geodesic curvature of the curve, and $F$ is an external forcing term. The authors restrict on graph surfaces and solve a system of partial differential equations for the parameterization $\vec x$, the tangent angle, the geodesic curvature $\kappa_\M$ and for $\|\vec x_\rho\|$. 

In this work, we make use of a concept developed in \cite{BGN10}, where Barrett et al. consider flows of curves on surfaces like geodesic curvature flow, geodesic surface diffusion and geodesic Willmore flow of curves. Here, we apply the methods of \cite{BGN10} to image segmentation applications resulting in a curvature driven flow: The normal velocity is the sum of the geodesic curvature $\kappa_\M$ weighted with a parameter $\sigma$ and an external forcing term. The forcing term is designed for image segmentation based on an extension of the Chan-Vese method to images on surfaces. As in \cite{BGN10}, a parametric approach is used to describe the curve. The surface is given implicitly, as zero level set of a smooth function $\Phi$. However, for practical image segmentation, the function $\Phi$ need not be explicitly known. Our method requires only a normal vector to the surface at each point on the surface. The developed  segmentation technique can handle multiple phases and networks of curves including triple junctions. 

The second image processing task we consider in this article is image restoration. To denoise a given image, we solve a diffusion equation on the surface $\M$. For related works on surface partial differential equations, we refer to \cite{Dziuk13} and the references therein. 
For a practical application see \cite{Cunderlik13}, where data of the Earth's surface captured on board of satellites are filtered by nonlinear diffusion. An implicit representation of the surface is used in \cite{Bertalmio01}, where the surface is embedded as zero level set of a higher dimensional function and the partial differential equations are solved on a fixed Cartesian grid using a special embedding function. 

Methods to solve total variation problems on surfaces are proposed in \cite{Lai11}. In detail, the method of \cite{Rudin1992} for denoising images on surfaces and the method of \cite{Chan06} (a convexified Chan-Vese model) for segmentation of images on surfaces are generalized. A direct, called \emph{intrinsic}, approach is pursued: Lai and Chan \cite{Lai11} perform the resulting calculations directly on the given surface by using differential geometry techniques and finite elements for images given on triangulated surfaces.  They also provide an overview and a comparison of several approaches for variational problems on surfaces (level set methods, parametric surfaces and direct/intrinsic methods). Total variation based image restoration and segmentation are also considered in \cite{Wu2012}, where the authors propose an extension of the augmented Lagrangian method (see e.g. \cite{Wu2010}) for scalar and vectorial total variation problems to images on surfaces. 

In this article, we will perform restoration of images on surfaces by considering an extension of the Mumford-Shah problem \cite{Mumford89}. The image restoration is performed as a postprocessing step after the segmentation. We solve a diffusion equation with Neumann boundary conditions in the already segmented regions. Thus, the image is not smoothed out across the regions boundaries. 

For both segmentation and restoration we present efficient numerical schemes. Further, topology changes of the parametric curves are detected efficiently. In \cite{Benninghoff2014a}, we used and extended a method of \cite{MikulaUrban12} for efficient detection of topology changes, such that also topology changes involving triple junctions can be handled during the curve evolution. We will extend this idea to curves on surfaces to detect several topology changes including splitting and merging of curves, and creation of triple junctions. 

In practical applications, surfaces are often not given as smooth surfaces but in a discrete form, for example as triangulated surfaces. We present all necessary computational aspects when applying the segmentation and restoration models to real data. 

In summing up, we develop a novel scheme for both segmentation and restoration of images defined on surfaces. Using our parametric approach, the evolution of curves is a one-dimensional problem; the postprocessing image diffusion is a two-dimensional problem computed only once after the segmentation has been finished. Compared to other approaches in the literature, where the curve is embedded in a higher-dimensional space, our method is very efficient from a computational point of view. 

%% file: segmentation_restoration.tex
\section{Segmentation and Restoration of Images on Surfaces}
\subsection{Preliminaries}
Let $\mathcal{M} \subset \mathbb R^3$ be a smooth two-dimensional manifold. We assume that we can describe $\mathcal{M}$ by the zero-level set of a function $\Phi$:
\begin{equation}
\mathcal{M} = \left\{ \vec z \in \mathbb{R}^3 \,:\, \Phi(\vec z) = 0 \right\}, 
\end{equation}
where $\Phi \in C^2(\mathbb R^3, \mathbb{R})$ is a function with $\|\nabla\Phi(\vec z)\| > 0$ for $\vec z \in \mathcal{M}$. A unit normal vector field $\vec n_\Phi$ on $\mathcal{M}$ is given by $\vec n_\Phi (\vec z) := \nabla \Phi(\vec z)/\|\nabla\Phi(\vec z)\|$ for $\vec z \in \mathcal{M}$. 

Let $\Gamma \subset \mathcal M$ be a curve on $\mathcal{M}$ parameterized by $\vec x: I \rightarrow \mathcal M$, where $I$ is a one-dimensional reference manifold, for example the unit interval $I= [0,1]$ for open curves or $\mathbb{R}/\mathbb{Z}$ for closed curves. We define $\vec \nu_\Phi: I \rightarrow \mathbb{R}^3$ such that $\vec \nu_\Phi(\rho) := \vec n_\Phi(\vec x(\rho))$ is the surface normal evaluated at $\vec x(\rho)$ for $\rho \in I$. Further, we define $\vec \nu_{\mathcal M}: I \rightarrow \mathbb{R}^3$ by $\vec \nu_\mathcal{M}(\rho) := \vec x_s(\rho) \times \vec \nu_\Phi(\rho)$, where $s$ denotes the arc-length. We  note that $\vec \nu_\mathcal{M}$ is perpendicular to $\vec x_s$, i.e. normal to the curve, but lies in the tangent space to $\mathcal{M}$. Figure \ref{fig:surface_geometry} illustrates a possible surface $\mathcal{M}$, a curve $\Gamma$ and the vector fields $\vec\nu_\Phi$, $\vec x_s$ and $\vec \nu_\M$.

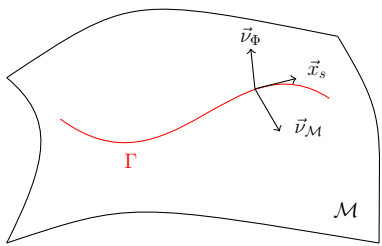
\begin{figure}[t]
\begin{center}
\begin{tikzpicture}[scale=0.55,transform shape]


\draw (0,4) .. controls (3,6) and (3,6) .. (8,5);
\draw (0,0) .. controls (3,1) and (3,1) .. (9,0);
\draw (0,4) .. controls (1.2,3) and (1,2) .. (0,0);
\draw (8,5) .. controls (9.2,3) and (9,3) .. (9,0);

\draw[red] (1.3,3) .. controls (4,1) and (5.5,5) .. (7.8,3.5);

\draw[->] (6,3.73)--(5.9,4.7);
\draw[->] (6,3.73)--(7,3.98);
\draw[->] (6,3.73)--(6.6,2.7);

\draw[red] (3,2) node{\Large $\Gamma$};
\draw (8.2,0.8) node{\Large $\mathcal M$};
\draw (5.9,5) node{\Large $\vec\nu_\Phi$};
\draw (7.5,4.2) node{\Large $\vec x_s$};
\draw (7.3,2.8) node{\Large $\vec \nu_\mathcal{M}$};

\end{tikzpicture}
\end{center}
\caption{Illustration of the vector fields $\vec x_s$, $\vec\nu_\Phi$ and $\vec \nu_\mathcal{M}=\vec x_s \times \vec\nu_\Phi$.}
\label{fig:surface_geometry}
\end{figure}

Further, the vector $\vec x_{ss}$ which is perpendicular to $\vec x_s$ can be written as the sum of its component in $\vec \nu_\Phi$-direction and its component in $\vec \nu_\mathcal{M}$-direction. This motivates the following definition \cite{BGN10}:

\begin{definition}
We define the \emph{geodesic curvature} $\kappa_\mathcal{M}: I \rightarrow \mathbb{R}$ and the \emph{normal curvature} $\kappa_\Phi: I \rightarrow \mathbb{R}$ by
\begin{equation}
\kappa_\M = \vec x_{ss} \,.\, \vec\nu_\mathcal{M}, \quad \kappa_\Phi = \vec x_{ss} \,.\, \vec\nu_\Phi.
\end{equation}
\end{definition}

As a consequence, the vector field $\vec x_{ss}$ can be expressed as 
\begin{equation}
\vec x_{ss} = \kappa_\mathcal{M} \vec\nu_\mathcal{M} + \kappa_\Phi \vec\nu_\Phi.
\label{eq:geodesic_and_normal_curvature}
\end{equation}

\subsection{Active Contours on Surfaces Based on Extensions of the Mumford-Shah and Chan-Vese Models to Images on Surfaces}
Let $u_0:\mathcal{M} \rightarrow [0,1]$ be the intensity function of an image given on a surface $\M$. For segmentation and restoration, we consider an extension of the Mumford-Shah functional \cite{Mumford89}  from the planar case to the case of images and curves on surfaces. For that, we consider the following minimization problem: 

Find a union of curves $\Gamma = \Gamma_1 \cup \ldots \cup \Gamma_{N_C}\subset \M$ and a piecewise smooth  approximation $u:\M \rightarrow \mathbb{R}$ of the original image $u_0$ such that 
\begin{equation}
E^{\mathrm{MS}}(u, \Gamma) = \,\sigma |\Gamma| + \int_{\M\setminus \Gamma} \|\nabla_\M u\|^2 \,\mathrm{d}A + \lambda \int_\M (u_0-u)^2 \,\mathrm{d}A
\label{eq:mumford_shah_geodesic}
\end{equation}
is minimized. Here, $\sigma, \lambda > 0$ are weighting parameters, $|\Gamma|$ denotes the total length of the curves, and  $\nabla_\M u$ is the surface gradient of $u$, also called tangential gradient, cf. \cite{Dziuk13}. Further, $\mathrm{d}A$ denotes the area element and $\|\,.\,\|$ the Euclidean norm. 

We first consider the case of one closed curve $\Gamma$ on $\mathcal{M}$ without any self-intersection which is homotopic to a point. Then, the curve divides $\mathcal{M}$ in two disjoint regions $\Omega_1$ and $\Omega_2$ such that 
\begin{equation*}
\mathcal{M} = \Omega_1 \cup \Gamma \cup \Omega_2.
\end{equation*}
The indices of the regions $\Omega_k$, $k=1,2$, are chosen such that $\vec \nu_\mathcal{M}$ points from $\Omega_2$ to $\Omega_1$. 

For segmentation, we consider a piecewise constant approximation with $u_{|\Omega_k}=c_k \in \mathbb{R}$, $k=1,2$. The Mumford-Shah functional for images on surfaces \eqref{eq:mumford_shah_geodesic} reduces to
\begin{equation}
E(\Gamma, c_1,c_2) = \,\, \sigma |\Gamma| + \lambda \int_{\Omega_1} (u_0-c_1)^2 \,\mathrm{d}A  + \lambda \int_{\Omega_2} (u_0-c_2)^2 \,\mathrm{d}A.
\label{eq:mumford_shah_piecewiese_const_twophase_geodesic}
\end{equation}

Similar to this functional, we can consider the analogue of the planar Chan-Vese model \cite{Chan01} for images on surfaces: 
\begin{equation}
E(\Gamma, c_1, c_2) = \,\, \sigma |\Gamma|  +  \mu \int_{\Omega_1} 1 \,\mathrm{d}A + \lambda_1  \int_{\Omega_1} f_1 \,\mathrm{d}A + \lambda_2 \int_{\Omega_2} f_2 \,\mathrm{d}A,
\label{eq:functional_chanvese_geodesic}
\end{equation}
where $\sigma, \lambda_1, \lambda_2 > 0$, $\mu \geq 0$ are weighting parameters. Similar to the planar Chan-Vese method \cite{Chan01}, the function $f_k$, $k=1,2$, is defined by 
\begin{equation}
f_k(\vec z) = (u_0(\vec z) - c_k)^2, \quad \vec z \in \overline{\Omega_k} \subset\mathcal{M}.
\end{equation}
The model \eqref{eq:functional_chanvese_geodesic} with $\mu=0$ and $\lambda_1=\lambda_2$ is the piecewise constant case \eqref{eq:mumford_shah_piecewiese_const_twophase_geodesic} of the Mumford-Shah model. 

Fixing now the curve $\Gamma$ in \eqref{eq:functional_chanvese_geodesic}, we obtain the following condition for the coefficients $c_k$, $k=1,2$:
\begin{equation}
c_k = \frac{\int_{\Omega_k} u_0 \,\mathrm{d}A}{\int_{\Omega_k} 1 \,\mathrm{d}A}.
\end{equation}
 
Let $\vec x: I = \mathbb{R}/\mathbb{Z} \rightarrow \mathcal M$ be a smooth parameterization of $\Gamma$. 

We now derive a flow for image segmentation as gradient flow using methods of the theory of calculus of variations. Since we consider curves on the surface $\M$, the variations are restricted to lie on $\mathcal M$. Therefore, as in \cite{BGN10}, we consider the variations to be elements of 
\begin{equation}
\underline{V}_\Phi = \left\{\vec\eta : I \rightarrow \mathbb{R}^3 \,:\, \vec\eta \,\text{ is smooth and } \, \vec\eta \,.\, \vec\nu_\phi = 0\right\}
\end{equation}
and define for functions $\vec\eta, \vec\chi : I \rightarrow \mathbb{R}^3$ the following inner product
\begin{equation}
(\vec\eta,\vec\chi)_{2,\mathcal{M},nor} := \int_\Gamma \vec P_\mathcal{M} \vec\eta \,.\, \vec P_\mathcal{M}\vec\chi \,\mathrm{d}s,
\end{equation}
where $\vec P_\mathcal{M}$ is the projection onto the part in direction $\vec\nu_\mathcal{M}$, i.e. $\vec P_\mathcal{M} \vec\eta = (\vec \eta \,.\, \vec\nu_\mathcal{M})\vec\nu_\mathcal{M}$ for $\vec\eta : I \rightarrow \mathbb{R}^3$.

Fixing the coefficients $c_1, c_2$, we consider for $\vec\eta \in \underline{V}_\Phi$ a variation $\vec y: I \times (-\epsilon_0, \epsilon_0) \rightarrow \mathcal M$  with $\vec y(\rho,0)=\vec x(\rho)$  and $\vec y_\epsilon(\rho,0)=\vec \eta(\rho)$.  

Let $\Gamma^\epsilon \subset \mathcal M$ be the image of $\vec y(\,.\,,\epsilon)$, let $\Omega_1^\epsilon, \Omega_2^\epsilon \subset \mathcal M$ be regions such that $\mathcal{M} = \Omega_1^\epsilon \cup \Gamma^\epsilon \cup \Omega_2^\epsilon$. 

Further, let $\vec\nu_\mathcal{M}^\epsilon(\rho) \in T_{\vec y(\rho,\epsilon)}\mathcal M$ be a vector in the tangent space to $\M$ in $\vec y(\rho,\epsilon)$  defined by $\vec\nu_\mathcal{M}^\epsilon(\rho) = \vec y_s(\rho,\epsilon) \times \vec n_\Phi(\vec y(\rho,\epsilon))$. The vector field $\vec\nu_\mathcal{M}^\epsilon(\rho)$ points in the direction $\Omega_1^\epsilon$. The vector $\vec\nu_\mathcal{M}^\epsilon(\rho)$ is normal to $\Gamma^\epsilon$, but lies in the tangent space of $\M$ at the point $\vec y(\rho,\epsilon)$.  

We define 
\begin{equation*}
(\delta E(\Gamma))(\vec\eta) := \left.\frac{\mathrm{d}}{\mathrm{d}\epsilon}\right|_{\epsilon=0} \left( \sigma \int_{\Gamma^\epsilon} 1 \,\mathrm{d}s  +  \mu \int_{\Omega_1^\epsilon} 1 \,\mathrm{d}A +  \lambda_1  \int_{\Omega_1^\epsilon} f_1 \,\mathrm{d}A + \lambda_2 \int_{\Omega_2^\epsilon} f_2 \,\mathrm{d}A\right)
\end{equation*}
on noting that $\vec y(\,.\,,\epsilon)$ and thus $\Gamma^\epsilon$, $\Omega_1^\epsilon$ and $\Omega_2^\epsilon$ depend on $\vec \eta$. We compute
\begin{align}
(\delta  E(\Gamma))(\vec\eta) &= \left.\frac{\mathrm{d}}{\mathrm{d}\epsilon}\right|_{\epsilon=0} \left(\sigma \int_I \|\vec y_\rho\| \, \mathrm{d}\rho + \mu \int_{\Omega_1^\epsilon} 1 \,\mathrm{d}A + \lambda_1 \int_{\Omega_1^\epsilon} f_1 \,\mathrm{d}A + \lambda_2 \int_{\Omega_2^\epsilon} f_2 \,\mathrm{d}A\right) \nonumber \\
&=  \left(\sigma \int_I \frac{\vec y_\rho}{\|\vec y_\rho\|} \,.\, \vec y_{\rho\epsilon} \,\mathrm{d}\rho + \mu \int_I (- \vec y_\epsilon \,.\, \vec\nu_\mathcal{M}^\epsilon) \|\vec y_\rho\| \,\mathrm{d}\rho \,+ \lambda_1 \int_I f_1(\vec y) (- \vec y_\epsilon \,.\, \vec\nu_\mathcal{M}^\epsilon) \|\vec y_\rho\| \,\mathrm{d}\rho \, + \right.\nonumber \\
&\quad \left.\left.+\lambda_2   \int_I f_2(\vec y) ( \vec y_\epsilon \,.\, \vec\nu_\mathcal{M}^\epsilon) \|\vec y_\rho\| \,\mathrm{d}\rho\right) \right|_{\epsilon=0}\nonumber \\
&= \sigma \int_\Gamma \vec x_s \,.\, \vec \eta_s \,\mathrm{d}s +  \int_\Gamma \left(-\mu - \lambda_1 f_1 + \lambda_2 f_2 \right) \vec\nu_\mathcal{M} \,.\,\vec\eta \,\mathrm{d}s \nonumber\\
& = \int_\Gamma \left(-\sigma\vec x_{ss} + \left(-\mu - \lambda_1 f_1 + \lambda_2 f_2 \right) \vec\nu_\mathcal{M}\right) \,.\,\vec\eta \,\mathrm{d}s \nonumber \\
& = \int_\Gamma \left(-\sigma \kappa_\mathcal{M} - \mu - \lambda_1 f_1 + \lambda_2 f_2 \right) \vec\nu_\mathcal{M} \,.\, \vec \eta \, \mathrm{d}s.
\label{eq:dE}
\end{align}
Here, we used a transport theorem for curves on surfaces  \cite{GarckeWieland06}. We applied integration by parts for the second last identity. The last identity follows from \eqref{eq:geodesic_and_normal_curvature} and $\vec\eta \,.\, \vec\nu_\phi=0$. 

A time-dependent function $\vec x: I \times [0,T] \rightarrow \mathcal{M}$ with $\vec x_t(\,.\,,t) \in \underline{V}_\Phi$ is called a solution to the gradient flow equation if 
\begin{equation}
(\vec x_t, \vec\eta)_{2,\mathcal{M},nor} = - \left(\delta E(\Gamma)\right)(\vec\eta) 
\label{eq:definition_gradientflow_surface}
\end{equation}
holds for all $\vec\eta \in \underline{V}_\Phi$. 

Let $\eta \in \underline{V}_\Phi$. We conclude from \eqref{eq:dE} and \eqref{eq:definition_gradientflow_surface} on noting that $\vec\nu_\M \,.\, \vec \eta = \vec\nu_\M \,.\, \vec P_\M \vec \eta$
\begin{equation}
\vec P_\mathcal{M} \vec x_t = -\left(-\sigma \kappa_\mathcal{M} - \mu - \lambda_1 f_1 + \lambda_2 f_2 \right) \vec\nu_\mathcal{M}.
\end{equation}
Let $V_n := \vec x_t \,.\, \vec\nu_\M$ denote the velocity in direction $\vec \nu_\M$, also called \emph{normal velocity}. Then $\vec P_\mathcal{M} \vec x_t = V_n \vec\nu_\mathcal{M}$ and consequently the equation above leads to
\begin{equation}
V_n = \sigma \kappa_\mathcal{M} + F,
\label{eq:Vn_ChanVese_Surfaces}
\end{equation}
where $F$ is given by 
\begin{equation}
F(\vec z) = \mu + \lambda_1 f_1(\vec z) - \lambda_2 f_2(\vec z) = \mu + \lambda_1 (u_0(\vec z)-c_1)^2 - \lambda_2 (u_0(\vec z)-c_2)^2. \label{eq:external_force_term}
\end{equation}

We rewrite equation \eqref{eq:Vn_ChanVese_Surfaces} to a scheme for $\vec x: I  \times [0,T] \rightarrow \mathbb R^3$ and $\kappa_\M, \kappa_\Phi:  I  \times [0,T] \rightarrow \mathbb R$. We assume that $\vec x(\rho,0)$ lies on $\M$. To force the curve to stay on the manifold $\M$, the velocity in direction normal to the surface $\vec x_t \,.\,\vec\nu_\Phi$ must be zero (i.e. $\vec x_t \in \underline{V}_\Phi$). We thus have the following scheme: 

Let $\vec x(I,0) = \Gamma(0) \subset \M$. For $t \in (0,T]$, find  $\vec x(\,.\,,t): I  \rightarrow \mathbb R^3$ and $\kappa_\M(\,.\,,t), \kappa_\Phi(\,.\,,t):  I   \rightarrow \mathbb R$  such that 
\begin{subequations}
\begin{align}
\vec x_t \,.\, \vec \nu_\M \,=& \,\,\sigma \kappa_\M + F, \\
\vec x_t \,.\, \vec \nu_\Phi \,=& \,\,0,  \\
\vec x_{ss} \,=& \,\,\kappa_\M \vec\nu_\M + \kappa_\Phi \vec\nu_\Phi.
\end{align}
\end{subequations}

\subsection{Multiphase Image Segmentation with Possible Triple Junctions}
We extend the above presented two-phase segmentation with a single closed curve to more general situations. We consider a curve network with closed and open curves which partition the image domain in $N_R$ regions. Also triple junctions can occur. 

Therefore, we consider a decomposition of $\M$ in time-depen\-dent regions $\Omega_1(t), \ldots,\Omega_{N_R}(t)$, $t\in [0,T]$, separated by curves $\Gamma_1(t), \ldots, \Gamma_{N_C}(t)$. Each curve is parameterized by a time-dependent function $\vec x_i(\,.\,,t): I_i \rightarrow \mathbb{R}^3$, where $I_i$ is a one-dimensional reference manifold for $i=1, \ldots, N_C$. Similar to the case of one curve, we set $\vec \nu_{\Phi,i} = \vec n_\Phi \circ \vec x_i$ and $\vec \nu_{\M,i} = (\vec x_i)_s \times \vec\nu_{\Phi,i}$. The geodesic curvature $\kappa_{\M,i}$ and the normal curvature $\kappa_{\Phi,i}$ are given by $\kappa_{\M,i} = (\vec x_i)_{ss} \,.\, \vec \nu_{\M,i}$ and $\kappa_{\Phi,i} = (\vec x_i)_{ss} \,.\, \vec \nu_{\Phi,i}$. All quantities are time-dependent. 

We define a piecewise constant image approximation by $u(\,.\,,t) = \sum_{k=1}^{N_R} c_k(t) \chi_{\Omega_k(t)}$, where $\chi_{\Omega_k(t)}$ is the characteristic function on the set $\Omega_{k}(t)$ and the coefficients $c_k(t)$ are computed by 
\begin{equation}
c_k(t) = \frac{\int_{\Omega_k(t)} u_0 \,\mathrm{d}A}{\int_{\Omega_k(t)} 1 \,\mathrm{d}A},
\label{eq:geodesic_coefficients_multiphase}
\end{equation}
i.e. they are set to the mean of $u_0$ in $\Omega_k(t)$. 

Let $\vec x_i(\,.\,,0)$, $i=1, \ldots, N_C$, be parameterizations of given curves $\Gamma_i(0)\subset\M$. We have to solve the following scheme for $t\in(0,T]$: Find $\vec x_i(\,.\,,t):I_i \rightarrow \mathbb{R}^3$, $\kappa_{\M,i}(\,.\,,t), \kappa_{\Phi,i}(\,.\,,t): I_i \rightarrow \mathbb{R}$ such that
\begin{subequations}
\label{eq:scheme_strong_geodesic}
\begin{align}
(\vec x_i)_t \,.\, \vec \nu_{\M,i} \,=& \,\, \sigma \kappa_{\M,i} + F_i, \label{eq:scheme_strong_geodesic_1} \\
(\vec x_i)_t \,.\, \vec \nu_{\Phi,i} \,=& \,\, 0,  \label{eq:scheme_strong_geodesic_2} \\
(\vec x_i)_{ss} \,=& \,\, \kappa_{\M,i} \,\vec\nu_{\M,i} + \kappa_{\Phi,i}\, \vec\nu_{\Phi,i},
\end{align}
\end{subequations}
hold for $i=1, \ldots, N_C$. The external force $F_i$ is defined for $\vec x \in \M$ by
\begin{equation}
F_i(\vec x) = \mu + \lambda_{k^+(i)} (u_0(\vec x) - c_{k^+(i)})^2   - \lambda_{k^-(i)} (u_0(\vec x) - c_{k^-(i)})^2,
\label{eq:external_forcing_term_rac_geodesic0}
\end{equation}
where $k^+(i), k^-(i) \in \{1, \ldots, N_R\}$ are indices of two regions, such that $\vec \nu_{\M, i}$ points from $\Omega_{k^-(i)}$ to $\Omega_{k^+(i)}$. 

In the experiments, presented in Section \ref{sec:results}, we always consider the case $\mu=0$ and $\lambda_k = \lambda$ for all $k=1, \ldots, N_R$, i.e. all segmentations in our demonstrations can be performed with only two weighting parameters $\sigma$ and $\lambda$. The external forcing term is 
\begin{equation}
F_i(\vec x) = \lambda [(u_0(\vec x) - c_{k^+(i)})^2   -  (u_0(\vec x) - c_{k^-(i)})^2].
\label{eq:external_forcing_term_rac_geodesic}
\end{equation}

We also allow open curves, i.e. curves with $\partial \Gamma_i(t) \neq \emptyset$. Since we consider interface curves, i.e. each curve $\Gamma_i(t)$ separates two different regions $\Omega_{k^+(i)}$ and $\Omega_{k^-(i)}$, we exclude free endpoints. This means, we exclude the case that a curve ends in $\M$ without meeting another curve at its endpoint. Further, we consider only smooth, compact surfaces $\M$ without boundary. Thus, endpoints of open curves are therefore part of triple junctions denoted with $\vec\Lambda_k \in \M$, $k=1, \ldots, N_T$. 

For each $k \in \left\{1, \ldots, N_T\right\}$ let the integers $i_{k,1}, i_{k,2}, i_{k,3}$ $\in \{1, \ldots, N_C\}$ denote the indices of curves $\Gamma_{i_{k,l}}$, $l=1,2,3$, $i_{k,1}\neq i_{k,2}\neq i_{k,3} \neq i_{k,1}$ with parameterizations $\vec x_{i_{k,l}} : I_{i_{k,l}} = [0,1] \rightarrow \M$, such that 
\begin{equation*}
\vec{x}_{i_{k,1}}(\rho_{k,1}) = \vec{x}_{i_{k,2}}(\rho_{k,2})= \vec{x}_{i_{k,3}}(\rho_{k,3}) =\vec \Lambda_k,
\end{equation*}
where $\rho_{k,l} \in \{0,1\}$ corresponds to the start or endpoint of the curve $i_{k,l}$, $l=1,2,3$.

At the triple junctions $\vec \Lambda_k$, $k=1, \ldots, N_T$,  an attachment condition and Young's law need to hold: 
\begin{subequations}
\label{eq:triple_junction_cond_geodesic}
\begin{align}
& \text{the triple junction $\vec \Lambda_k$ does not pull apart}, \label{eq:attachment_cond_tj_geodesic}\\
& \sum_{l=1}^3 (-1)^{\rho_{k,l}} \,\vec\tau_{i_{k,l}} = 0, \label{eq:youngs_law_geodesic}
\end{align}
\end{subequations}
where $\vec\tau_{i_{k,l}} := (\vec x_{i_{k,l}})_s$ is a tangent vector field at $\Gamma_{i_{k,l}} \subset \M$, $l=1,2,3$. We refer to \cite{BGN07a} for the planar case of evolution of curves with triple junctions.  

\subsection{Vector-valued Images}
The segmentation method can be easily extended to vector-valued images such as color images. Only the external force $F$ need to be adapted compared to scalar images. The adaptations of $F$ can be done similarly as for planar images \cite{Chan00}, \cite{Benninghoff2014a}. 

Let $\vec u_0 =(u_{0,1},u_{0,2},u_{0,3}): \M \rightarrow \mathbb{R}^3$ represent a color image in the RGB color space. The three components of the vector-valued image function represent the red, green and blue color channel. The external force $F_i$, $i=1,\ldots,N_C$, in \eqref{eq:external_forcing_term_rac_geodesic} has to be modified. We therefore set:
\begin{equation*} 
F_i = \sum_{j=1}^3 \lambda_j \left[ (u_{0,j}-c_{k^+(i),j})^2 - (u_{0,j}-c_{k^-(i),j})^2 \right], 
\end{equation*}
where $\lambda_j$ weights the $j$-th color component, $j=1,2,3$. The coefficients $\vec c_k =(c_{k,1},c_{k,2},c_{k,3})$, $k=1,\ldots, N_R$, are given by setting $\vec c_{k,j}$ to the mean of $u_{0,j}$ in $\Omega_k$. 

Another color space, we will use for segmentation of color images, is the chromaticity-brightness space. Let $\vec v_0 = \vec u_0 / \| \vec u_0\|$ be the chromaticity and $b_0 = \|\vec u_0\|$ the brightness of an image with image function $\vec u_0$. We modify $F_i$ to
\begin{equation*}
F_i = \lambda_C \left[ \|\vec v_0-\vec v_{k^+(i)}\|^2 - \|\vec v_0-\vec v_{k^-(i)}\|^2 \right] +  \lambda_B \left[(b_0-b_{k^+(i)})^2 - (b_0-b_{k^-(i)})^2 \right],
\end{equation*}
where $\lambda_C$ weights the chromaticity term and $\lambda_B$ weights the brightness term. Here, $\vec v_{k}$ is a normalized mean of $\vec v_0$ in the region $\Omega_k$ (see \cite{Benninghoff2014a} for details) and $b_{k}$ is the mean of $b_0$ in $\Omega_k$.

\subsection{Restoration of Images on Surfaces with Edge Enhancement}
\label{subsec:image_restoration}
Searching for a minimizer of the Mumford-Shah functional for images on surfaces \eqref{eq:mumford_shah_geodesic}  involves  both a set of curves $\Gamma$ and  an image approximation $u$. The segmentation technique presented above uses piecewise constant image approximations to divide an image into characteristic regions of similar image intensity or color. For image restoration, more details of the original image $u_0$ should be kept; the piecewise constant approximation would be a too large simplification. Therefore, we aim at approximating $u_0: \M \to \mathbb{R}$ by a piecewise smooth function $u: \M \to \mathbb{R}$. 

We propose to first perform a segmentation of the image with piecewise constant approximations by solving the evolution equations \eqref{eq:scheme_strong_geodesic} with \eqref{eq:triple_junction_cond_geodesic} in case of triple junctions. This is followed by a restoration using the already identified regions. By denoising the image in this way during a postprocessing step, the edges in the image will not be smoothed out if the curves $\Gamma$ match with these edges. 

We thus consider piecewise smooth approximations $u_{|\Omega_k} = u_k$, where $u_k : \Omega_k \rightarrow [0,1]$ is a smooth function defined on the region $\Omega_k \subset \M$, $k=1, \ldots, N_R$. For smoothing the image in the regions $\Omega_k$, we derive surface partial differential equations from the Mumford-Shah functional. Since the curve set $\Gamma$ has already been determined, we can fix $\Gamma$ in the functional \eqref{eq:mumford_shah_geodesic} and consider variations of $u$ of the form $u+\epsilon v$, for $v:\M \rightarrow \mathbb{R}$ and $\epsilon > 0$. We compute
\begin{align*}
\left.\frac{\mathrm{d}}{\mathrm{d}\epsilon}\right|_{\epsilon=0} E^{\mathrm{MS}}(u+\epsilon v, \Gamma) &= \lim_{\epsilon \rightarrow 0} \frac{1}{\epsilon} (E^{\mathrm{MS}}(u+\epsilon v, \Gamma) - E^{\mathrm{MS}}(u, \Gamma))\\
&= \lim_{\epsilon \rightarrow 0} \frac{1}{\epsilon} \left(\int_{\M \setminus \Gamma} \left(2 \epsilon \nabla_\M u \,.\, \nabla_\M v + \epsilon^2 \|\nabla_\M v \|^2 \right) \,\mathrm{d}A  \right. \\
&\quad \left. + \lambda \int_\M  \left(2 \epsilon (u-u_0)\,v  + \epsilon^2 v^2\right)\,\mathrm{d}A \right) \\
&= 2\int_{\M \setminus \Gamma} \nabla_\M u \,.\,\nabla_\M v \,\mathrm{d}A + 2 \lambda \int_\M (u-u_0)\,v\,\mathrm{d}A \\
&= 2 \sum_{k=1}^{N_R} \int_{\Omega_k} \left(\nabla_\M u_k \,.\,\nabla_\M v + \lambda (u_k-u_0)\,v\right)\,\mathrm{d}A.
\end{align*}
For a stationary solution, we search for a function $u$ satisfying $0 = \left.\frac{\mathrm{d}}{\mathrm{d}\epsilon}\right|_{\epsilon=0} E^{\mathrm{MS}}(u+\epsilon v, \Gamma)$. This leads to 
\begin{equation*}
0 = \int_{\Omega_k}\left(\nabla_\M u_k \,.\,\nabla_\M v + \lambda (u_k-u_0)\,v\right)\,\mathrm{d}A 
\end{equation*}
for each $k=1,\ldots,N_R$ and an arbitrary function $v$ which is smooth on $\Omega_k$. Using an integration by parts formula \cite{Dziuk13}, we obtain
\begin{equation}
0 = \int_{\Omega_k} \left(-\Delta_\M u_k + \lambda (u_k-u_0)\right)\,v\,\mathrm{d}A  + \int_{\partial\Omega_k} \nabla_\M u_k \,.\, \vec\mu_k \,v\,\mathrm{d}s 
\end{equation}
where $\Delta_\M$ is the Laplace-Beltrami operator and $\vec\mu_k(\vec p)$ is a unit outer normal vector on $\partial \Omega_k$ in $T_{\vec p}\M$, i.e. it is tangent to the surface for each $\vec p\in \partial\Omega_k \subset M$ but normal to $\partial\Omega_k$ in $\vec p$. Since $\M$ is a smooth, compact surface without boundary, the boundary $\partial \Omega_k$ of the region $\Omega_k$ consists of one or more curves $\Gamma_i$, $i\in\{1,\ldots,N_C\}$. Thus, locally, $\vec\mu_k$ is $\pm\vec\nu_{\M,i}$. Since $v$ is arbitrary chosen, we have to solve the following surface partial differential equation with Neumann boundary condition for $k=1,\ldots, N_R$:
\begin{subequations}
\label{eq:diffusion_eq_geodesic_strong}
\begin{align}
-\frac{1}{\lambda}\Delta_\M u_k + u_k =& u_0, && \text{in} \,\Omega_k,\\
\nabla_\M u_k \,.\, \vec\mu_k =& 0, && \text{on} \,\partial \Omega_k.
\end{align}
\end{subequations}
The smoothing effect is due to the Laplace-Beltrami operator. We can control the smoothing extent using the weighting parameter $\lambda>0$. The smaller $\lambda$, the larger is the denoising. The larger $\lambda$, the closer is the approximation to the original image. The Neumann boundary condition provides that edges in the image are not smoothed out. 

Vector-valued images can be smoothed by considering each component individually like a scalar image.

%% file: num_approx.tex
\section{Numerical Approximation}

\subsection{Finite Element Approximation of the Image Segmentation Scheme}
\label{subsec:geodesic_fe_appr}
We introduce a finite element approximation for the scheme \eqref{eq:scheme_strong_geodesic} with \eqref{eq:triple_junction_cond_geodesic} in the case of triple junctions. The evolution equations \eqref{eq:scheme_strong_geodesic_1}, for $i=1,\ldots,N_C$, can be interpreted as a weighted geodesic curvature flow with  external forcing terms. Therefore, we make use of the finite element scheme for geodesic curvature flow developed in \cite{BGN10} for closed curves and generalize the approach for possible open curves with triple junctions and for image segmentation problems.

In order to formulate a finite element scheme, we first introduce a spatial and time discretization, introduce discrete function spaces and discrete inner products. 

For $i=1, \ldots, N_C$, let $0=q_0^i < q_1^i < \ldots < q_{N_i}^i = 1$ be a decomposition of the interval $I_i = I = [0,1]$. If $\Gamma_i$ is a closed curve, we make use of the periodicity $N_i=0$, $N_i+1=1$, $-1=N_i-1$, etc.

We introduce the following discrete function spaces
\begin{subequations}
\label{eq:def_function_spaces_2}
\begin{align}
W^h &:= \left\{ (\eta_1, \ldots,\eta_{N_C}) \in \left[C(I,\mathbb{R})\right]^{N_C}  \,:\, \eta_i|_{[q_{j-1}^i, q_j^i]} \,\text{ is  linear, }\,\,\forall i=1,\ldots,N_C, \,j=1, \ldots, N_i\right\},\\
\underline{V}^h 
&:= \left\{ (\vec\eta_1, \ldots,\vec\eta_{N_C}) \in \left[C(I,\mathbb{R}^3)\right]^{N_C} \,:\,\vec\eta_{i_{k,1}}(\rho_{k,1})= \vec\eta_{i_{k,2}}(\rho_{k,2}) = \vec\eta_{i_{k,3}}(\rho_{k,3}), \, \forall k=1,\ldots,N_T, \right. \nonumber\\
& \quad \quad \left. \vec\eta_i|_{[q_{j-1}^i, q_j^i]} \,\text{ is linear, }\,\,\forall i=1,\ldots,N_C, \,j=1, \ldots, N_i\right\}. 
\end{align}
\end{subequations} 

The attachment conditions for triple junctions are incorporated in the definition of the space $\underline{V}^h$.

A basis of $W^h$ is given by functions $\chi_{i,j}:=((\chi_{i,j})_1,$ $\ldots, (\chi_{i,j})_{N_C}) \in W^h$, where $(\chi_{i,j})_k(q_l^k) :=\delta_{ik} \delta_{jl}$ for $i,k= 1, \ldots, N_C$, $j=j_0^i, \ldots, N_i$, $l=j_0^k, \ldots, N_k$, where $j_0^i = 1$ if $\Gamma_i$ is closed and $j_0^i=0$ else. 

Further, let $0=t_0 < t_1 < \ldots < t_M = T$ be a partitioning of the time interval $[0,T]$ into possibly variable time steps $\tau_m := t_{m+1}-t_m$, $m=0, \ldots, M-1$. Let $\vec X^m = (\vec X_1^m, \ldots, \vec X_{N_C}^m) \in \underline{V}^h$ be an approximation of $\vec x(\,.\,,t_m) = (\vec x_1(\,.\,,t_m), \ldots,$ $\vec x_{N_C}(\,.\,,t_m))$. Let $\Gamma^m = (\Gamma_1^m, \ldots, \Gamma_{N_C}^m)$ denote the image of $\vec X^m$. In case of triple junctions let $j_{k,l} \in \{0, N_{i_{k,l}}\}$ denote the index of the corresponding curve endpoint such that $q_{j_{k,l}}^{i_{k,l}}=\rho_{k,l}$, $k=1, \ldots, N_T$, $l=1,2,3$. 

For scalar or vector functions $u=(u_1,\ldots,u_{N_C}),v=(v_1,\ldots,v_{N_C}) \in \left[L^2(I, \mathbb{R}^{(3)})\right]^{N_C}$, the $L^2$-inner product $\langle \cdot\,,\,\cdot\rangle_m$ over the current polygonal curve network $\Gamma^m$ is given by
\begin{equation}
\langle u, v\rangle_m := \int_{\Gamma^m} u\,.\,v \,\mathrm{d}s := \sum_{i=1}^{N_C} \int_{I_i} u_i\,.\,v_i\, \|(\vec X_i^m)_\rho\|\,\mathrm{d}\rho.
\label{eq:def_product_fe1_geodesic}
\end{equation}
We follow the ideas of \cite{BGN07a}, \cite{BGN10}, and define a mass lumped inner product $\langle \cdot\,,\,\cdot\rangle_m^h$ for piecewise continuous functions $u=(u_1, \ldots,$ $u_{N_C})$ and $v=(v_1, \ldots, v_{N_C})$ by 
\begin{align}
\langle u,v\rangle_m^h :=& \frac12 \sum_{i=1}^{N_C} \sum_{j=1}^{N_i} h_{i,j-\frac12}^m \left[ \left(u_i\,.\,v_i\right)\left([q_j^i]^-\right) + \right. \nonumber \\
& \left. +\left(u_i\,.\,v_i\right)\left([q_{j-1}^i]^+\right)\right],
\label{eq:def_product_fe2_geodesic}
\end{align}
where $u_i([q_j^i]^\pm):= \mathrm{lim}_{\epsilon \rightarrow 0, \epsilon > 0} u_i(q_j^i \pm \epsilon)$ and $h_{i,j-\frac12}^m := \| \vec X_i^m(q_j^i)- \vec X_i^m(q_{j-1}^i)\|>0$ is the distance between two neighbor nodes.

Let $h^m := \max_{i=1,\ldots,N_C, j=1,\ldots, N_i} h_{i,j-\frac12}^m$ denote the maximum distance between two neighbor nodes of the polygonal curves. Let $\vec X^m \in \underline{V}^h$ be a given parameterization of the polygonal curve network $\Gamma^m$ satisfying the following assumption:

\vspace{2ex}
$(\mathcal A)$ The distance between two neighbor nodes of $\Gamma^m$ is positive, i.e. $h_{i,j-\frac12}^m > 0$ for $i=1, \ldots, N_C$, $j=1,\ldots, N_i$ and $ \vec X_i^m(q_{j+1}^i) \neq \vec X_i^m(q_{j-1}^i)$ for $i=1, \ldots, N_C$ and $j=1, \ldots, N_i$ if $\partial \Gamma_i^m = \emptyset$ and $j=1, \ldots, N_i-1$ if $\partial \Gamma_i^m \neq \emptyset$. 

\vspace{2ex}
We define $\vec \omega_\Phi^m = (\vec \omega^m_{\Phi,1}, \ldots, \vec\omega^m_{\Phi,N_C})$, by 
\begin{equation*}
\vec\omega^m_{\Phi,i}(q_j^i) = \vec n_{\Phi}(\vec X_i^m(q_j^i)) = \frac{\nabla \Phi(\vec X_i^m(q_j^i))}{\|\nabla \Phi(\vec X_i^m(q_j^i))\|}, 
\end{equation*}
for $i=1, \ldots, N_C$ and $j=j_0^i,\ldots, N_i$, i.e. $\vec\omega^m_{\Phi,i}$ approximates $\vec\nu_{\Phi,i}$ at time $t_m$. Further, the tangent vector field is approximated by 
$\vec \omega^m_d = (\vec\omega^m_{d,1}, \ldots, \vec\omega^m_{d,N_C})$: We set
\begin{equation*}
\vec\omega^m_{d,i}(q_j^i) = \frac{ \vec X_i^m(q_{j+1}^i) - \vec X_i^m(q_{j-1}^i)}{ \| \vec X_i^m(q_{j+1}^i) - \vec X_i^m(q_{j-1}^i)\|}
\end{equation*}
if $\Gamma_i^m$ is closed, or if $\Gamma_i^m$ is an open curve and $j\neq 0, N_i$. For closed curves, we make use of the periodicity $N_i = 0$, $N_i+1 = 1$ and $-1 = N_i-1$.  For the endpoints of an open curve, we define
\begin{equation*}
\vec\omega^m_{d,i}(q_0^i)       = \frac{\vec X_i^m(q_{1}^i) - \vec X_i^m(q_{0}^i)}{ \| \vec X_i^m(q_{1}^i) - \vec X_i^m(q_{0}^i)\|}, \quad 
\vec\omega^m_{d,i}(q_{N_i}^{i}) = \frac{ \vec X_i^m(q_{N_i}^i) - \vec X_i^m(q_{N_i-1}^i)}{ \| \vec X_i^m(q_{N_i}^i) - \vec X_i^m(q_{N_i-1}^i)\|}.
\end{equation*}
Furthermore, we define $\vec \omega^m_\M = (\vec\omega^m_{\M,1}, \ldots, \vec\omega^m_{\M,N_C})$ by 
\begin{equation*}
\vec\omega^m_{\M,i}(q_j^i) = \vec \omega^m_{d,i}(q_j^i) \times \vec\omega^m_{\Phi,i}(q_j^i),
\end{equation*}
Thus, $\vec \omega^m_{\M,i}$ approximates $\vec \nu_{\M,i}$ at time $t_m$. 

The assumption $(\mathcal A)$ is necessary, such that $\vec\omega^m_d$ is well-defined, see also  \cite{BGN10}. 

We define a discrete analogue to the space $\underline{V}_\Phi$ by 
\begin{equation}
\underline{V}^h_\Phi = \left\{ \vec\eta \in \underline{V}^h \,:\, \vec\eta_i \,.\, \vec\omega_{\Phi,i}^m = 0, \quad i=1, \ldots, N_C\right\}.
\label{eq:definition_VhPhi}
\end{equation}

We now propose the following discrete scheme: Let $\vec X^0\in \underline{V}^h$ be a given parameterization of a polygonal curve network $\Gamma^0$. We assume that the initial nodes $\vec X_i^0(q_j^i)$ lie on the surface $\M$. Further, we assume that  assumption $(\mathcal A)$ holds for $\vec X^m$, $m=0, \ldots, M-1$.

For $m=0, \ldots, M-1$, find $\delta \vec X^{m+1} \in \underline{V}_\Phi^h$ and $\kappa_\M^{m+1} \in W^h$ such that 
\begin{subequations}
\label{eq:fem_scheme_geodesic}
\begin{align}
\langle \frac{\delta \vec X^{m+1}}{\tau_m}, \chi \,\vec\omega_\M^m \rangle_m^h -  \sigma\langle \kappa_\M^{m+1}, \chi \rangle_m^h =&\langle F^m, \chi\rangle_m^h,  &&\forall \chi \in W^h, \label{eq:fem_scheme_a_geodesic}\\
\langle \kappa_\M^{m+1}\, \vec\omega_\M^m, \vec\eta\rangle_m^h + \langle \nabla_s \delta \vec X^{m+1}, \nabla_s \vec\eta \rangle_m =& -\langle \nabla_s  \vec X^{m}, \nabla_s \vec\eta \rangle_m, && \forall \vec\eta \in \underline{V}_\Phi^h,\label{eq:fem_scheme_b_geodesic}
\end{align}
\end{subequations}
where $F^m = (F_1^m,\ldots, F_{N_C}^m) \in W^h$, with $F_i^m$, $i=1, \ldots, N_C$, is the piecewise linear function uniquely given by 
\begin{equation}
F_i^m(q_j^i) = \lambda [(u_0(\vec X_i^m(q_j^i)) - c_{k^+(i)}^m)^2 - (u_0(\vec X_i^m(q_j^i)) - c_{k^-(i)}^m)^2 ],
\label{eq:external_forcing_term_rac_geodesic_numerical}
\end{equation}
where $c_{k^\pm(i)}^m$ are approximations of the coefficients $c_{k^\pm(i)}$ at $t_m$, cf. \eqref{eq:geodesic_coefficients_multiphase}. We will later state how the coefficients $c_k^m$, $k=1,\ldots,N_R$, can be computed. 

Having found $\delta \vec X^{m+1}\in \underline{V}_\Phi^h$, we set $\vec X^{m+1}:=\delta \vec X^{m+1} + \vec X^m \in \underline{V}^h$.

Before we proceed to prove existence and uniqueness of a solution of the scheme \eqref{eq:fem_scheme_geodesic}, we state some very mild assumptions.

\vspace{2ex}
\noindent $(\mathcal A_1) \quad $   Let $i\in \{1, \ldots, N_C\}$. If $\partial \Gamma_i^m = \emptyset$, we assume that 
$\mathrm{dim} \,\mathrm{span}  \{ \vec\omega_{\M,i}^m(q_j^i),  \vec\omega_{\Phi,i}^m(q_j^i)\}_{j=1}^{N_i} = 3.$

\vspace{2ex}
\noindent $(\mathcal A_2)\quad $ For each $k\in \{1, \ldots, N_T\}$, we assume that 
$\mathrm{dim} \,\mathrm{span} \{ \{ \vec\omega_{\M,i_{k,l}}^m(q_j^{i_{k,l}}), \vec\omega_{\Phi,i_{k,l}}^m(q_j^{i_{k,l}})\}_{j=1}^{N_{i_{k,l}}-1} \}_{l=1}^3  = 3.$ 

\vspace{2ex}
\begin{theorem}
Let the assumptions $(\mathcal{A})$, $(\mathcal{A}_1)$ and $(\mathcal{A}_2)$ hold. Then there exists a unique solution $(\delta \vec X^{m+1}, \kappa_\M^{m+1})$ $\in \underline{V}_\Phi^h \times W^h$ to the system \eqref{eq:fem_scheme_geodesic}. 
\end{theorem}
\begin{proof}
The system \eqref{eq:fem_scheme_geodesic} is linear. Therefore, existence of a solution follows from its uniqueness. To prove uniqueness, consider the following system: Find $\{ \vec X, \kappa_\M \} \in \underline{V}_\Phi^h \times W^h$ such that 
\begin{subequations}
\begin{align}
\langle \vec X, \chi \,\vec\omega_\M^m \rangle_m^h - \sigma \tau_m \langle \kappa_\M, \chi \rangle_m^h &= 0, && \forall \chi \in W^h, \label{eq:uniqueness1_geodesic}\\
\langle \kappa_\M \,\vec\omega_\M^m, \vec\eta \rangle_m^h + \langle \nabla_s \vec X, \nabla_s \vec\eta \rangle_m &=0, && \forall \vec\eta \in \underline{V}_\Phi^h. \label{eq:uniqueness2_geodesic}
\end{align}
\end{subequations}
We obtain choosing $\chi = \kappa_\M \in W^h$ in \eqref{eq:uniqueness1_geodesic} and $\vec\eta = \vec X \in \underline{V}_\Phi^h$ in \eqref{eq:uniqueness2_geodesic} 
\begin{equation*}
\sigma \tau_m \langle \kappa_\M, \kappa_\M \rangle_m^h + \langle \nabla_s \vec X, \nabla_s \vec X\rangle_m = 0.
\end{equation*}
From this equation, we conclude $\kappa_\M \equiv 0$ and $\vec X \equiv \vec X^c$ for a constant $\vec X^c = (\vec X_1^c, \ldots,$ $  \vec X_{N_C}^c) \in (\mathbb{R}^3)^{N_C}$  with $\vec X_{i_{k,1}}^c = \vec X_{i_{k,2}}^c = \vec X_{i_{k,3}}^c$ for all $k\in \{1, \ldots, N_T\}$. Further, $\vec X_i^c \in \mathbb{R}^3$ satisfies
\begin{equation}
\vec X_i^c \,.\,\vec\omega_{\Phi,i}^m(q_j^i) = 0
\label{eq:uniqueness_help}
\end{equation}
for all $i=1,\ldots,N_C$ and $j=j_0^i, \ldots, N_i$, since $\vec X \in \underline{V}_\Phi^h$. Inserting $\kappa\equiv 0$ and $\vec X \equiv \vec X^c$,   \eqref{eq:uniqueness1_geodesic} reduces to 
\begin{equation}
\langle \vec X^c, \chi \,\vec\omega_\M^m   \rangle_m^h  = 0, \quad \forall \chi \in W^h. 
\label{eq:uniqueness1b_geodesic}
\end{equation}
We now choose $\chi = \chi_{i,j} \in W^h$, $i\in \{1, \ldots, N_C\}$, $j\in \{j_0^i, \ldots, N_i\}$ in \eqref{eq:uniqueness1b_geodesic}. This yields
\begin{equation}
\vec X_i^c \,.\, \vec\omega_{\M,i}^m(q_j^i) = 0. 
\label{eq:uniqueness1c_geodesic}
\end{equation}
We conclude $\vec X_i^c \equiv 0$ using \eqref{eq:uniqueness_help}, \eqref{eq:uniqueness1c_geodesic} and the assumptions $(\mathcal A_1)$ and $(\mathcal A_2)$. 

\end{proof}

\subsection{Solution of the Discrete System}
\label{subsec:solution_discrete_system_geodesic}
Let $N = \sum_{i=1}^{N_C} N_i^*$, with $N_i^* = N_i$ for closed curves and $N_i^* = N_i + 1$ for open curves. We make use of a small abuse of notation and consider functions in $W^h$ as elements in $\mathbb{R}^N$ and functions in $\underline{V}^h$ as elements in  
\begin{equation*}
\mathbb{X} = \left\{ (\vec z_1, \ldots, \vec z_{N_C}) \in (\mathbb R^3)^{N} \,:\, [\vec z_{i_{k,1}}]_{j_{k,1}} = [\vec z_{i_{k,2}}]_{j_{k,2}} = [\vec z_{i_{k,3}}]_{j_{k,3}}, \,k=1, \ldots N_T\right\},
\end{equation*}
where $\vec z_i \in (\mathbb R^3)^{N_i^*}$ and $[\vec z_i]_j \in \mathbb{R}^3$ is the $j$-th component of the vector $\vec z_i$. Functions in $\underline{V}_\Phi^h$ are considered as elements in 
\begin{equation*}
\mathbb{X}_\Phi =  \left\{ (\vec z_1, \ldots, \vec z_{N_C}) \in \mathbb{X} \,:\, [\vec z_i]_j \,.\, \vec\omega_{\Phi,i}^m(q_j^i) = 0, \,i=1,\ldots,N_C, \, j=j_0^i, \ldots, N_i\right\},
\end{equation*}
with $j_0^i=0$ for open curves and $j_0^i=1$ for closed curves. Let $\vec P_\Phi: (\mathbb{R}^3)^N \rightarrow \mathbb{X}_\Phi$ denotes the orthogonal projection onto the space $\mathbb{X}_\Phi$. 

In order to state a matrix formulation for the discrete system \eqref{eq:fem_scheme_geodesic}, we introduce the following matrices
\begin{equation*}
M:= \left(
\begin{array}{ccc}
M^1 & \cdots & 0 \\
\vdots &  \ddots & \vdots \\
0 & \ldots & M^{N_C} 
\end{array}
\right),\quad 
\vec N_\M:= \left(
\begin{array}{ccc}
\vec N_\M^1  & \cdots & 0 \\
\vdots  & \ddots & \vdots \\
0 &  \ldots & \vec N_\M^{N_C} 
\end{array}
\right),\quad 
\vec A:= \left(
\begin{array}{ccc}
\vec A^1  & \cdots & 0 \\
\vdots &  \ddots & \vdots \\
0 &  \ldots & \vec A^{N_C} 
\end{array}
\right),
\end{equation*}
where $M^i \in \mathbb{R}^{N_i^* \times N_i^*}$, $\vec N_\M^i \in (\mathbb{R}^3)^{N_i^* \times N_i^*}$ and $\vec A^i \in (\mathbb{R}^{3\times 3})^{N_i^* \times N_i^*}$,  $i=1, \ldots, N_C$, are defined by 
\begin{equation*}
M_{jl}^i := \langle \chi_{i,j}, \chi_{i,l} \rangle_m^h, \quad 
(\vec N_\M^i)_{jl} :=\langle \chi_{i,j}, \chi_{i,l} \,\vec\omega_\M^m \rangle_m^h, \quad 
\vec A_{jl}^i := \langle \nabla_s \chi_{i,j}, \nabla_s \chi_{i,l} \rangle_m \,\vec{\mathrm{Id}}_{3},
\end{equation*}
where $\vec{\mathrm{Id}}_{3}$ denotes the $3 \times 3$ identity matrix. 
We define $b^m = (b_1^m, \ldots, b_{N_C}^m) \in \mathbb{R}^N$ by 
\begin{equation}
b_i^m = (b_{i,j_0^i}^m, \ldots, b_{i,N_i}^m), \quad \text{with }\,\,b_{i,j}^m := \langle F_i^m, \chi_{i,j} \rangle_m^h, \,
 i=1, \ldots, N_C, \,j=j_0^i, \ldots, N_i.
\label{eq:righthandside_geodesic}
\end{equation}

The discrete system \eqref{eq:fem_scheme_geodesic} can be rewritten into the following matrix-vector formulation: Find $\kappa_\M^{m+1}\in \mathbb{R}^N$ and  $\delta \vec X^{m+1}  \in \mathbb{X}_\Phi$, such that
\begin{equation}
\left(
\begin{array}{cc}
-\sigma \tau_m M & \,\,\vec N_\M^T \vec P_\Phi \\
\vec P_\Phi \vec N_\M & \,\,\vec P_\Phi \vec A \vec P_\Phi 
\end{array}
\right) \left(
\begin{array}{c}
\kappa_\M^{m+1} \\
\delta \vec X^{m+1} 
\end{array}
\right) = \left( 
\begin{array}{c}
\tau_m b^m \\
-\vec P_\Phi \vec A \vec X^m
\end{array} \right),
\label{eq:linear_system_geodesic}
\end{equation}
holds, on assuming that $\vec X^0 \in \mathbb{X}$. 

Since $M$ is non-singular, we can apply a Schur complement approach and obtain
\begin{subequations}
\begin{align}
\kappa_\M^{m+1} &= \frac{1}{\sigma \tau_m} M^{-1} \left( \vec N_\M^T \vec P_\Phi \,\delta \vec X^{m+1} - \tau_m b^m\right), \\
\left( \vec P_\Phi \vec A \vec P_\Phi + \frac{1}{\sigma \tau_m} \vec P_\Phi \vec N_\M M^{-1} \vec N_\M^T \vec P_\Phi \right) 
\delta \vec X^{m+1} &= \frac{1}{\sigma } \vec P_\Phi \vec N_\M M^{-1} b^m - \vec P_\Phi \vec A \vec X^m.\label{eq:schur_geodesic}
\end{align}
\end{subequations}
Since $\vec P_\Phi$ is a projection to a subspace of $(\mathbb R^3)^N$, the system matrix of the linear equation \eqref{eq:schur_geodesic} is singular as a mapping of $(\mathbb R^3)^N \rightarrow (\mathbb R^3)^N$. However, considered as a mapping of $\mathbb{X}_\Phi \rightarrow \mathbb{X}_\Phi$ it is non-singular if the assumptions $(\mathcal A_1)$ and $(\mathcal A_2)$ hold. 

Since the system matrix is sparse, \eqref{eq:schur_geodesic} can be efficiently solved with linear effort using an iterative solver (with possible preconditioning) or using a direct solver for sparse matrices. In the examples presented later in Section~\ref{sec:results}, we use the UMFPACK algorithm \cite{Davis04} (direct solver) as MATLAB built-in routine for sparse linear systems. 

In the following, we will use the abbreviation $\vec X_{i,j}^m := \vec X_i^m(q_j^i)$. 

\subsection{Topology Changes}
\label{subsec:num_top_changes}
Our scheme is based on a parametrization of evolving curves on surfaces. Topology changes concerning the curves are not automatically handled which is often considered as the main drawback of parametric methods. 

There are different topology changes that can occur: A curve can split in two sub-curves (splitting), two curves separating the same regions can merge to one single curve (merging), two curves separating different regions can touch and a new curve and two new triple junctions occur (creation of triple junctions) and a curve can shrink and has to be deleted. 

The latter can be simply detected by considering the length of the curve. If the length of a curve is smaller than a predefined tolerance, it will be deleted. The other topology changes will occur if two points from different curves or two points from one curve which are not neighbors have a small distance. A simple comparison of all nodes would lead to an effort of $\mathcal{O}(N^2)$, where $N$ is the total number of nodes of the polygonal curves $\Gamma_1^m, \ldots, \Gamma_{N_C}^m$. Since we can solve the linear equation of our main algorithm \eqref{eq:schur_geodesic} efficiently with linear effort, a sub-algorithm to detect topology changes should not result in a too large computational effort. 

For curves in the plane, we extended an efficient method to detect topology changes \cite{Benninghoff2014a} which was originally developed by Mikula and Urb\'{a}n \cite{MikulaUrban12}, see also \cite{Balazovjech12}. The method to detect topology changes is based on an artificial background grid covering the image domain and consisting of a finite set of arrays (squares). If two nodes from different curves or two nodes from different parts of one curve lie in the same array of the virtual background grid, a topology change likely occurs close to the two points. Using this method, the effort to detect topology changes is $\mathcal{O}(N)$.

In principle, the idea of an artificial background grid can be extended to detect topology changes involving curves on surfaces. One can construct a 3D background grid around the surface $\M$. Again, we can check whether two points of different curves or different parts of one curve belong to the same array of the background grid. 

However, using the Euclidean distance in $\mathbb{R}^3$ to detect topology changes can lead to false detections: Surfaces exist where two points can have a small Euclidean distance but their \textit{geodesic distance}, i.e. the length of the smallest curve on the surface connecting the two points, is large. In such situations, a topology change does not occur. The geodesic distance would be a better indicator for detection of topology changes compared to the Euclidean distance. However, the computation of the geodesic distance between two points on a surface is very expensive from a computational point of view and cannot be used in a sub-algorithm in practice. 

Let $a>0$ be the grid size of the cubes of the 3D background grid. For extending the method to detect topology changes from the planar case \cite{Benninghoff2014a} to the case of curves on surfaces, we have to choose the grid size $a$ small enough to exclude such wrong detections as described above. 

For $\vec p\in \M$ let $T_{\vec p} \M$ denote the tangent space and $N_{\vec p}\M = (T_{\vec p}\M)^\perp$ the normal space. Let $N\mathcal{M} = \left\{(\vec p,\vec n)\,:\, \vec p \in \mathcal{M}, \,\vec n \in N_{\vec p}\mathcal{M}\right\}$ denote the normal bundle. 

For the smooth, embedded hypersurface, we consider the map 
\begin{equation*}
E: N\mathcal{M} \rightarrow \mathbb{R}^3, \quad (\vec p,\vec n)\mapsto \vec p + \vec n.
\end{equation*}

\begin{theorem}[Tubular neighborhood theorem]
Every embedded hypersurface $\mathcal{M}$ of $\mathbb{R}^3$ has a tubular neighborhood, i.e. a neighborhood $U\subset \mathbb{R}^3$ that is the diffeomorphic image under $E: N\mathcal{M} \rightarrow \mathbb{R}^3$ of an open subset $V \subset N\mathcal{M}$ of the form
\begin{equation}
V = \left\{(\vec p,\vec n) \in N\mathcal{M}\,:\, \|\vec n\| < \delta(\vec p)\right\},
\label{eq:tubular_neighborhood}
\end{equation}
for some positive continuous function $\delta:\mathcal{M}\rightarrow \mathbb{R}$. 
\end{theorem}
\begin{proof}
See \cite{Lee02}, Chapter 6, Embedding and Approximation Theorems, or \cite{Lang02}, Chapter 4, Vector Fields and Differential Equations. 
\end{proof}

For images on surfaces, we assume the surface $\M$ to be a compact, embedded hypersurface. As a consequence, set $\delta_0 = \mathrm{min}\left\{\delta(\vec p)\,:\, \vec p \in \M\right\}>0$. For each $\vec p \in \M$ the intersection $B_{\delta_0}(\vec p)\cap\mathcal{M}$ is simply connected, which is a consequence of the fact that $E|_V: V \rightarrow U$ is a  diffeomorphism.

Thus, the key idea when extending the algorithm from the planar case to curves on surfaces is to choose the grid size $a$ of the auxiliary 3D background grid small enough with respect to $\delta_0$, such that points from two different parts of the surface (with nearly opposite normal vector $\vec \nu_\Phi$) cannot lie in one array of the grid.
 
Topology changes are now detected as follows:
\begin{itemize}
\item Construct an underlying 3D grid with grid size $a$ with $a\sqrt{3} < \delta_0$. Note that the intersection of a grid element (=cube of grid length $a$) with $\M$ is simply connected. 
\item Mark the grid elements with the indices of the curves and the mesh points: We successively consider the mesh points $\vec X_{i,j}^m$, $i=1,\ldots,N_C$, $j=j_0^i,\ldots, N_i$. If the corresponding grid array, in which $\vec X_{i,j}^m$ lies, is empty, the grid is marked with $(i,j)$. 
\item If a grid array is already marked with $(i_1,j_1)$ and if $\vec X_{i,j}^m$ and $\vec X_{i_1,j_1}^m$ are no neighbor nodes, a topology change is detected. 
\item Since $\vec X_{i,j}^m$ and $\vec X_{i_1,j_1}^m$ may not be the pair with the smallest distance, we consider a few neighbor nodes around $\vec X_{i,j}^m$ and $\vec X_{i_1,j_1}^m$. Let $\vec X_{i,l}^m$ and $\vec X_{i_1,l_1}^m$ be the pair with the smallest Euclidean distance in these two small groups of nodes. They can be found by a pairwise comparison, which is not computationally expensive since only a few nodes are involved. 
\end{itemize}

The topology changes splitting, merging and creation of triple junctions (see explanations above) are distinguished as follows: 
\begin{itemize}
\item If $i=i_1$, a splitting of the curve $\Gamma_i^m$ is detected.  
\item If $i\neq i_1$, we consider the regions separated by $\Gamma_i^m$ and $\Gamma_{i_1}^m$: If $k^+(i)=k^+(i_1)$  $\land\, k^-(i)=k^-(i_1)$ ,  or alternatively $k^+(i)=k^-(i_1)$ $\land\, k^-(i)=k^+(i_1)$ holds, a merging occurs. 
\item Otherwise, a creation of a new contour and a creation of two new triple junctions occur. 
\end{itemize}

After having detected and identified the topology change, the curves need to be adapted near $\vec X_{i,l}^m$ and $\vec X_{i_1,l_1}^m$. This involves changing the neighbor relations, changing curve indices in case of merging or splitting, and  creation of a small new contour with a few nodes in case of triple junctions. Details are given in \cite{Benninghoff2014a}.

In case of triple junctions, a new curve is created. When creating new nodes, one has to ensure that these nodes lie on the surface $\M$. In the next section, we describe how nodes can be efficiently projected to the surface. 

\subsection{Additional Computational Aspects}
\label{subsec:num_add_comp}
\paragraph{Triangulated surfaces}
In practical applications, a smooth function $\Phi : \mathbb{R}^3 \rightarrow \mathbb{R}$, such that $\M$ is the zero level set of $\Phi$, is usually not provided. Moreover, a surface $\M$ is typically given as a triangulated surface instead of a smooth surface. 

Therefore, we assume that $\M$ is a union of triangles of a triangulation $\mathcal{T}^h$, i.e. $\mathcal{M} = \bigcup_{\sigma^h \in \mathcal{T}^h} \overline{\sigma^h}$. Note, that the function $\Phi$ was only needed to compute $\vec n_\Phi$. Normal vectors to the surface can now be easily computed for each triangle $\sigma^h$. For a point $\vec p$ on a curve $\Gamma^m \subset \M$, we first need to assign $\vec p$ to a triangle $\sigma^h \in \mathcal{T}^h$ in which the node lies, to compute $\vec n_\Phi(\vec p)$. Further, the color data  $u_0$ is often piecewise constant and uniquely given by its value on the triangles. To evaluate $u_0(\vec p)$, we also need to assign the node to its corresponding triangle. 

For each simplex, we can project a vector in $\mathbb{R}^3$ to the simplex plane and can use barycentric coordinates to determine if the projected node lies inside the triangle. Surfaces are often composed of $10^5$ to $10^6$ triangles. Therefore, for a given point, finding the corresponding triangle in which the point lies results in a high computational effort if no additional knowledge is used. 

For $m=0$ and a curve $\Gamma_i^m$, $i\in\{1,\ldots,N_C\}$, with nodes $\vec X_{i,j}^m$, $j=j_0^i, \ldots, N_i$, we perform a global search only for $\vec X_{i,j_0^i}^m$. For $j>j_0^i$, we consider first the simplex to which $\vec X_{i,j-1}^m$ has been assigned. If the node $\vec X_{i,j}^m$ is not located in the same simplex, we start a search considering successively the neighbor simplices. 
For $m>0$, we can assume that a node has moved only slightly on the surface from step $m-1$ to $m$. Therefore, we start the search using the triangle to which the node was assigned in time step $m-1$. Consequently, a global search has to be performed only $N_C$ times at the beginning of the segmentation. 

After the linear system \eqref{eq:schur_geodesic} has been solved, some of the nodes may not lie exactly on the surface. For smooth surfaces (like spheres, tori, etc.), the nodes stay very close to the surface if small time steps are used, see \cite{BGN10}. However, for triangulated surfaces, a reprojection onto the surface is necessary since $\vec \nu_\Phi$ is not continuous. A reprojection onto the surface does not result in an additional computational effort: In the next time step, we need to compute $\vec \nu_\Phi = \vec n_\Phi \circ \vec x$ and need to evaluate $u_0$ again for each node. For both, we have to determine again the closest triangle for a point. As described above, this is done by projection of the original node to the triangle plane and by using barycentric coordinates. I.e. we already need to determine a projection of the original point to the corresponding triangle.

\paragraph{Computation of regions and coefficients}
For the external forcing term, we need to determine the regions $\Omega_k^m$, approximations of  $\Omega_k(t_m)$, $k=1,\ldots,N_R$. The regions $\Omega_k^m$ are separated and thus determined by the union of  discrete curves $\Gamma^m = \Gamma_1^m \cup \ldots \cup \Gamma_{N_C}^m$. Further, we need to compute the coefficients $c_k^m$ which are the average color values of the image function $u_0$ in the corresponding regions.

For $m=0$, we need to assign each simplex $\sigma^h \in \mathcal{T}^h$ to a region $\Omega_k^0$. For $m>0$, we need to update the assignment only in a neighborhood of the curves. 

Let $\vec p_{\sigma^h,j}$, $j=1,2,3$, denote the vertices of a triangle $\sigma_h$. We assign the simplex to a region $\Omega_k^m$ if its center $\vec p_{\sigma^h} = (\vec p_{\sigma^h,1}+\vec p_{\sigma^h,2}+\vec p_{\sigma^h,3})/3$ belongs to the region. In the rare case, that $\vec p_{\sigma^h}$ lies directly on a curve $\Gamma_i^m$, it is assigned either to $\Omega_{k^+(i)}^m$ or $\Omega_{k^-(i)}^m$. For image segmentation, we do not apply any special treatment to simplices which are truncated by a curve.  

For a simplex $\sigma^h$ close to a curve with center $\vec p_{\sigma^h}$, we can search for the closest node $\vec X_{i,j}^m$ and consider the sign of $(\vec p_{\sigma^h} - \vec X_{i,j}^m ) \,.\,\vec\omega_\M^m(q_j^i)$. 

For $m=0$, we also need to consider simplices which are not close to a curve. The direction $\vec\omega_\M^m(q_j^i)$ cannot be used for remote simplices if the surface $\M$ is curved. Having assigned a small band of simplices around the curves, the remaining simplices can inherit the region index by using the neighbor relation between the simplices of the triangulation. 

Motivated by these thoughts, we propose the following algorithm for computation of the regions: 
\begin{itemize}
\item For all nodes $\vec X_{i,j}^m$, $i=1,\ldots,N_C$, $j=j_0^i,\ldots,N_i$, we consider the triangle $\sigma^h$ to which $\vec X_{i,j}^m$ belongs (see determination of the closest triangle described above). If $(\vec p_{\sigma^h} - \vec X_{i,j}^m  ) \,.\,\vec\omega_\M^m(q_j^i)$ is positive, the simplex is assigned to $\Omega_{k^+(i)}^m$, otherwise to $\Omega_{k^-(i)}^m$. The indices of neighbor simplices of $\sigma^h$ are stored in a list. 
\item We consider the simplices of the auxiliary list, which have not been assigned to a region yet. For a simplex $\sigma^h$ of the list, we search for the closest node point $\vec X_{i,j}^m$ and determine a region index using the sign of $( \vec p_{\sigma^h} - \vec X_{i,j}^m) \,.\,\vec\omega_\M^m(q_j^i)$. We store the neighbor simplices of $\sigma^h$ in a new list. 
\item Having assigned all simplices of the current list to a region, the current list is deleted and the simplices of the new list are considered. We repeat the procedure $n_0$ times. Afterwards, a small band of simplices around each curve is assigned to regions. 
\item For $m=0$, the remaining simplices are considered successively by using again lists of neighbor simplices. After the step $n_0$, we do not determine the closest node $\vec X_{i,j}^m$. A new simplex inherits the region index directly from its neighbor. 
\end{itemize}

Figure \ref{fig:bunny_regions} illustrates the assignment of regions for simplices of a triangulated surface. It shows the Stanford Bunny\footnote{\url{https://graphics.stanford.edu/data/3Dscanrep/}} from the Stanford Computer Graphics Laboratory, cf. \cite{Turk94}, and an image with three small discs on its surface. We used $n_0=4$ levels of neighbor simplices to assign the simplices of a small band around the initial curve to one of the two regions separated by the initial curve. The remaining simplices (marked with dark color in the second subfigure) are finally assigned to a region by heritage of the region index.

The final coefficients are computed as follows: Let $C_k^m = \sum_{\sigma^h \in \Omega_k^m} u_0|_{\sigma^h}$ denote the sum of the color data and $n_k^m$ the number of simplices belonging to $\Omega_k^m$. The coefficient for the Chan-Vese model is computed by setting $c_k^m = C_k^m / n_k^m$. For color spaces like the CB space, we need to make use of normalized means for some components of the color \cite{Benninghoff2014a}. 

For $m>0$, we need to update the coefficients only close the curve, i.e. we consider only the simplices of a small band around the curves (using again $n_0$ levels of neighbor simplices around the curves). If a simplex $\sigma^h$ changes its region assignment from $\Omega_l^m$ to $\Omega_k^m$, we set
\begin{equation}
n_k^m = n_k^m + 1, \quad n_l^m = n_l^m -1, \quad
C_k^m = C_k^m + u_0|_{\sigma^h}, \quad C_l^m = C_l^m - u_0|_{\sigma^h}.
\end{equation}

\begin{figure}[t]
	\centering
		\includegraphics[width=0.25\textwidth]{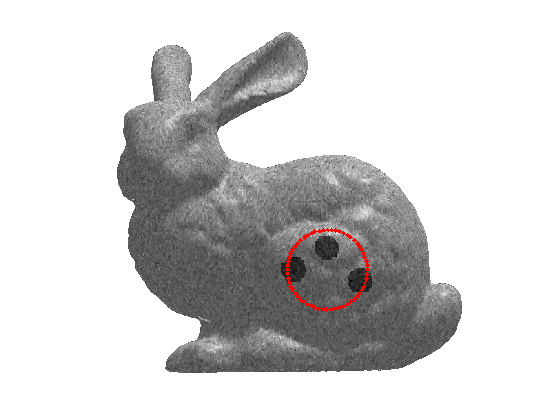}
		\includegraphics[width=0.25\textwidth]{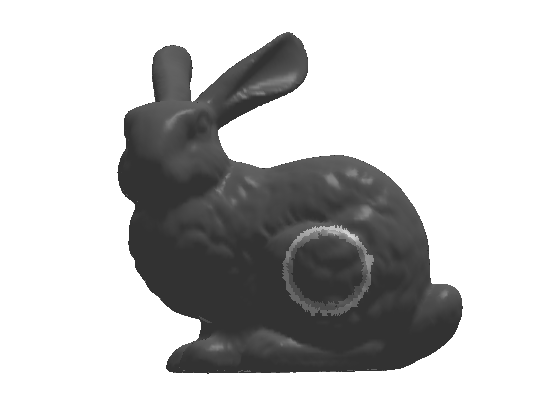} 
		\includegraphics[width=0.25\textwidth]{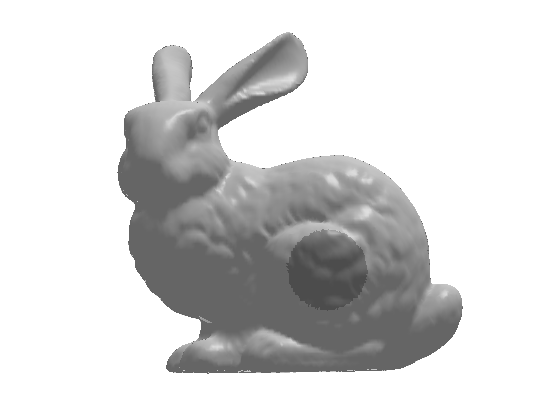}
	\caption{Illustration how triangles are assigned to a region. 1st sub-figure: image and initial curve. 2nd sub-figure: small band after $n_0=4$ steps. 3rd sub-figure: regions colored with mean brightness value after assignment of all triangles. The surface data is from the Stanford Computer Graphics Laboratory, cf. \cite{Turk94}.}
	\label{fig:bunny_regions}
\end{figure}

\paragraph{Global refinement-coarsening strategy}
We can perform a glo\-bal refinement and coarsening of the curves by using two thresholds $l_\mathrm{max}$ and $l_{\mathrm{min}}$ for the average distance between neighboring nodes. Let $N_i^m$ be the number of nodes belonging to a curve $\Gamma_i^m$. If $|\Gamma_i^m|\,/N_i^m > l_{\mathrm{max}}$, we perform a global refinement of the curve by inserting a new node between two neighbor nodes. On the contrary, if $|\Gamma_i^m|\,/N_i^m < l_{\mathrm{min}}$, we perform a global coarsening, i.e. we delete every second node of the polygonal curve which is not a boundary point.

When inserting a new node between two nodes $\vec X_{i,j}^m$ and $\vec X_{i,j+1}^m$ during a refinement, we first compute $\vec p = (\vec X_{i,j}^m + \vec X_{i,j+1}^m)/2$ and determine the closest triangle (again the search for the closest triangle is done efficiently by starting with the triangle in which e.g. $\vec X_{i,j}^m$ lies). Having found the closest triangle $\sigma^h$, $\vec p$ is projected orthogonally to $\sigma^h$. Consequently, all newly generated nodes lie on the surface.

\subsection{Numerical Solution of the Image Restoration Scheme}
\label{subsec:finitedifference_imagedenoising_geodesic}
The scheme \eqref{eq:diffusion_eq_geodesic_strong} for image restoration can be solved numerically with a finite element approach. Again, we consider a polyhedral surface $\M$ given by a set $\mathcal{T}^h$ of triangles. 

The image restoration is performed as postprocessing step using the final regions from the time step $m=M$. The surface thus consists of polyhedral regions $\Omega_k^h := \Omega_k^{M}$, $k=1,\ldots,N_R$. Let $\mathcal{T}_k^h = \{\sigma^h \in \mathcal{T}^h \,:\, \sigma^h \subset \Omega_k^h\}$ denote the set of triangles belonging to $\Omega_k^h$ and let $\vec p_{k,j}$, $j=1, \ldots, N_k^h$, denote the vertices of the triangles belonging to $\mathcal{T}_k^h$.

For each $k=1,\ldots,N_R$, we define the following finite element space 
\begin{equation}
S_k^h := \left\{ u^h \in C(\overline{\Omega_k^h}, \mathbb{R}) \,:\, {u^h}_{|\sigma^h} \,\text{ is linear } \,\,\forall \sigma^h \in \mathcal{T}_k^h\right\}.
\end{equation}
For piecewise continuous functions $u^h, v^h: \Omega_k^h \rightarrow \mathbb{R}^{(3)}$ with possible jumps at  edges of simplices $\sigma^h \in \mathcal{T}_k^h$, we define the mass lumped inner product
\begin{equation}
\langle u^h, v^h \rangle^h := \frac{1}{3} \sum_{\sigma^h \in \mathcal{T}_k^h} |\sigma^h| \sum_{j=1}^3 (u^h\,.\, v^h)((\vec p_{\sigma^h, j})^-),
\end{equation}
where $|\sigma^h|$ denotes the area of $\sigma^h$, and as above $\vec p_{\sigma^h,j}$, $j=1,2,3$, denote the vertices of the triangle $\sigma^h \in \mathcal{T}_k^h$ and 
\begin{equation*}
u^h((\vec p_{\sigma^h, j})^-) := \lim_{\vec p \rightarrow \vec p_{\sigma^h,j},\,\vec p \in \sigma^h} u^h(\vec p).
\end{equation*}
Further, for functions $u^h, v^h \in L^2(\Omega_k^h, \mathbb{R}^{(3)})$, we define
\begin{equation}
\langle u^h, v^h \rangle := \int_{\Omega_k^h} u^h\,.\, v^h \,\mathrm{d}A.
\end{equation}

We consider the following discrete system for each region $k \in \{1,\ldots, N_R\}$: Find $u^h \in S_k^h$ such that 
\begin{equation}
 \frac{1}{\lambda} \langle \nabla_\M u^h, \nabla_\M v^h \rangle + \langle u^h, v^h \rangle^h = \langle u_0, v^h \rangle^h, \quad \forall v^h \in S_k^h,
\label{eq:image_diffusion_geodesic_discrete}
\end{equation}
where $\lambda > 0$ is a weighting parameter (cf. \eqref{eq:mumford_shah_geodesic}).

Let $\{\phi_{k,i}^h\}_{i=1}^{N_k^h}$ with $\phi_{k,i}^h(\vec p_{k,j}) = \delta_{ij}$ denote the standard basis of $S_k^h$. Using this standard basis, we can identify each element in $S_k^h$ with its coefficient vector in $\mathbb{R}^{N_k^h}$. 
Further, we define the matrices $M_k^h, A_k^h \in \mathbb{R}^{N_k^h \times N_k^h}$ by 
\begin{equation*}
(M_k^h)_{ij} := \langle \phi_{k,i}^h, \phi_{k,j}^h \rangle^h, \quad
(A_k^h)_{ij} := \langle \nabla_\M \phi_{k,i}^h, \nabla_\M \phi_{k,j}^h \rangle, \quad i,j=1,\ldots,N_k^h.
\end{equation*}
The entries of the matrices $M_k^h$ and $A_k^h$ are computed by considering each triangle $\sigma^h \in \mathcal{T}_k^h$ and computing the contribution of $\sigma^h$ to the entries corresponding to the indices of its vertices. For computing the contribution of $\sigma^h$ to $A_k^h$ we need to compute surface gradients.

For that, we consider three nodes $\vec p_{k,j_1},\vec p_{k,j_2}$ and $\vec p_{k,j_3}$, $j_1,j_2,j_3 \in \{1,\ldots,N_k^h\}$,  being the vertices of a triangle $\sigma^h$ in $\mathcal{T}_k^h$. For the ease of notation, we assume $j_1=1$, $j_2=2$ and $j_3=3$. One tangential vector is given by
\begin{equation*}
\vec\tau_{1} := \frac{\vec p_{k,3} - \vec p_{k,2}}{\|\vec p_{k,3} - \vec p_{k,2}\|}.
\end{equation*}
A second tangential vector which is orthogonal to $\vec\tau_1$ can be obtained by
\begin{equation*}
\vec\tau_{2} := \frac{\vec p_{k,1} - \vec q_k}{\|\vec p_{k,1} - \vec q_k\|},
\end{equation*}
with $\vec q_k = \vec p_{k,2} + \left((\vec p_{k,1} - \vec p_{k,2})\,.\, \vec\tau_1\right)\vec\tau_1$. 

We note that $\partial_{\vec\tau_1} \phi_{k,1}^h = 0$. For $\partial_{\vec\tau_2} \phi_{k,1}^h$ we consider a curve $\gamma: [0, \|\vec p_{k,1} - \vec q_k\|] \rightarrow \mathbb{R}^3$, $\gamma(\epsilon)= \vec q_k + \epsilon\, \vec \tau_2$. The composition $\phi_{k,1}^h \circ \gamma$ is given by
\begin{equation*}
(\phi_{k,1}^h \circ \gamma)(\epsilon) = \frac{\epsilon}{\|\vec p_{k,1} - \vec q_k\|}.
\end{equation*}
The derivative of $\phi_{k,1}^h$ in direction $\vec\tau_2$ is
\begin{equation*} 
\partial_{\vec\tau_2} \phi_{k,1}^h = \frac{\mathrm{d}}{\mathrm{d}\epsilon} \phi_{k,1}(\gamma(\epsilon))|_{\epsilon=0} =  \frac{1}{\|\vec p_{k,1} - \vec q_k\|}.
\end{equation*}
The surface gradient is then given by
\begin{equation*}
{\nabla_\M \phi_{k,1}^h}|_{\sigma^h} = \partial_{\vec\tau_1} \phi_{k,1}^h \,\vec\tau_1 + \partial_{\vec\tau_2} \phi_{k,1}^h \,\vec\tau_2 
=  \frac{\vec p_{k,1} - \vec q_k}{\|\vec p_{k,1} - \vec q_k\|^2}.
\end{equation*}
Similarly, we compute ${\nabla_\M \phi_{k,2}^h}|_{\sigma^h}$ and ${\nabla_\M \phi_{k,3}^h}|_{\sigma^h}$.

The discrete equation \eqref{eq:image_diffusion_geodesic_discrete} can be rewritten to the following linear system: Find $u^h \in \mathbb{R}^{N_k^h}$ such that 
\begin{equation}
\frac{1}{\lambda} A_k^h u^h + M_k^h u^h = M_k^h U_0,
\label{eq:geodesic_restoration_linear_system_matrix_form}
\end{equation}
holds, where $U_0 \in \mathbb{R}^{N_k^h}$ is given by
\begin{equation}
(U_0)_j = \frac{\sum_{\sigma^h \in \mathcal{T}_{k,j}^h} \,|\sigma^h| \,(u_0)_{|\sigma^h}}{\sum_{\sigma^h \in \mathcal{T}_{k,j}^h}|\sigma^h|}, \quad j=1, \ldots, N_k^h,
\end{equation}
with $\mathcal{T}_{k,j}^h = \left\{\sigma^h \in \mathcal{T}_k^h\,:\, \vec p_{k,j} \in \overline{\sigma^h} \right\}.$ Using this definition of $U_0$, we obtain $\langle u_0, \phi_{k,j}^h \rangle^h = (M_k^h U_0)_j$ for $j=1, \ldots, N_k^h$. 

Note, that the Neumann boundary conditions are automatically incorporated in the scheme \eqref{eq:geodesic_restoration_linear_system_matrix_form}. The finite element approach is based on a weak formulation of \eqref{eq:diffusion_eq_geodesic_strong} which contains the Neumann boundary conditions as natural conditions.

We obtain an element-wise constant image approximation $U^h$ by setting 
\begin{equation}
{U^h}_{|\sigma^h} = \frac{1}{3} \sum_{j=1}^3 u^h(\vec p_{\sigma^h, j}).
\end{equation}
for simplices $\sigma^h \in \mathcal{T}_k^h$. 

By solving the diffusion equation on each region independently, the boundaries of the regions are not smoothed out. Vector-valued images are denoised by applying the method on each component. 

\subsection{Summary of the Image Processing Algorithm}
We propose the following algorithm for image segmentation with postprocessing image restoration: Given a set of polygonal curves $\Gamma^0 = (\Gamma_1^0, \ldots, \Gamma_{N_C}^0)$ and $\vec X^0 = (\vec X_1^0, \ldots, \vec X_{N_C}^0)$ with $\vec X_i^0(I_i)=\Gamma_i^0$, $\vec X_i^0(q_j^i) \in \M$, $i=1,\ldots,N_C$, $j=j_0^i, \ldots, N_i$, perform the following steps for $m=0, 1, \ldots, M-1$: 
\begin{enumerate}
\item \label{step1g} Compute the regions $\Omega_k^m \subset \M$ and the coefficients $c_k^m$, $k=1,\ldots,N_R$, as described in Section \ref{subsec:num_add_comp}. 
\item \label{step2g} Compute $b^m$ as defined in \eqref{eq:righthandside_geodesic} by using the coefficients $c_k^m$ of step \ref{step1g}. Compute $\vec X^{m+1} = \vec X^m + \delta \vec X^{m+1}$ by solving the linear equation \eqref{eq:schur_geodesic}, see Section \ref{subsec:solution_discrete_system_geodesic}.
\item Check if topology changes occur, see Section \ref{subsec:num_top_changes}. In case of a topology change, except for a pure deletion of a curve, repeat the steps \ref{step1g} and \ref{step2g} $n_\mathrm{sub}$-times with a step size of $\tau_m / n_\mathrm{sub}$ and execute the topology change when it occurs in a substep. 
\item If necessary, perform global coarsening or refinement as described in Section \ref{subsec:num_add_comp}
\end{enumerate}

Having found a final segmentation of the image at time $m=M$, perform a restoration by computing a piecewise smooth approximation of the image function as presented in Section \ref{subsec:finitedifference_imagedenoising_geodesic}.

%% file: results.tex
\section{Results and Discussion}
\label{sec:results}

\subsection{Artificial Test Images}

\begin{figure}[t]
\centering
\includegraphics[viewport = 140 280 430 550, width = 0.23\textwidth]{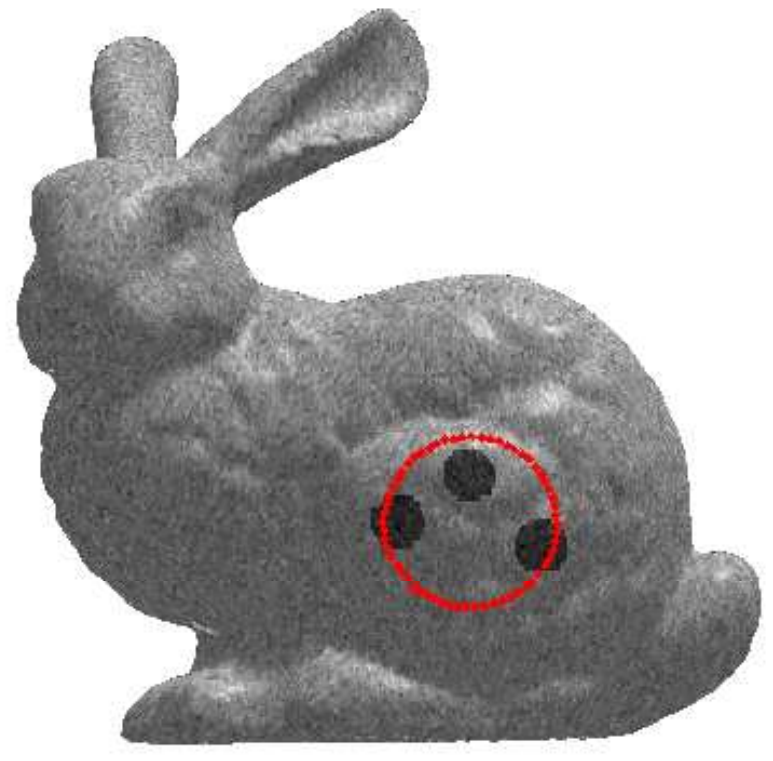}
\includegraphics[viewport = 140 280 430 550, width = 0.23\textwidth]{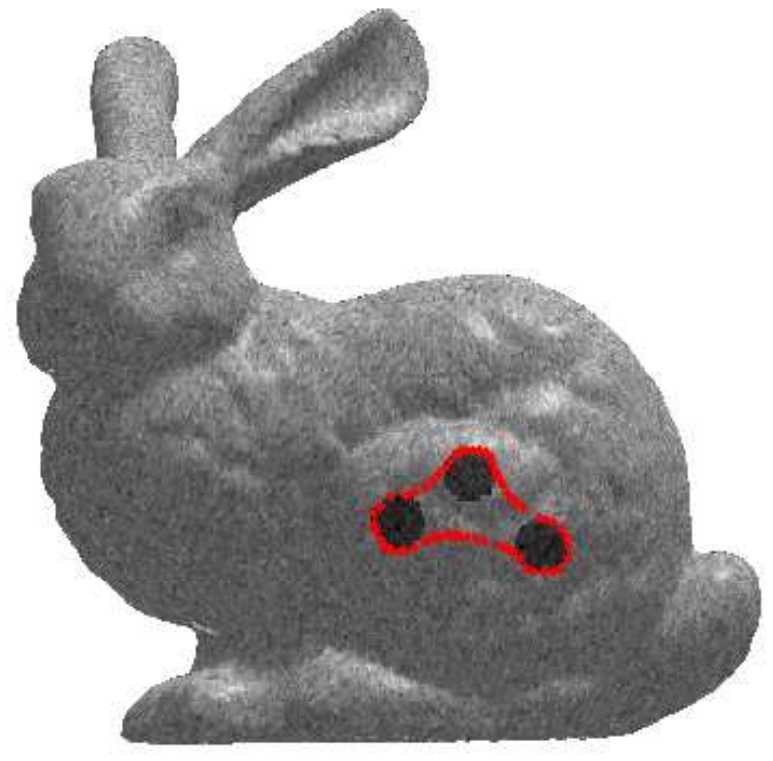}
\includegraphics[viewport = 140 280 430 550, width = 0.23\textwidth]{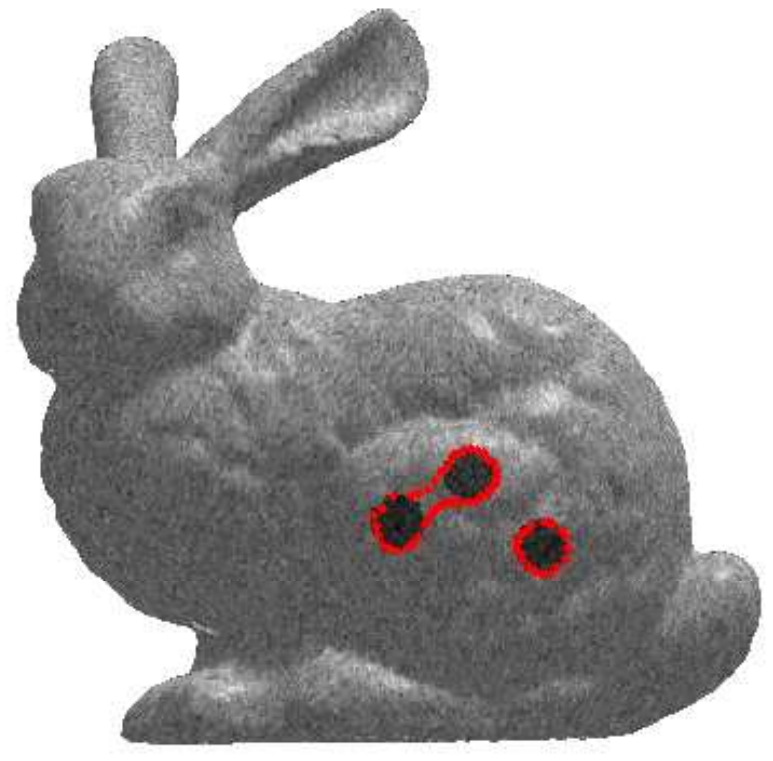}
\includegraphics[viewport = 140 280 430 550, width = 0.23\textwidth]{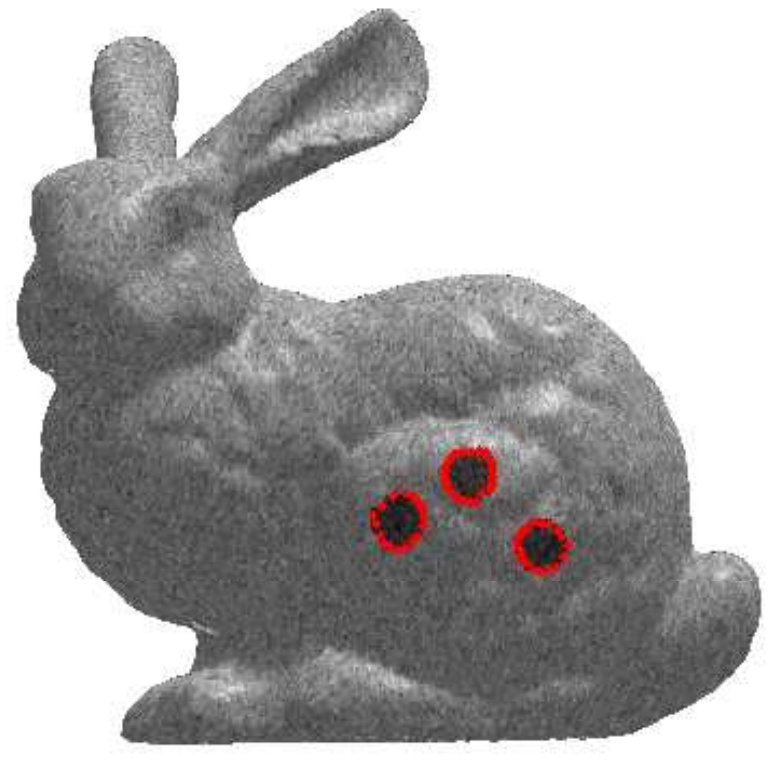}\\
\includegraphics[viewport = 140 280 430 550, width = 0.23\textwidth]{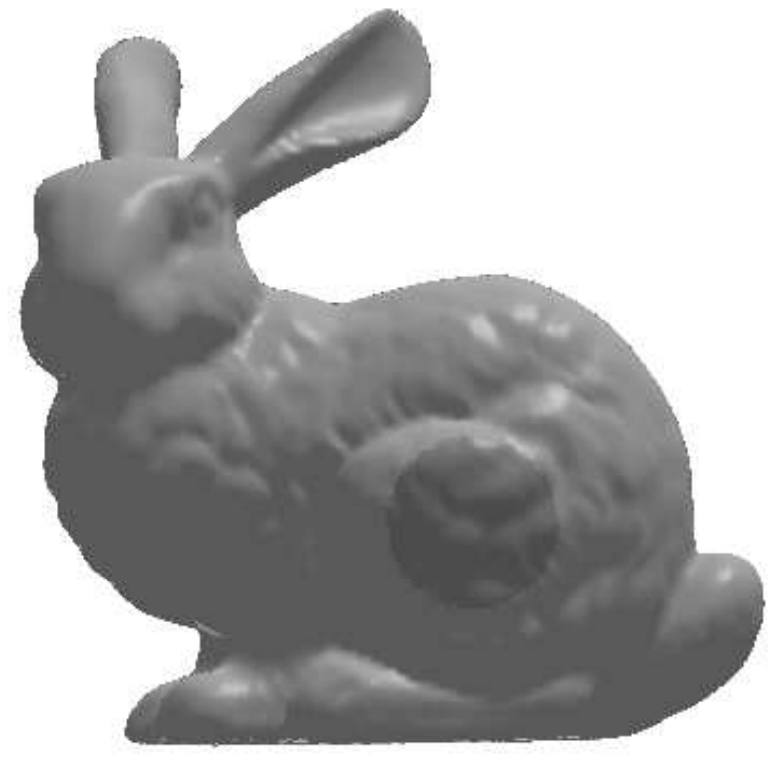}
\includegraphics[viewport = 140 280 430 550, width = 0.23\textwidth]{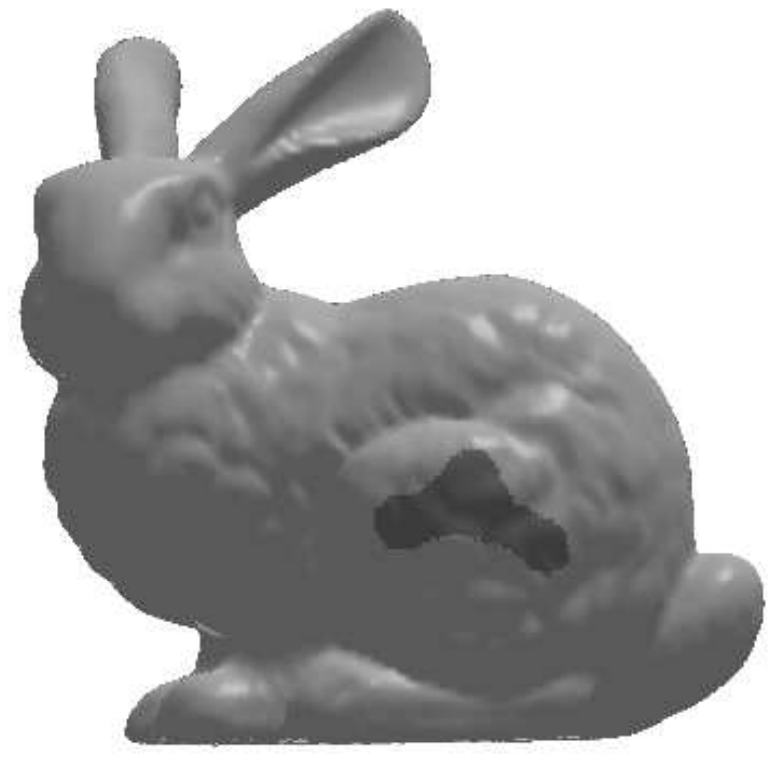}
\includegraphics[viewport = 140 280 430 550, width = 0.23\textwidth]{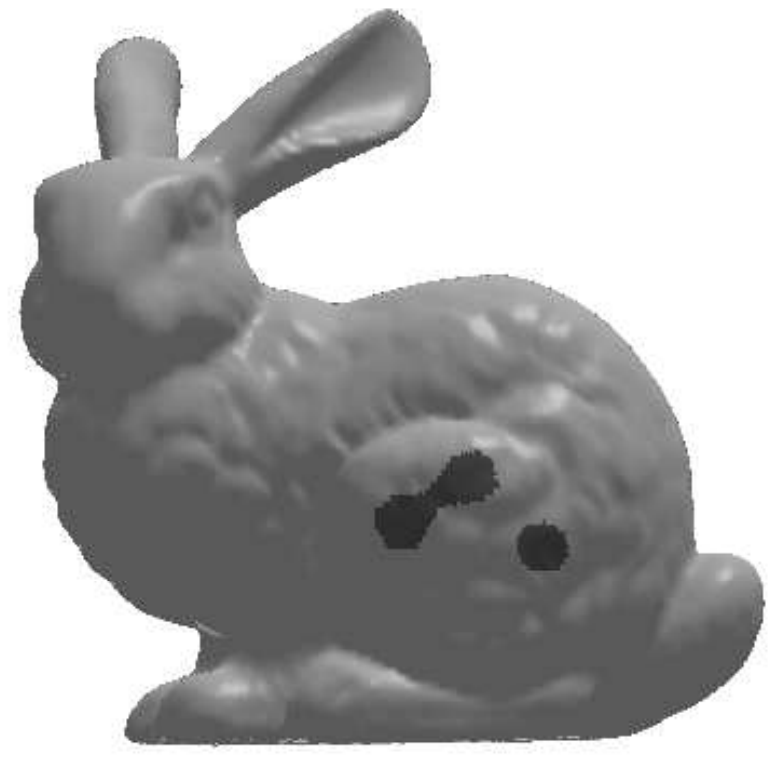}
\includegraphics[viewport = 140 280 430 550, width = 0.23\textwidth]{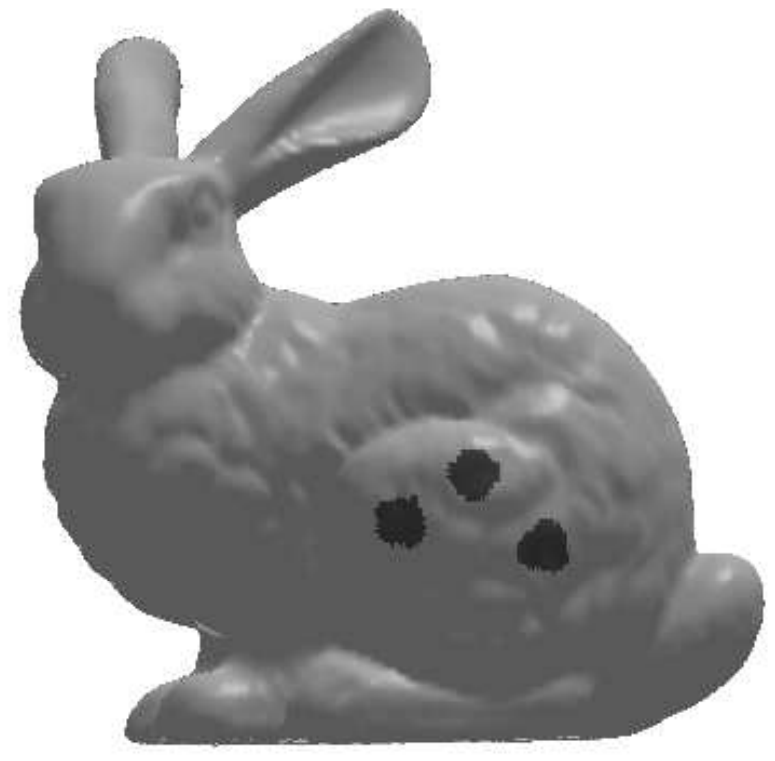}
\caption{Image segmentation of a gray-scaled image on the Stanford bunny. Demonstration of topology changes (splitting). First row: Original image and contours for $m=1, 100, 140, 150$. Second row: Piecewise constant approximation. The surface is from the Stanford Computer Graphics Laboratory, cf. \cite{Turk94}.}
\label{fig:bunny_balls}
\end{figure}

\begin{figure}[t]
\centering
\includegraphics[viewport = 140 280 430 550, width = 0.23\textwidth]{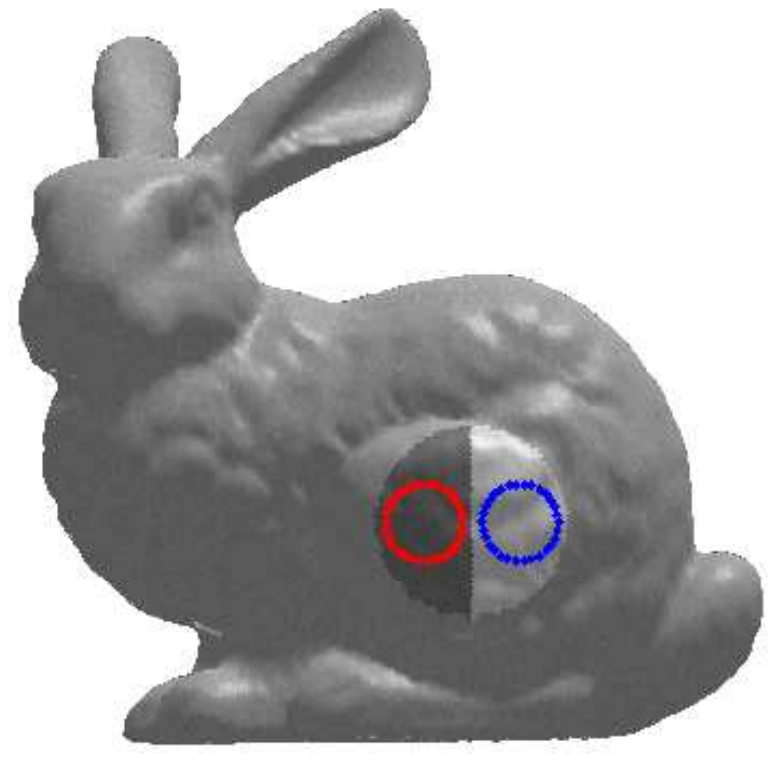}
\includegraphics[viewport = 140 280 430 550, width = 0.23\textwidth]{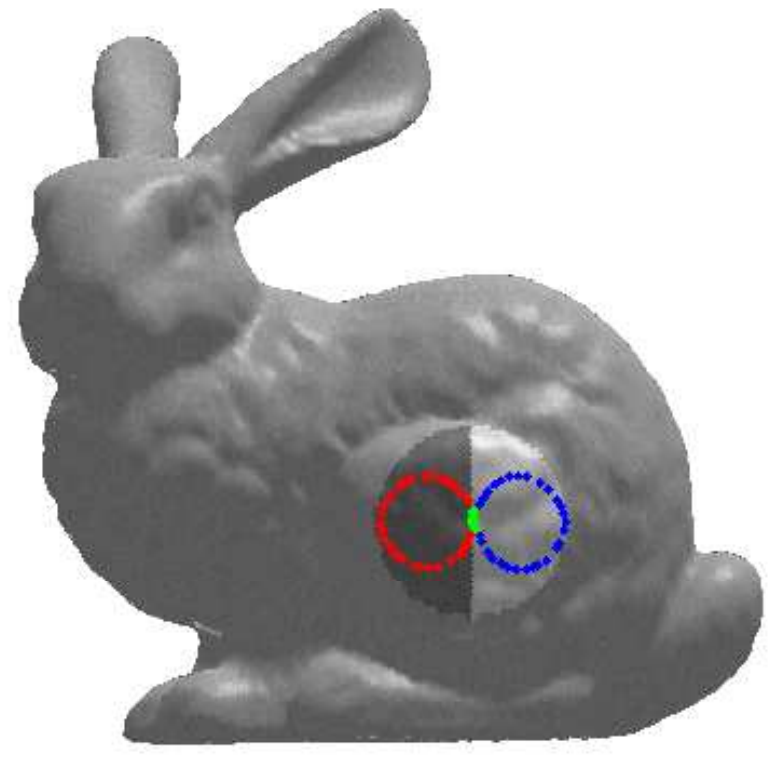}
\includegraphics[viewport = 140 280 430 550, width = 0.23\textwidth]{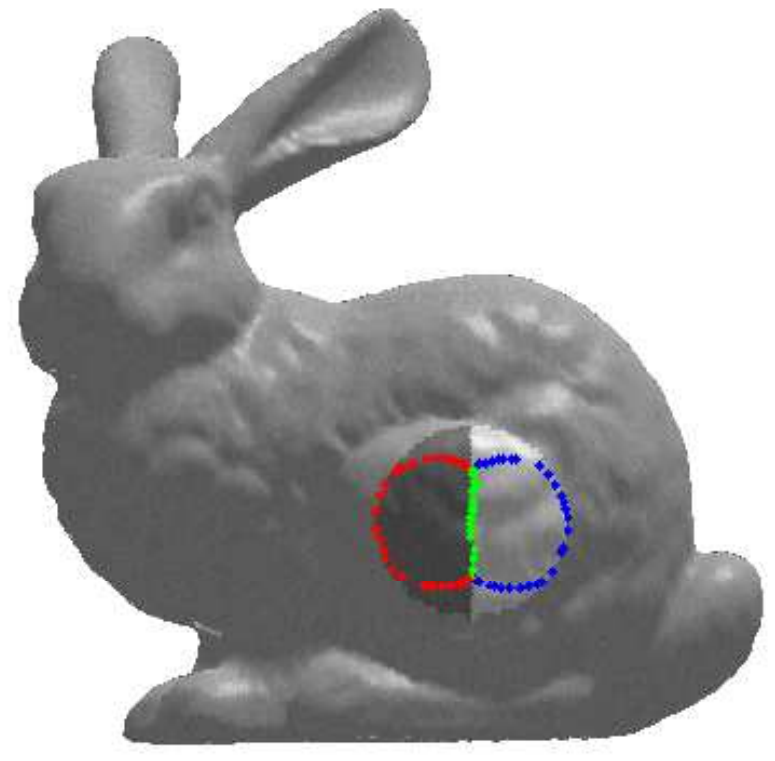}
\includegraphics[viewport = 140 280 430 550, width = 0.23\textwidth]{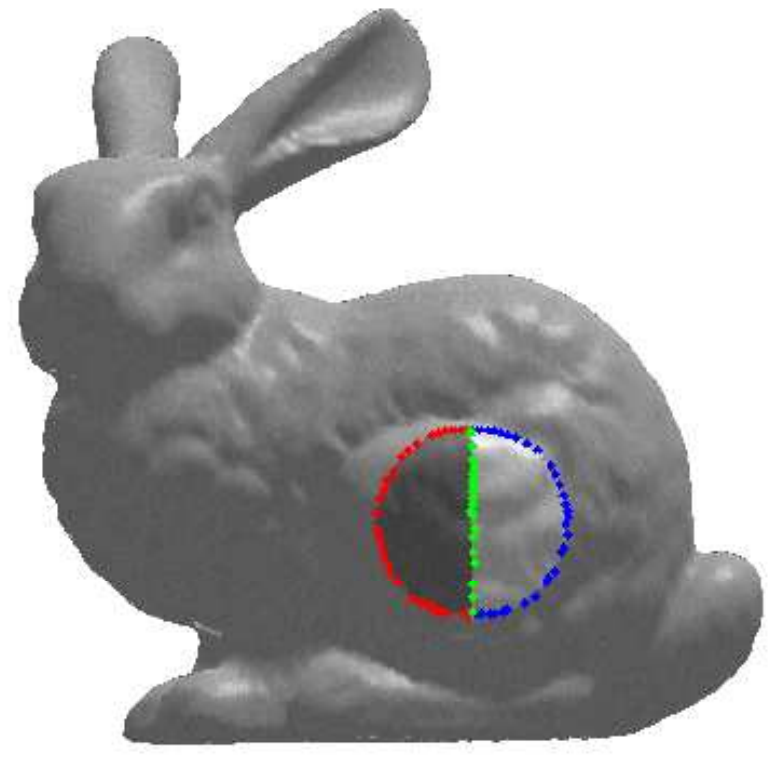}\\
\includegraphics[viewport = 140 280 430 550, width = 0.23\textwidth]{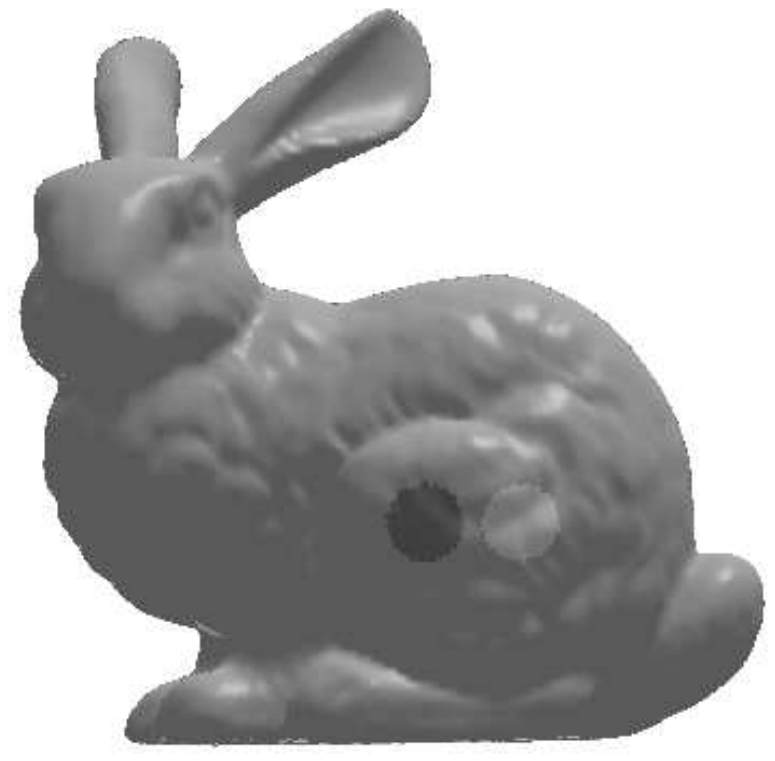}
\includegraphics[viewport = 140 280 430 550, width = 0.23\textwidth]{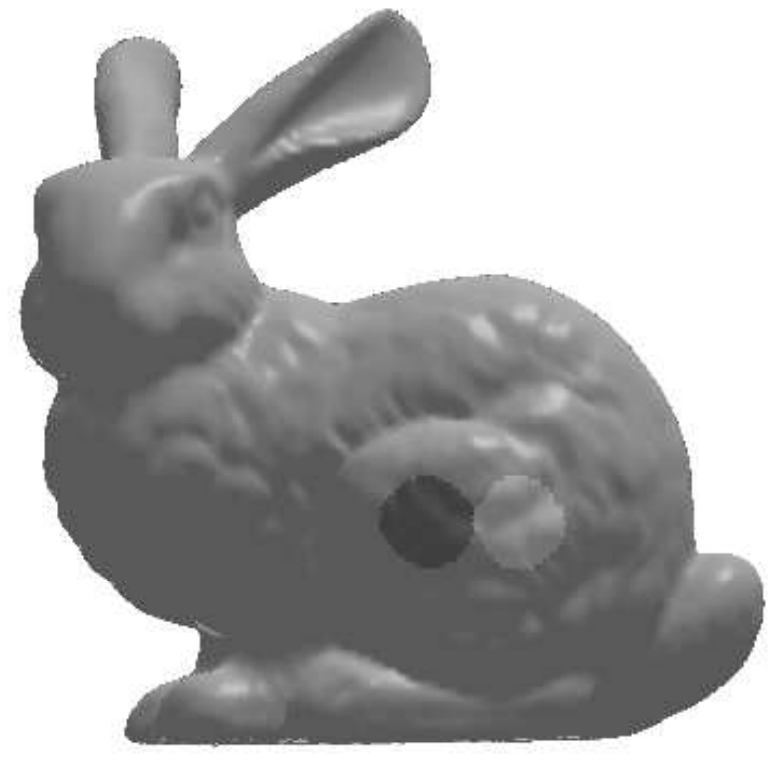}
\includegraphics[viewport = 140 280 430 550, width = 0.23\textwidth]{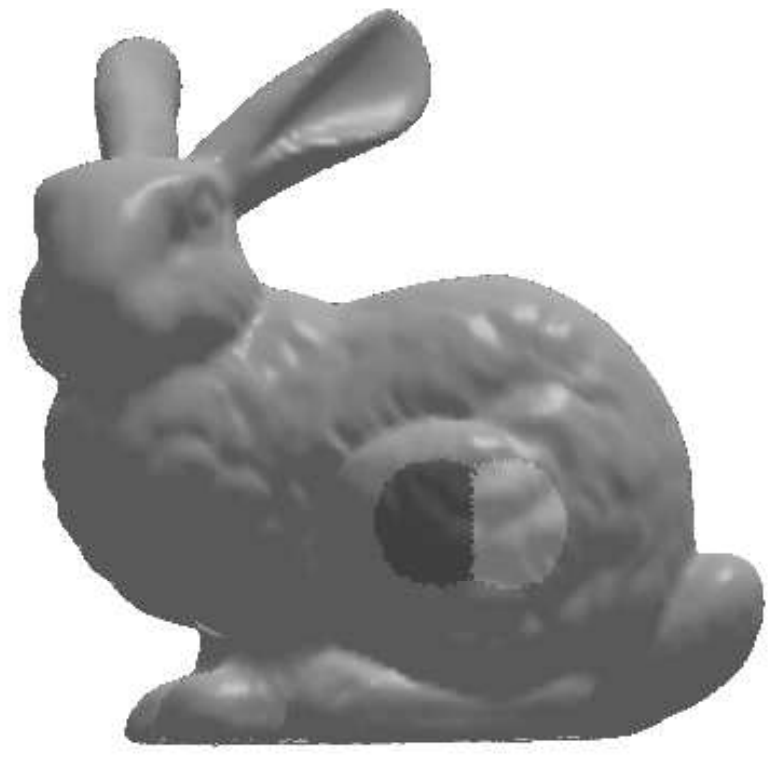}
\includegraphics[viewport = 140 280 430 550, width = 0.23\textwidth]{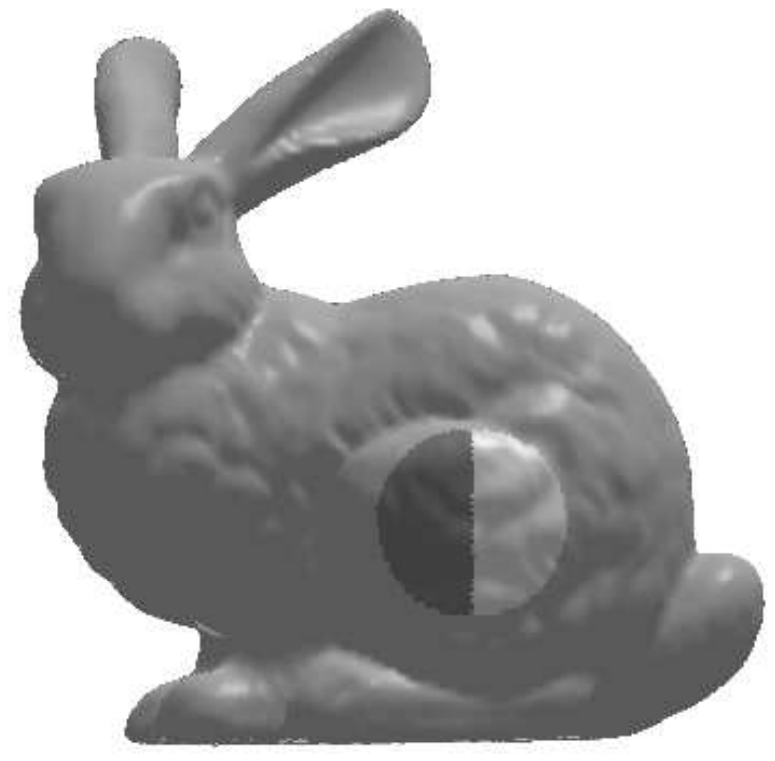}
\caption{Image segmentation of a gray-scaled image on the Stanford bunny. Curve network with triple junctions. First row: Original image and contours for $m=1, 42, 100, 250$. Second row: Piecewise constant approximation. The surface is from the Stanford Computer Graphics Laboratory, cf. \cite{Turk94}.}
\label{fig:bunny_triple}
\end{figure}

\begin{figure}
\centering
\includegraphics[viewport = 140 280 430 550, width = 0.23\textwidth]{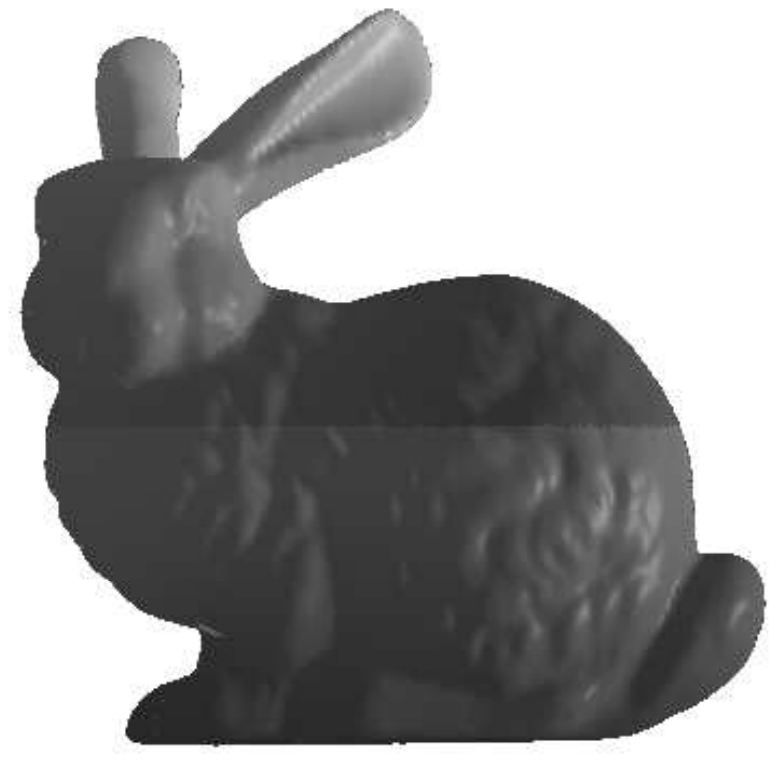}
\includegraphics[viewport = 140 280 430 550, width = 0.23\textwidth]{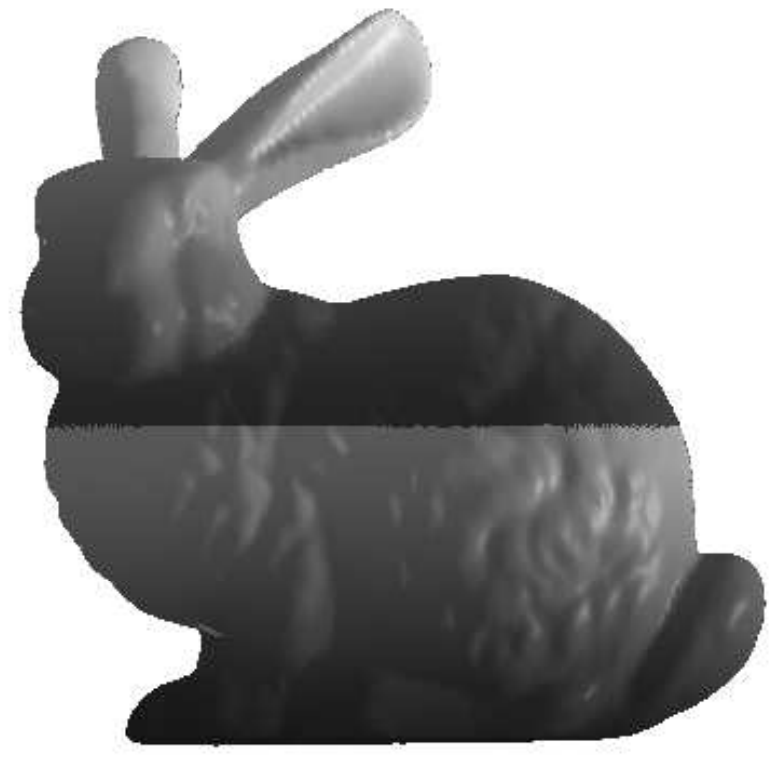}
\includegraphics[viewport = 140 280 430 550, width = 0.23\textwidth]{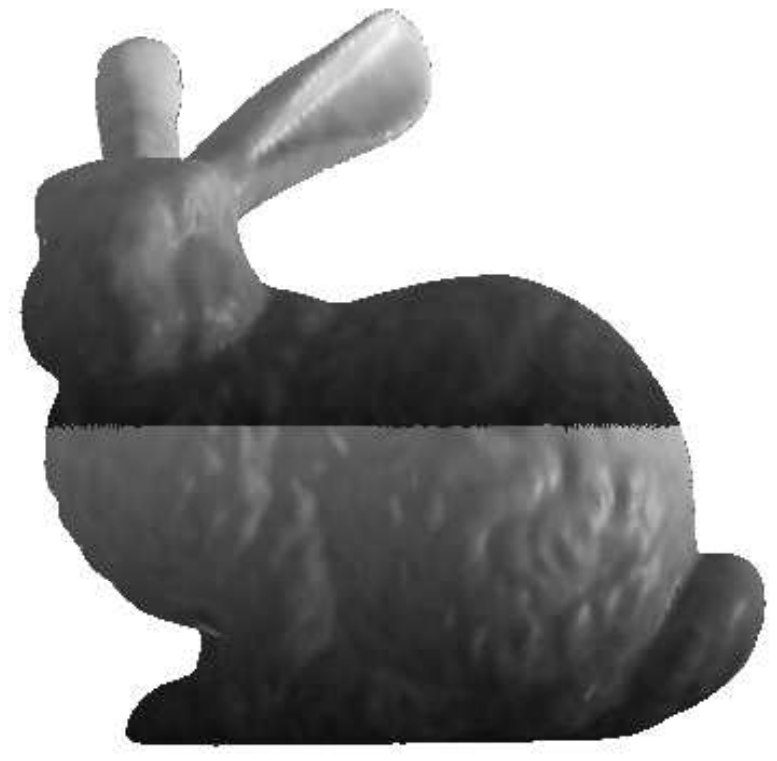}\\
\includegraphics[viewport = 140 280 430 550, width = 0.23\textwidth]{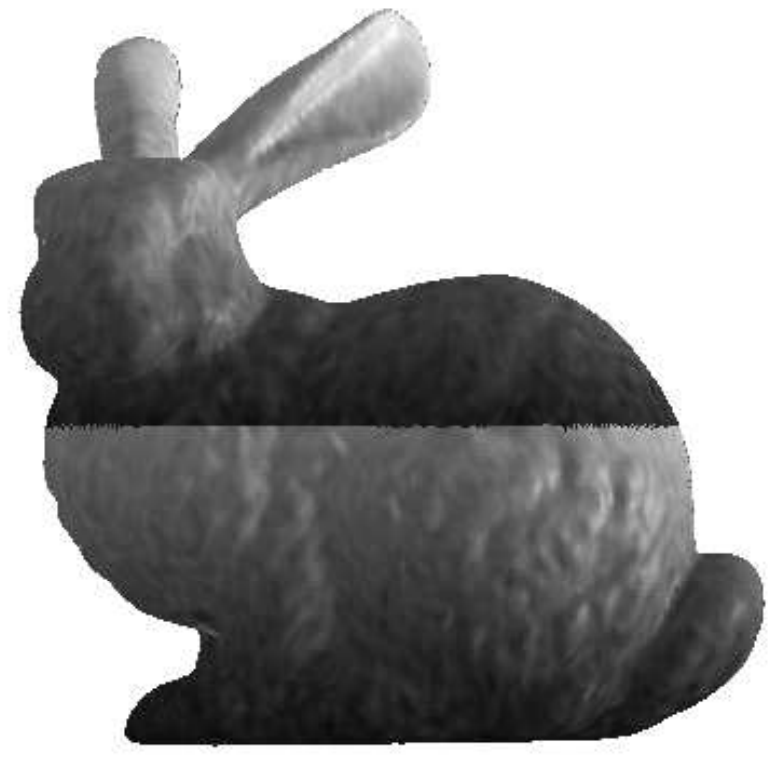}
\includegraphics[viewport = 140 280 430 550, width = 0.23\textwidth]{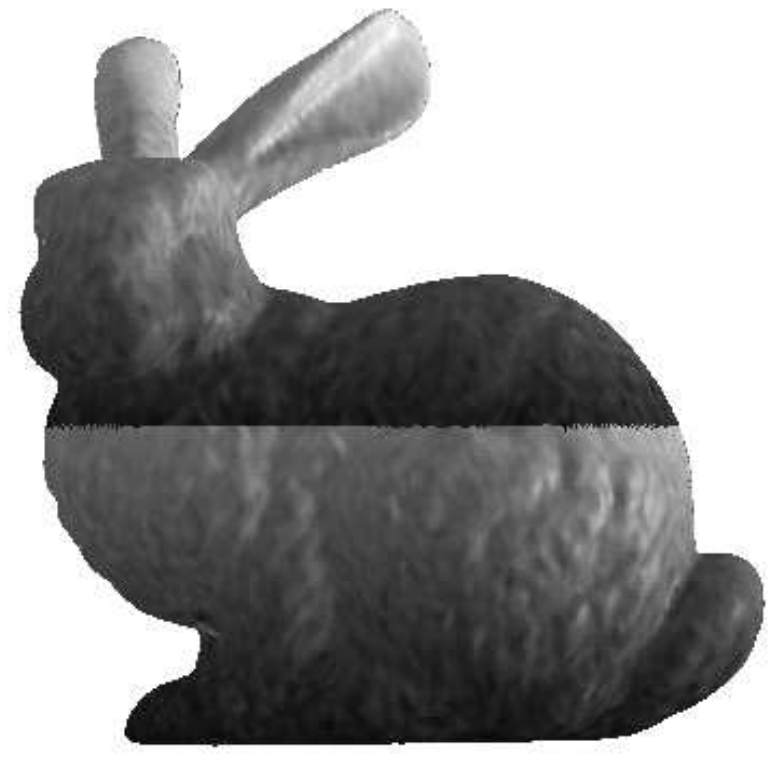}
\includegraphics[viewport = 140 280 430 550, width = 0.23\textwidth]{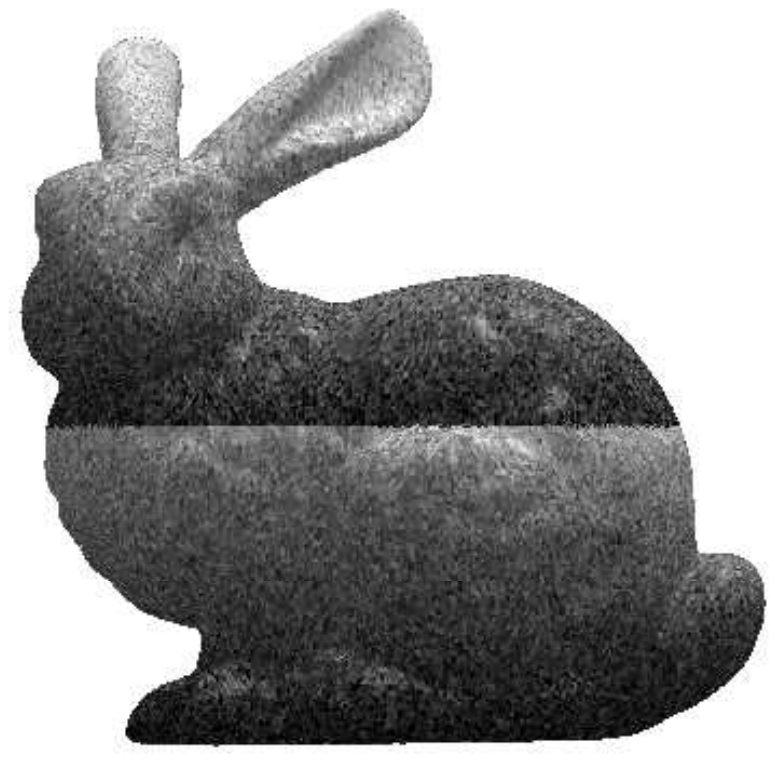}
\caption{Image denoising with edge enhancement. Subfigure 1-5 (row-wise): Denoising result using $\lambda = 0.1, 1, 100, 1000, 10000$.  Subfigure 6: Original image with noise. The surface is from the Stanford Computer Graphics Laboratory, cf. \cite{Turk94}.}
\label{fig:bunny_denosing}
\end{figure}

\paragraph{Images on the Stanford bunny}
We test the developed algorithm for image segmentation by considering artificial images on the Standford bunny (Stanford Computer Graphics Laboratory \cite{Turk94}). In a first experiment, we consider a gray-scaled image showing three dark discs. This example is similar to an experiment presented by Kr\"uger et al. \cite{Krueger08}, who use a level set method to solve a geodesic active contours model with a balloon force for images on surfaces, whereas we use a parametric method for a Chan-Vese like model for images on surfaces. 

Figure \ref{fig:bunny_balls} shows our image segmentation result using the developed direct, parametric approach for image segmentation. We use the parameters $\sigma = 2$ and $\lambda = 50$ to weight the curvature and external term. Let $\Delta t = \tau_m$ denote the time step size. The time step size is set to $\Delta t= 0.01$. This example demonstrates topology changes; in detail it shows how one initial closed curve is split up in three single curves. The contours at four different time steps and the corresponding piecewise constant approximation are presented in Figure \ref{fig:bunny_balls}. Of course, a level set technique as used in  \cite{Krueger08} can handle splitting automatically, whereas we need to detect the change in topology explicitly using the method described in Section \ref{subsec:num_top_changes}. However, our method to detect topology changes is efficient, since it has a computational effort of $\mathcal{O}(N)$, where $N$ is the number of node points of the polygonal curves. 

In a second experiment, we demonstrate the creation and handling of triple junctions for a curve network on the Stanford bunny, cf. Figure \ref{fig:bunny_triple}. For this experiment, we set $\sigma = 1$, $\lambda = 40$ and $\Delta t= 0.01$. The possibility of triple junctions is not considered in \cite{Krueger08}.

Further, we apply the image denoising method described in Section \ref{subsec:image_restoration} and \ref{subsec:finitedifference_imagedenoising_geodesic}, cf. Figure~\ref{fig:bunny_denosing}. Since the term $\Delta_\M u_k$ in \eqref{eq:diffusion_eq_geodesic_strong} is weighted with $1/\lambda$, the denoised image is close to the original image, the larger $\lambda$ is chosen, cf. Figure \ref{fig:bunny_denosing}. Setting $\lambda = 1000$ or $\lambda=10000$, the noise is not completely smoothed out. Setting $\lambda = 0.1$ the resulting denoised version is close to a piecewise constant image approximation. Setting $\lambda = 100$ results in a good denoised image. 

\paragraph{Image on a torus}
We consider an artificial image on a torus to demonstrate that the method can be applied on surfaces of arbitrary topology type (for example arbitrary genus). Figure \ref{fig:torus_objects} shows a torus with a color image from different viewing angles and the curves at different iteration steps during the evolution. The RGB color space is used for the segmentation. As weighting parameter for the curvature $\sigma = 1$ is used, and all three components of the color are weighted equally with $\lambda_1=\lambda_2=\lambda_3=20$.
The time step size in this experiment is set to $\Delta t = 0.0001$. 

Two different topology changes occur in this example: Around $m=335$, two triple junctions and a new curve are created. Shortly after $m=422$, one blue curve splits into two single curves. 

\begin{figure}
\centering  
\includegraphics[trim = 20mm 20mm 20mm 20mm,clip, width = 0.2\textwidth]{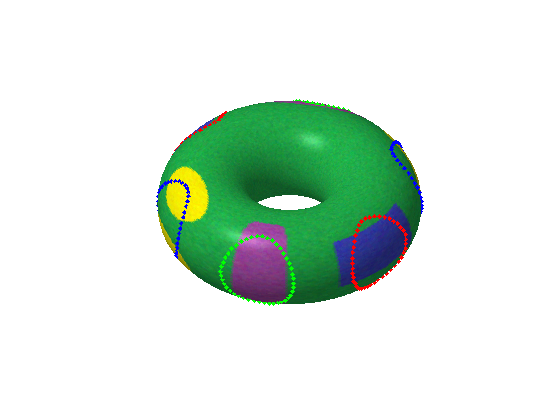}
\includegraphics[trim = 20mm 20mm 20mm 20mm,clip, width = 0.2\textwidth]{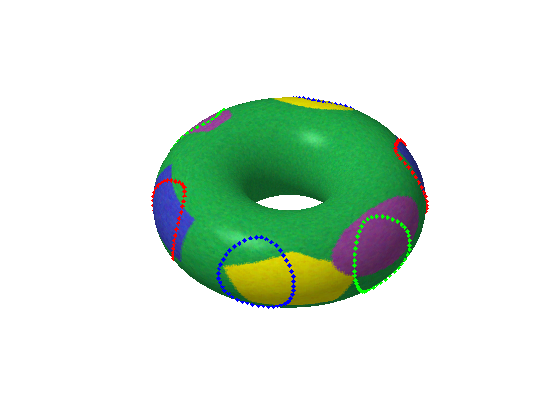}
\includegraphics[trim = 20mm 22mm 20mm 22mm,clip, width = 0.216\textwidth]{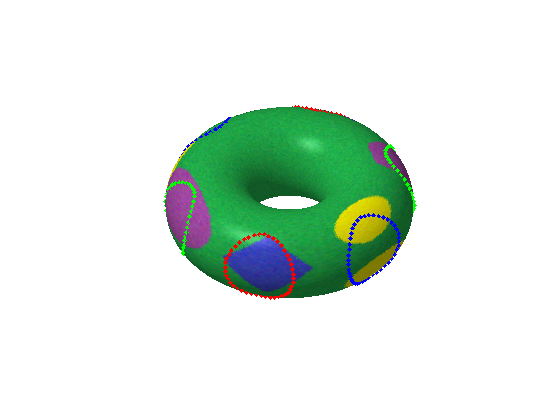}\\[2ex]
\includegraphics[trim = 20mm 20mm 20mm 20mm,clip, width = 0.2\textwidth]{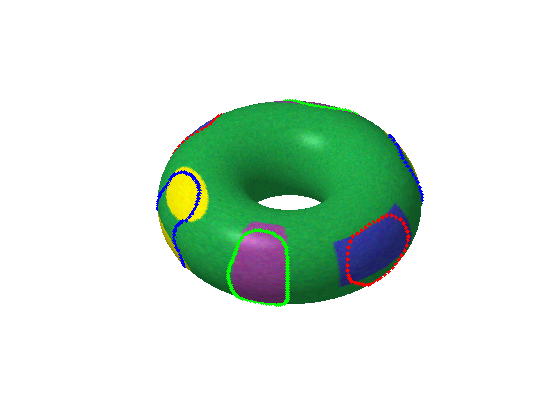}
\includegraphics[trim = 20mm 20mm 20mm 20mm,clip, width = 0.2\textwidth]{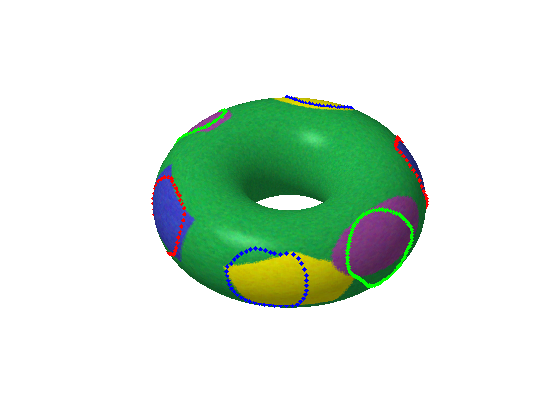}
\includegraphics[trim = 20mm 22mm 20mm 22mm,clip, width = 0.216\textwidth]{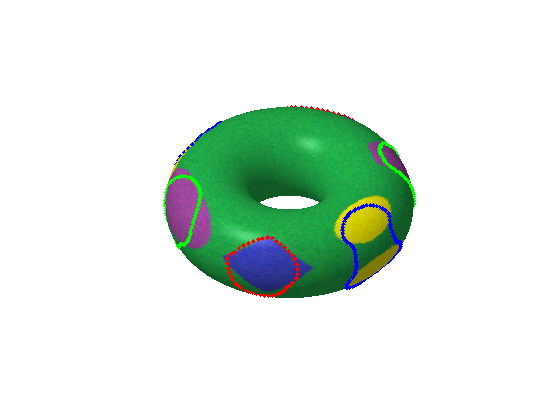}\\[2ex]
\includegraphics[trim = 20mm 20mm 20mm 20mm,clip, width = 0.2\textwidth]{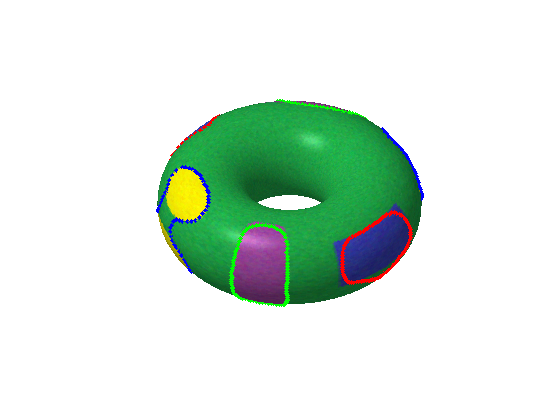}
\includegraphics[trim = 20mm 20mm 20mm 20mm,clip, width = 0.2\textwidth]{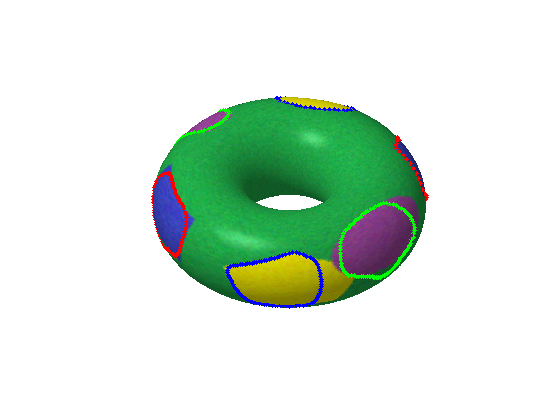}
\includegraphics[trim = 20mm 22mm 20mm 22mm,clip, width = 0.216\textwidth]{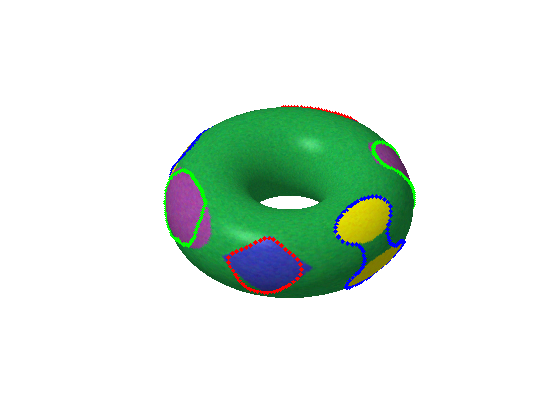}\\[2ex]
\includegraphics[trim = 20mm 20mm 20mm 20mm,clip, width = 0.2\textwidth]{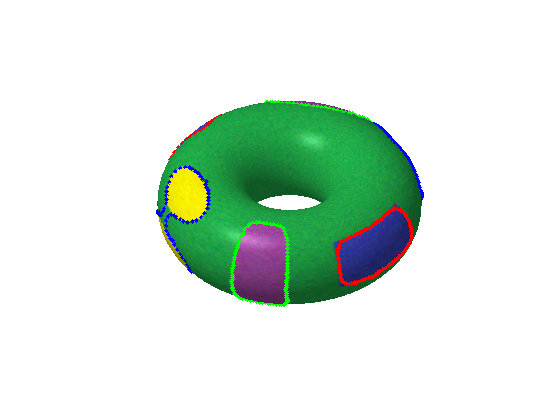}
\includegraphics[trim = 20mm 20mm 20mm 20mm,clip, width = 0.2\textwidth]{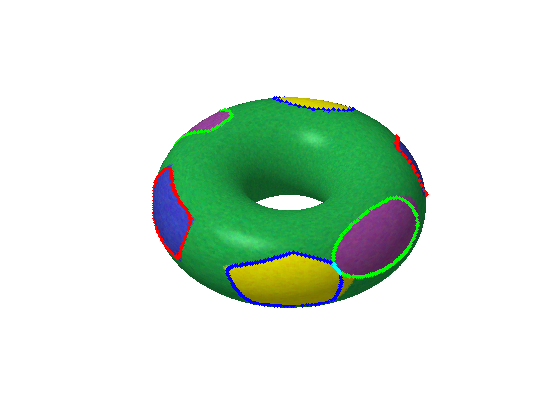}
\includegraphics[trim = 20mm 22mm 20mm 22mm,clip, width = 0.216\textwidth]{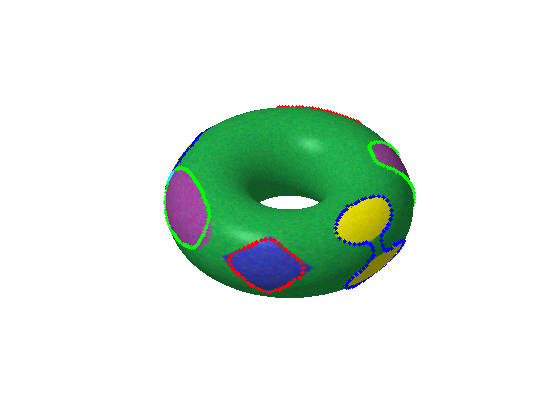}\\[2ex]
\includegraphics[trim = 20mm 20mm 20mm 20mm,clip, width = 0.2\textwidth]{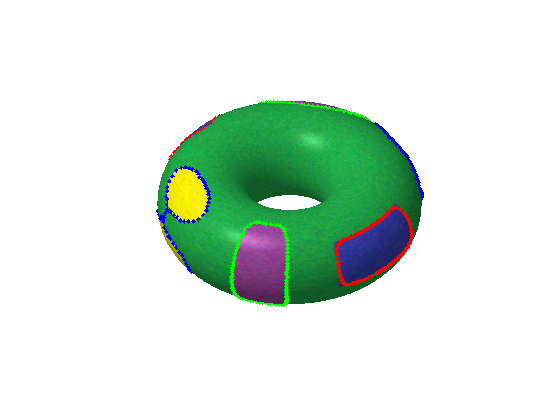}
\includegraphics[trim = 20mm 20mm 20mm 20mm,clip, width = 0.2\textwidth]{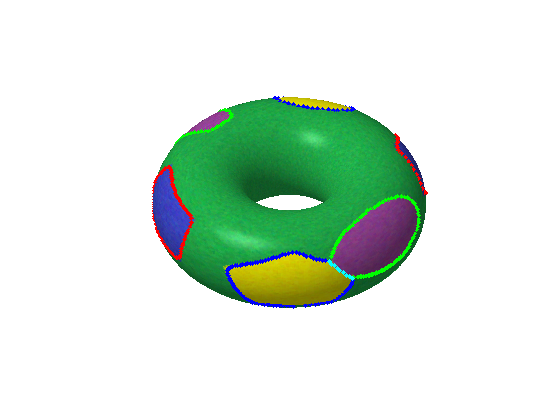}
\includegraphics[trim = 20mm 22mm 20mm 22mm,clip, width = 0.216\textwidth]{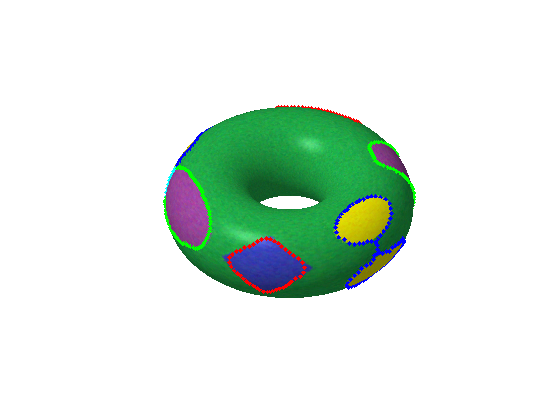}\\[2ex]
\includegraphics[trim = 20mm 20mm 20mm 20mm,clip, width = 0.2\textwidth]{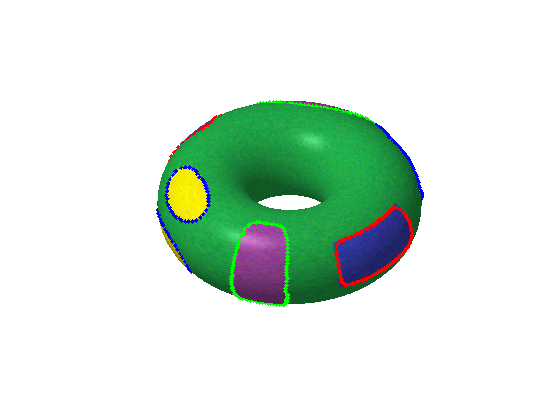}
\includegraphics[trim = 20mm 20mm 20mm 20mm,clip, width = 0.2\textwidth]{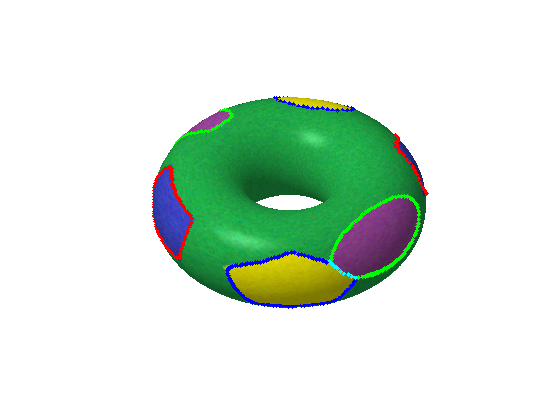}
\includegraphics[trim = 20mm 22mm 20mm 22mm,clip, width = 0.216\textwidth]{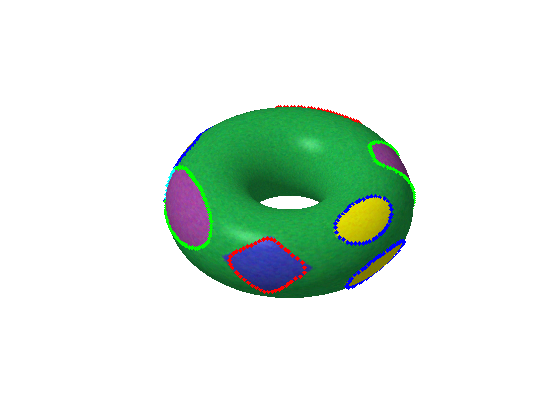}
\caption{Segmentation of an artificial image showing different objects on a torus. First - sixth row: Original image and contours for $m=1, 100, 200, 335, 422, 600$. First - third column: Different viewing angles.}
\label{fig:torus_objects}
\end{figure}

\subsection{Real Images}
\paragraph{Lip contour segmentation}
We consider an application where lip contours should be detected on given face image data. Figure \ref{fig:lip_contours} shows the results where the image segmentation algorithm is applied on three sample images of the 3D face scans from the 3D Basel Face Model (BFM) published by the Computer Science department of the University of Basel\footnote{\url{http://faces.cs.unibas.ch/bfm/main.php?nav=1-0&id=basel_face_model}}, see also \cite{Paysen09}. The images are segmented using the chromaticity-brightness color space with $\sigma = 0.25$, $\lambda_C = 200$, $\lambda_B = 20$ and $\Delta t = 0.001$. The initial contour is a closed curve placed around the lips. Applying our algorithm for image segmentation, we obtain the final lip contours.

\begin{figure}[t]
\centering
\includegraphics[viewport = 170 300 410 530, width = 0.2\textwidth]{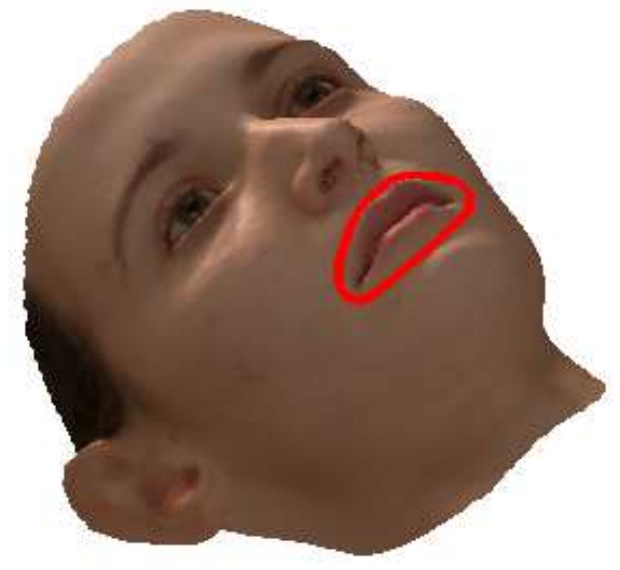}
\includegraphics[viewport = 170 300 410 530, width = 0.2\textwidth]{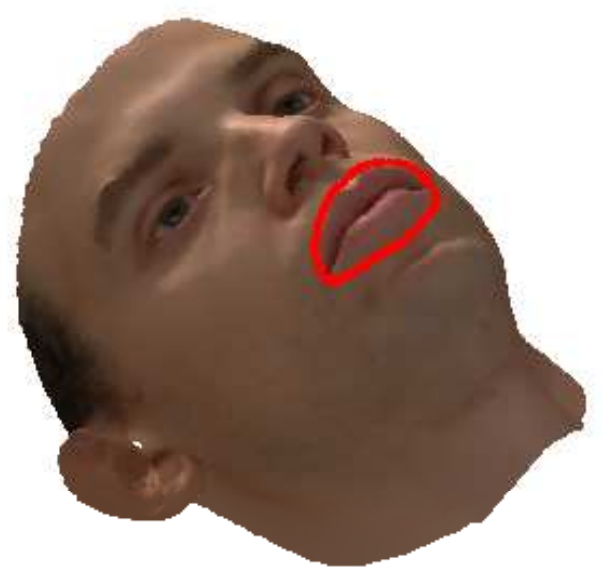}
\includegraphics[viewport = 170 300 410 530, width = 0.2\textwidth]{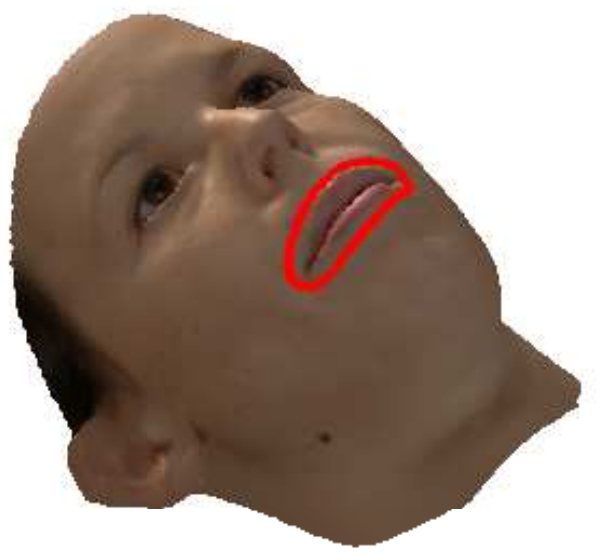}\\
\includegraphics[viewport = 170 300 410 530, width = 0.2\textwidth]{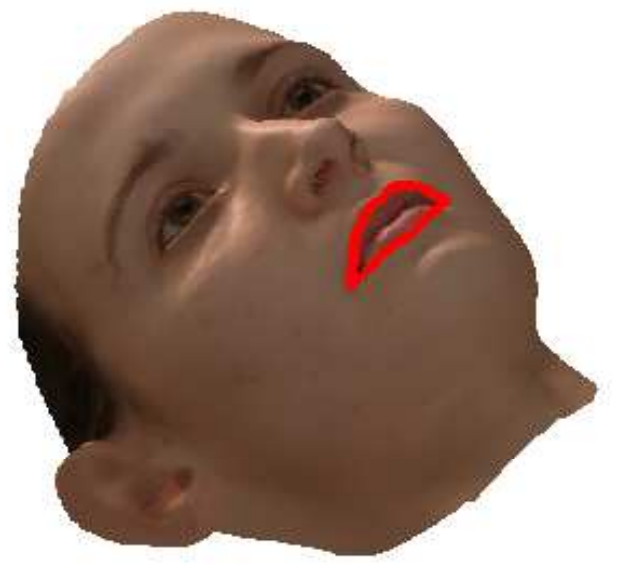}
\includegraphics[viewport = 170 300 410 530, width = 0.2\textwidth]{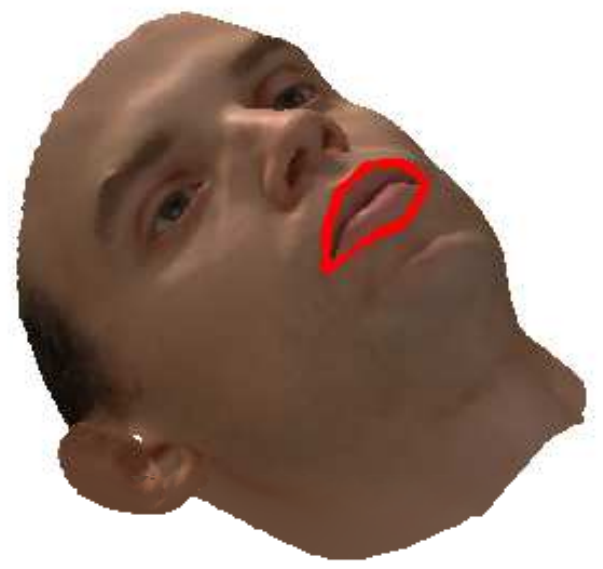}
\includegraphics[viewport = 170 300 410 530, width = 0.2\textwidth]{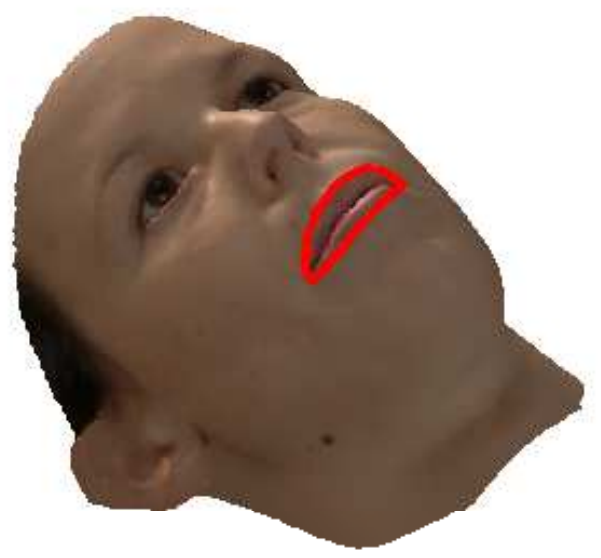}\\
\caption{Lip contour detection with a two-phase segmentation. Initial (first row) and final (second row) contours. The surfaces and images are from the 3D Basel Face Model (BFM) of the Computer Science department of the University of Basel, cf. \cite{Paysen09}.}
\label{fig:lip_contours}
\end{figure}

\paragraph{Processing of global Earth observation data}
Another application is the processing of Earth observation data. Global Earth observation data can be interpreted as an image given on a sphere. 
We apply the image segmentation and denoising method on data from the NASA Earth Observation data set\footnote{\url{http://neo.sci.gsfc.nasa.gov/}}, cf. \cite{NEO}.   
We segment an image showing outgoing longwave radiation\footnote{Imagery by Jesse Allen, NASA Earth Observatory, based on FLASHFlux data. FLASHFlux data are produced using CERES observations convolved with MODIS measurements from both the Terra and Aqua satellite. Data provided by the FLASHFlux team, NASA Langley Research Center.}, see Figure \ref{fig:earth_longwave_radiation1} - \ref{fig:earth_longwave_radiation3}. The colors represent the amount of outgoing longwave radiation leaving the Earth's atmosphere in one month (here: January 2014). Yellow and orange color represent greater heat emission (around $300$ to $350\,\mathrm{Wm}^{-2}$); purple and blue color represent intermediate emissions (around $200\,\mathrm{Wm}^{-2}$).

Figure \ref{fig:earth_longwave_radiation1} presents the given image data and the contours at different time steps (rows), observed from different viewing angles (columns). 
For the segmentation we used the RGB color space and set the weighting parameters for the curvature and the external term to $\sigma = 1$ and $\lambda_1 = \lambda_2 = \lambda_3 = \lambda = 50$. As time step size we set $\Delta t = 0.0005$. Several topology changes (splitting and merging) occur during the segmentation (cf. e.g. $m=75$). Figure \ref{fig:earth_longwave_radiation2} shows the corresponding piecewise constant approximations. Figure \ref{fig:earth_longwave_radiation3} presents the result of the postprocessing image restoration using the parameters $\lambda = 100$, $\lambda = 1000$ and $\lambda = 10000$. 

\begin{figure}[t]
\centering 
\includegraphics[trim = 10mm 2mm 10mm 2mm,clip, width = 0.25 \textwidth]{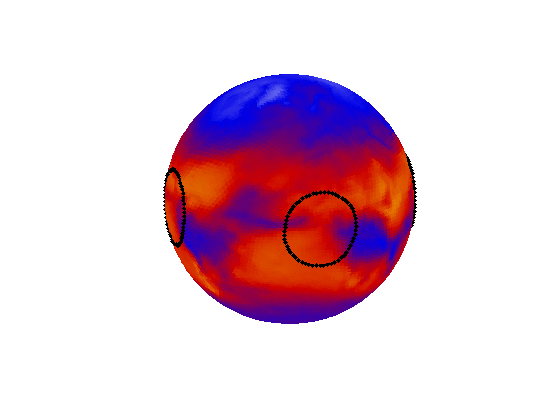}
\includegraphics[trim = 10mm 0mm 10mm 0mm,clip, width = 0.235\textwidth]{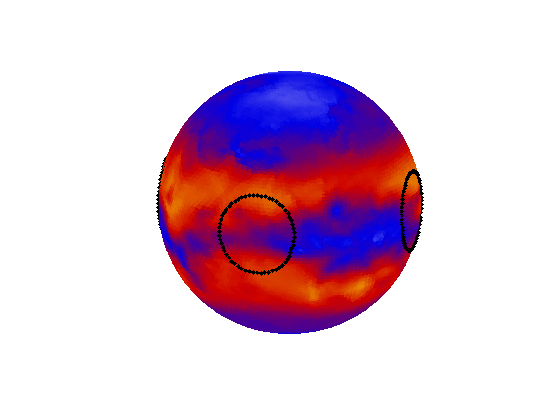}
\includegraphics[trim = 10mm 3mm 10mm 3mm,clip, width = 0.255\textwidth]{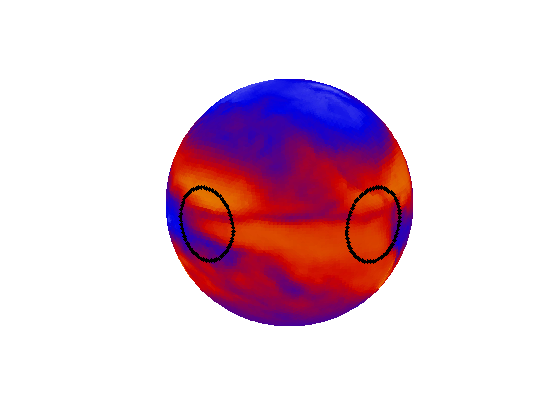}\\[0ex]
\includegraphics[trim = 10mm 2mm 10mm 2mm,clip, width = 0.25 \textwidth]{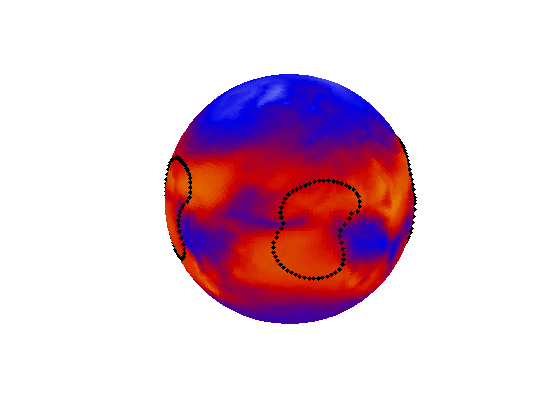}
\includegraphics[trim = 10mm 0mm 10mm 0mm,clip, width = 0.235\textwidth]{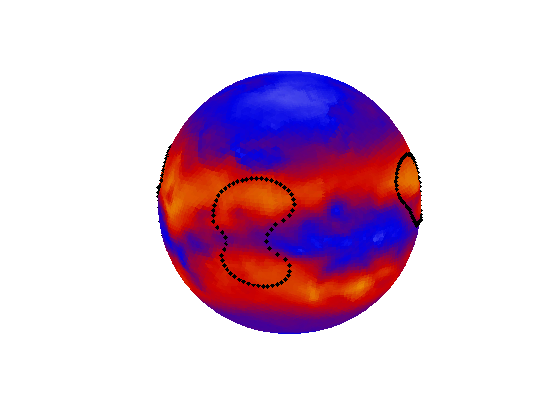}
\includegraphics[trim = 10mm 3mm 10mm 3mm,clip, width = 0.255\textwidth]{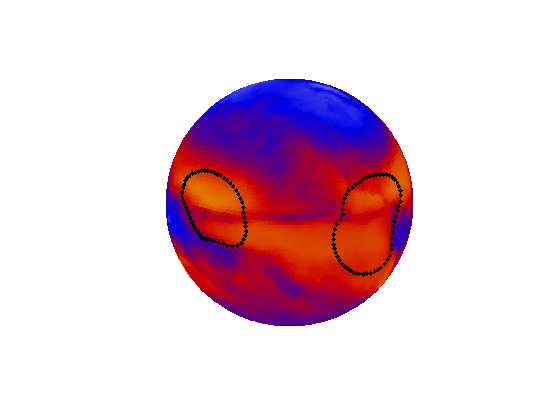}\\[0ex]
\includegraphics[trim = 10mm 2mm 10mm 2mm,clip, width = 0.25 \textwidth]{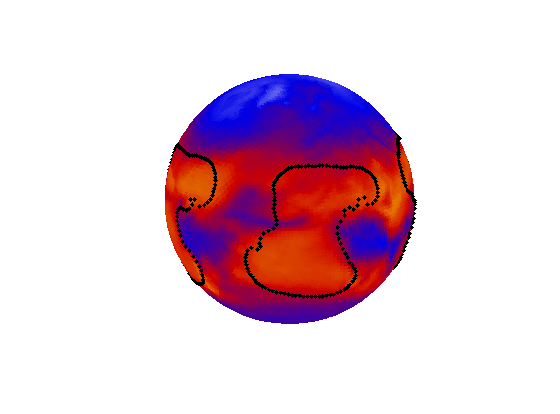}
\includegraphics[trim = 10mm 0mm 10mm 0mm,clip, width = 0.235\textwidth]{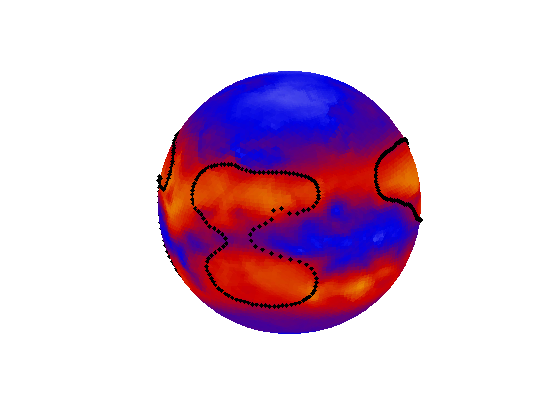}
\includegraphics[trim = 10mm 3mm 10mm 3mm,clip, width = 0.255\textwidth]{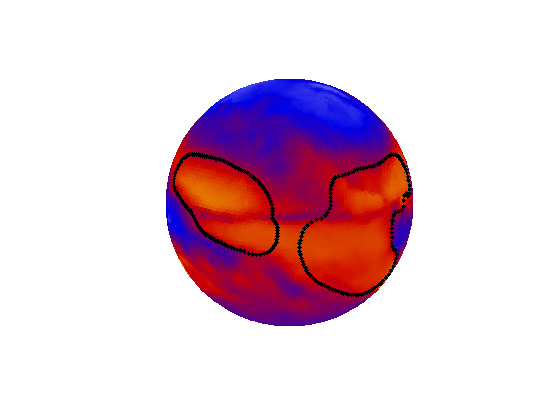}\\[0ex]
\includegraphics[trim = 10mm 2mm 10mm 2mm,clip, width = 0.25 \textwidth]{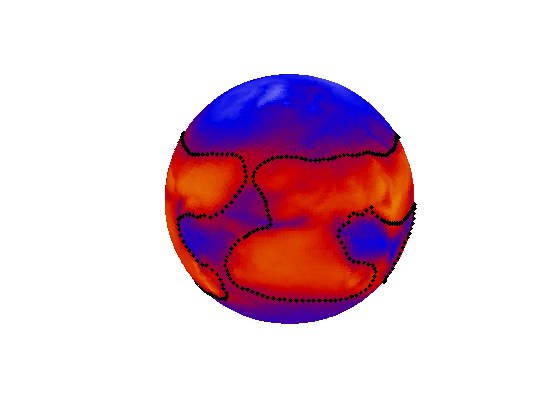}
\includegraphics[trim = 10mm 0mm 10mm 0mm,clip, width = 0.235\textwidth]{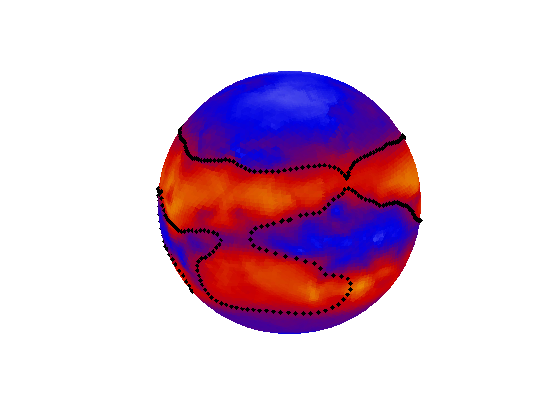}
\includegraphics[trim = 10mm 3mm 10mm 3mm,clip, width = 0.255\textwidth]{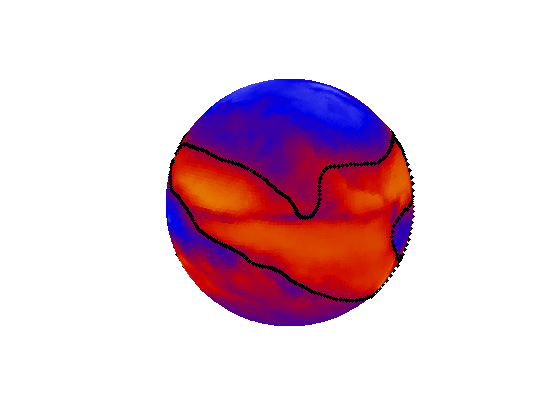}\\[0ex]
\includegraphics[trim = 10mm 2mm 10mm 2mm,clip, width = 0.25 \textwidth]{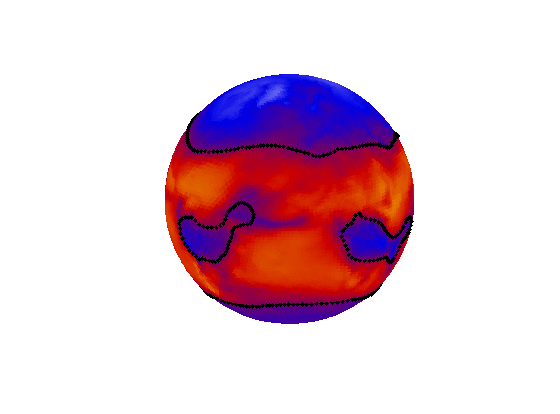}
\includegraphics[trim = 10mm 0mm 10mm 0mm,clip, width = 0.235\textwidth]{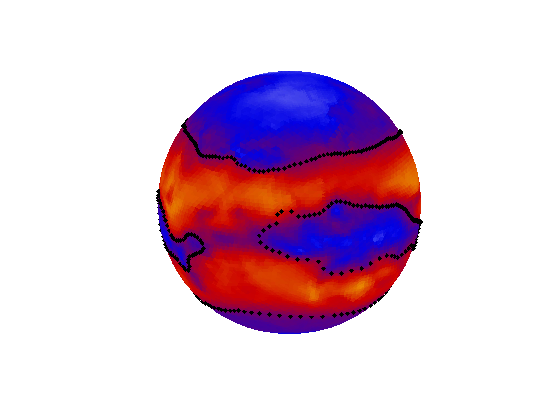}
\includegraphics[trim = 10mm 3mm 10mm 3mm,clip, width = 0.255\textwidth]{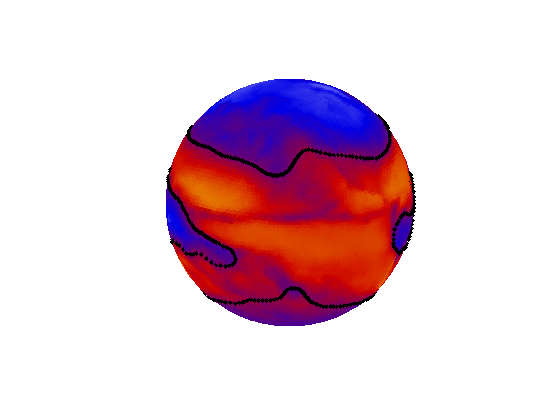}
\caption{Segmentation of longwave radiation data given on the Earth's surface. First - fifth row: Original image  and contours for $m=1, 20, 50, 75, 150$. First - third column: Different viewing angles. The original image is from the NASA Earth Observation data set, \cite{NEO}.}
\label{fig:earth_longwave_radiation1}
\end{figure}

\begin{figure}[t]
\centering
\includegraphics[trim = 10mm 2mm 10mm 2mm,clip, width = 0.25 \textwidth]{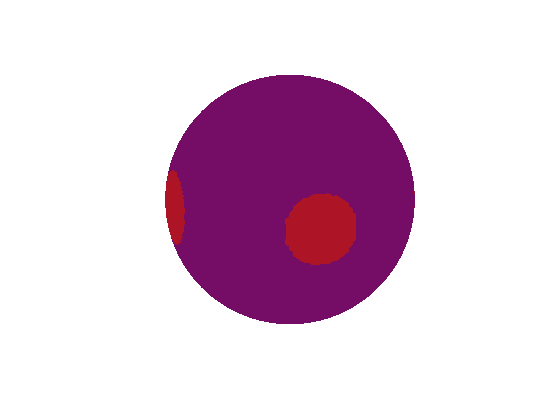}
\includegraphics[trim = 10mm 0mm 10mm 0mm,clip, width = 0.235\textwidth]{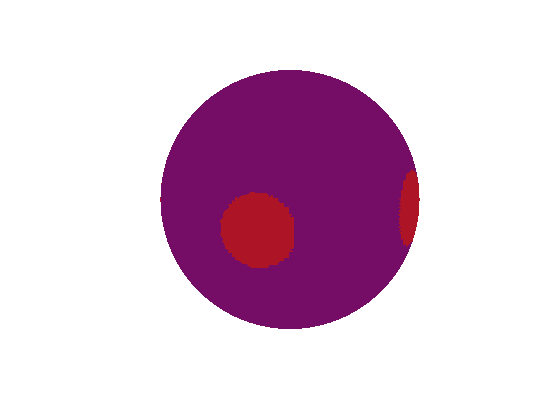}
\includegraphics[trim = 10mm 3mm 10mm 3mm,clip, width = 0.255\textwidth]{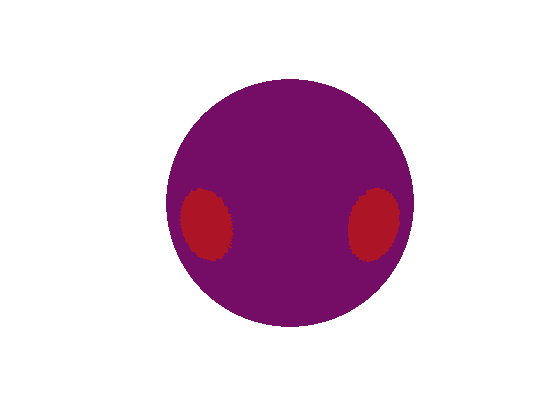}\\[0ex]
\includegraphics[trim = 10mm 2mm 10mm 2mm,clip, width = 0.25 \textwidth]{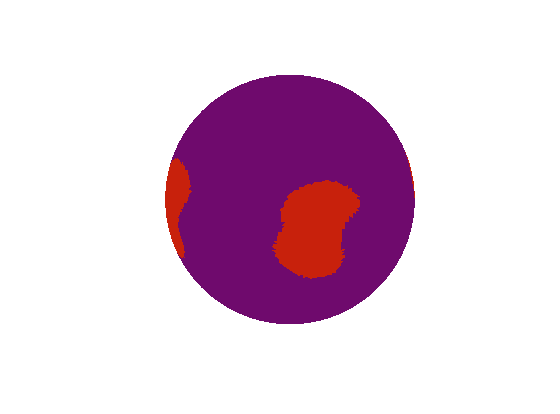}
\includegraphics[trim = 10mm 0mm 10mm 0mm,clip, width = 0.235\textwidth]{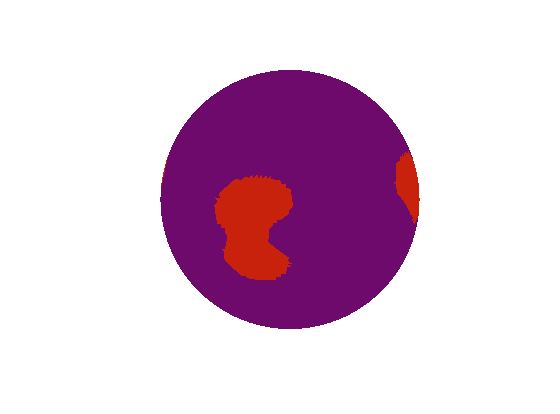}
\includegraphics[trim = 10mm 3mm 10mm 3mm,clip, width = 0.255\textwidth]{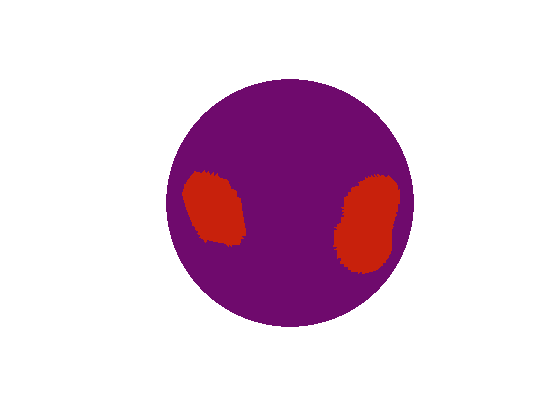}\\[0ex]
\includegraphics[trim = 10mm 2mm 10mm 2mm,clip, width = 0.25 \textwidth]{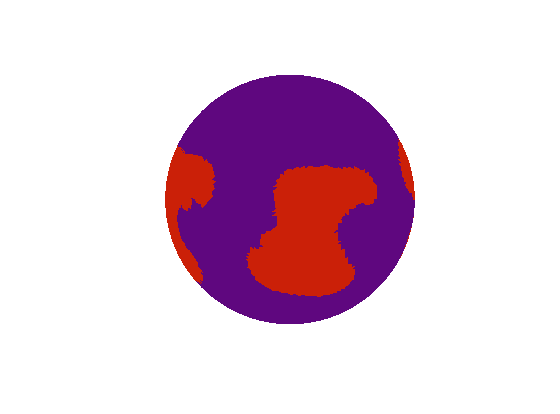}
\includegraphics[trim = 10mm 0mm 10mm 0mm,clip, width = 0.235\textwidth]{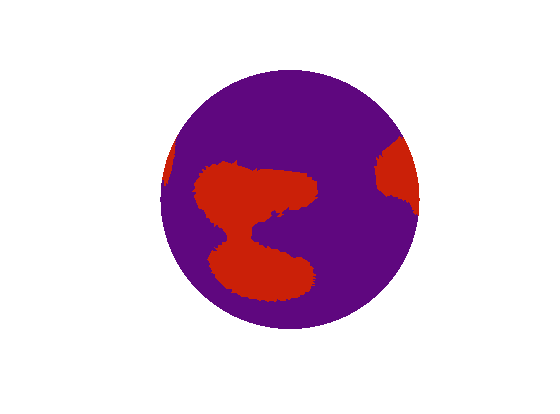}
\includegraphics[trim = 10mm 3mm 10mm 3mm,clip, width = 0.255\textwidth]{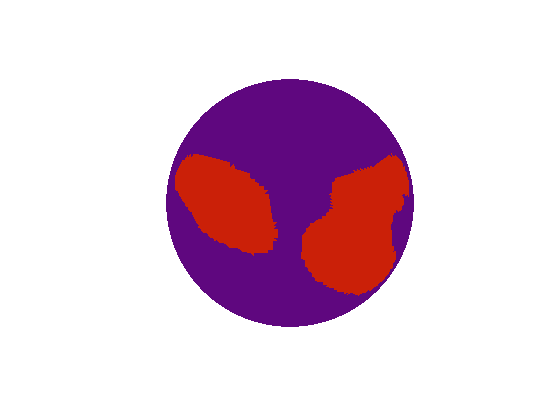}\\[0ex]
\includegraphics[trim = 10mm 2mm 10mm 2mm,clip, width = 0.25 \textwidth]{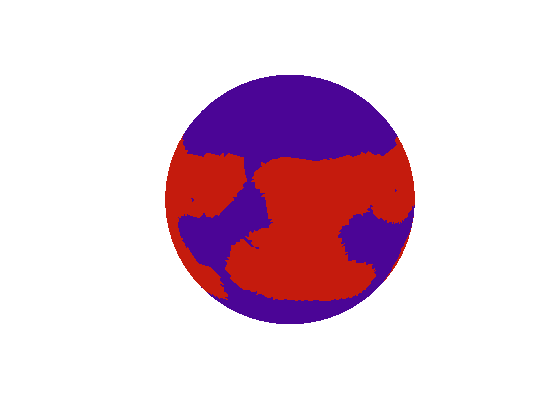}
\includegraphics[trim = 10mm 0mm 10mm 0mm,clip, width = 0.235\textwidth]{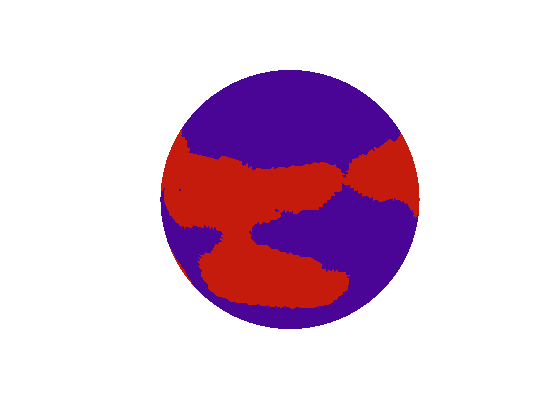}
\includegraphics[trim = 10mm 3mm 10mm 3mm,clip, width = 0.255\textwidth]{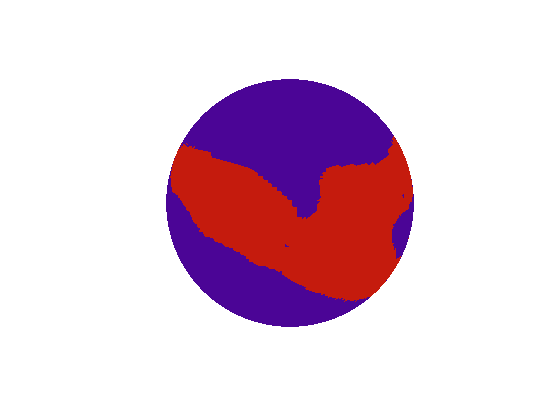}\\[0ex]
\includegraphics[trim = 10mm 2mm 10mm 2mm,clip, width = 0.25 \textwidth]{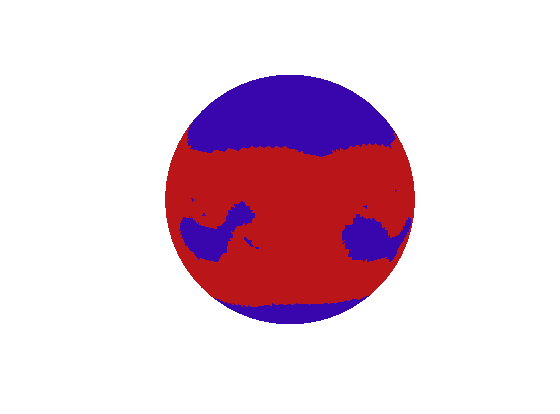}
\includegraphics[trim = 10mm 0mm 10mm 0mm,clip, width = 0.235\textwidth]{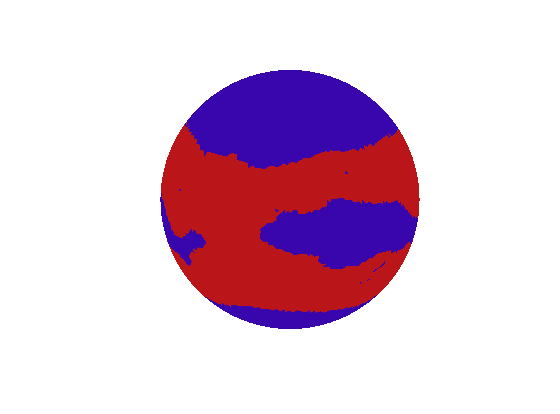}
\includegraphics[trim = 10mm 3mm 10mm 3mm,clip, width = 0.255\textwidth]{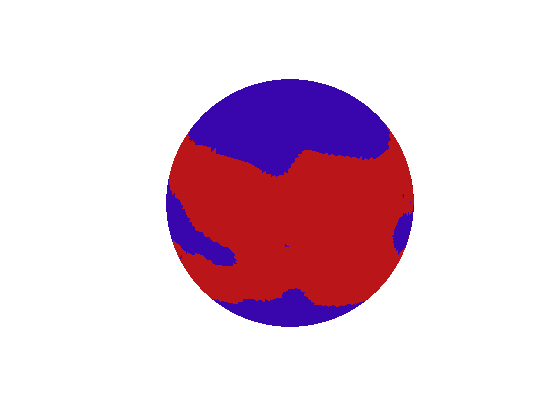}
\caption{Segmentation of longwave radiation data given on the Earth's surface. First - fifth row: Piecewise constant approximation for $m=1, 20, 50, 75, 150$. First - third column: Different viewing angles. The original image is from the NASA Earth Observation data set, \cite{NEO}.}
\label{fig:earth_longwave_radiation2}
\end{figure}

\begin{figure}[t]
\centering
\includegraphics[trim = 10mm 2mm 10mm 2mm,clip, width = 0.25 \textwidth]{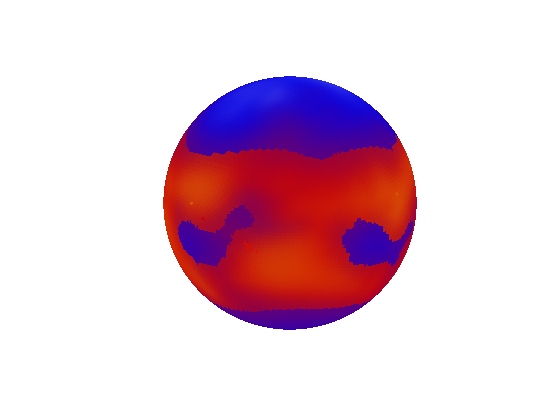}
\includegraphics[trim = 10mm 0mm 10mm 0mm,clip, width = 0.235\textwidth]{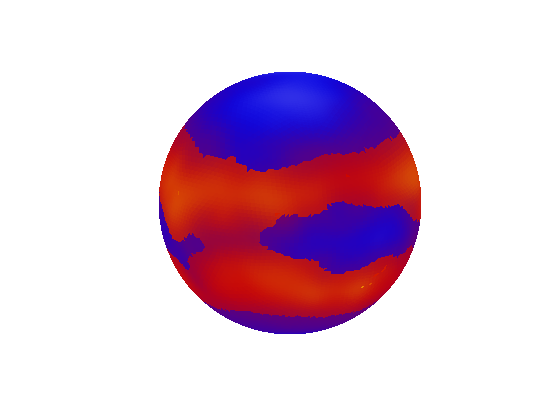}
\includegraphics[trim = 10mm 3mm 10mm 3mm,clip, width = 0.255\textwidth]{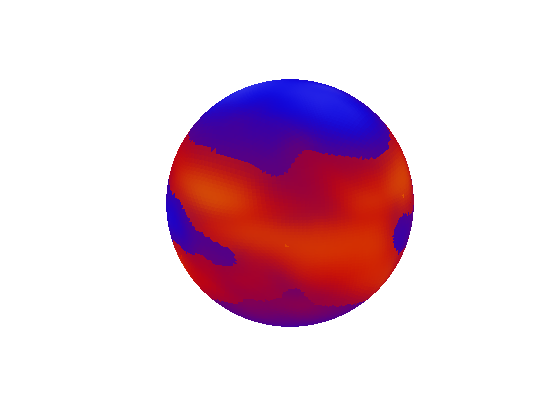}\\[0ex]
\includegraphics[trim = 10mm 2mm 10mm 2mm,clip, width = 0.25 \textwidth]{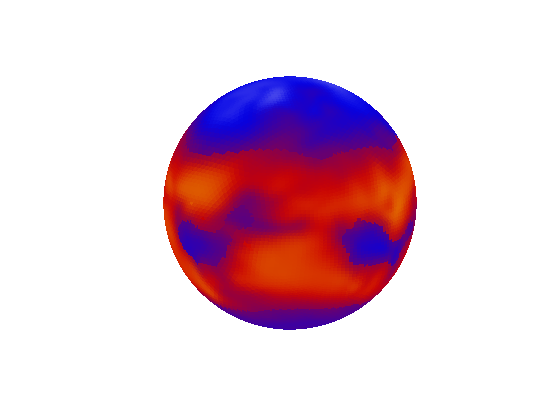}
\includegraphics[trim = 10mm 0mm 10mm 0mm,clip, width = 0.235\textwidth]{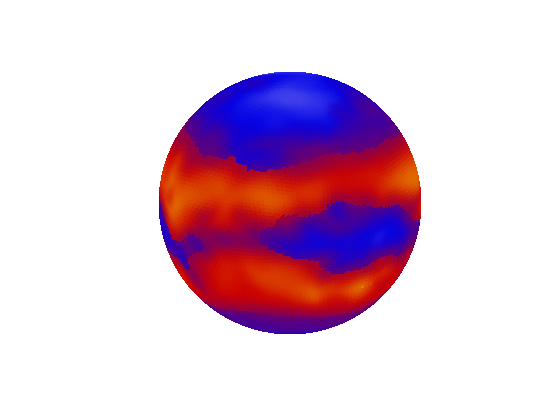}
\includegraphics[trim = 10mm 3mm 10mm 3mm,clip, width = 0.255\textwidth]{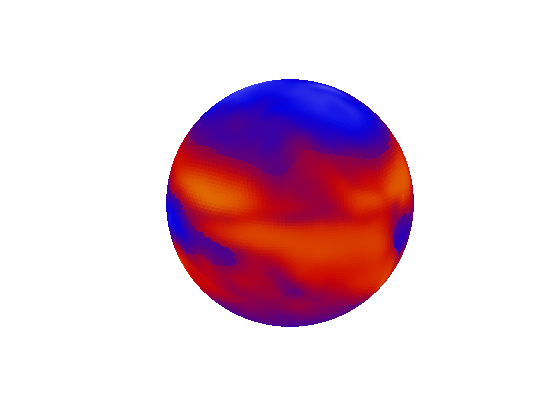}\\[0ex]
\includegraphics[trim = 10mm 2mm 10mm 2mm,clip, width = 0.25 \textwidth]{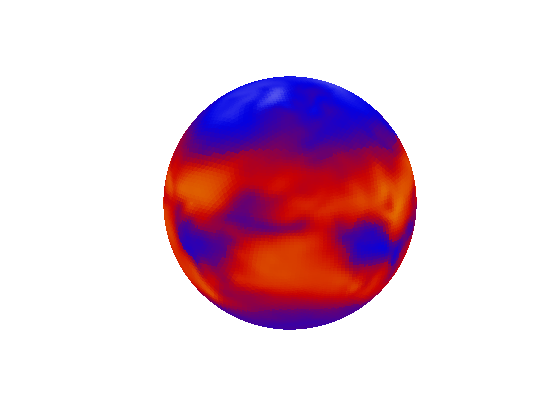}
\includegraphics[trim = 10mm 0mm 10mm 0mm,clip, width = 0.235\textwidth]{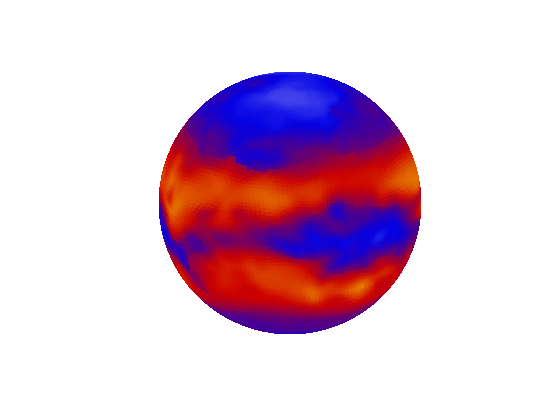}
\includegraphics[trim = 10mm 3mm 10mm 3mm,clip, width = 0.255\textwidth]{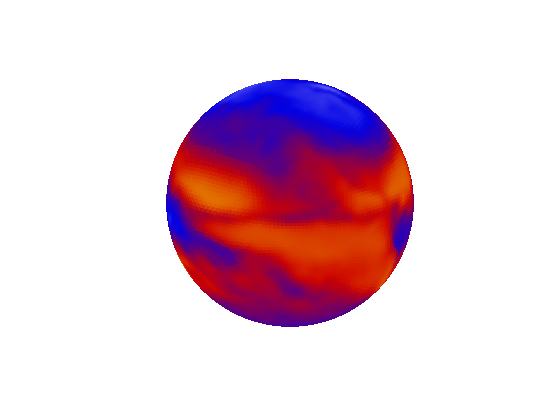}
\caption{Denoising of the longwave radiation data using the detected regions from time step $m=M=150$. First row: $\lambda = 100$. Second row: $\lambda = 1000$. Third row: $\lambda = 10000$. First - third column: Different viewing angles. The original image is from the NASA Earth Observation data set, \cite{NEO}.}
\label{fig:earth_longwave_radiation3}
\end{figure}

For demonstration of multiphase image segmentation, we consider a second example image from the NASA Earth Observation data set, cf. \cite{NEO}. We now segment an image showing the Earth's net radiation.
The net radiation is defined as the difference between the amount of solar energy which enters the Earth system and the amount of heat energy which escapes into space during one month (here: March 2014). Red color represents a net radiation around $280 \,\mathrm{Wm}^{-2}$, yellow color a net radiation around  $0 \,\mathrm{Wm}^{-2}$ and blue-green color a net radiation of  $-280 \,\mathrm{Wm}^{-2}$. 

Figure \ref{fig:earth_net_radiation1} presents the original image with the contours at different time steps (rows), observed from different viewing angles (columns). For the segmentation we used $\sigma = 1$ and $\lambda_1 = \lambda_2 = \lambda_3 = \lambda = 300$ and the RGB color space. As time step size we set $\Delta t = 0.002$. The detected regions are not separated by sharp image edges. Here, the detected boundaries are weak edges, i.e. they lie at locations in the image where the color smoothly changes from yellow to orange or yellow to green, respectively. At $m=71$ a merging and at $m=149$ a splitting occurs. Figure \ref{fig:earth_net_radiation2} shows the corresponding piecewise constant approximation.

\begin{figure}
\centering
\includegraphics[trim = 10mm 0mm 10mm 0mm,clip, width = 0.25\textwidth]{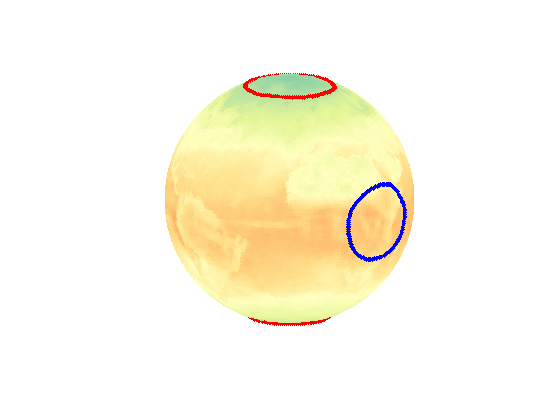}
\includegraphics[trim = 10mm 0mm 10mm 0mm,clip, width = 0.25\textwidth]{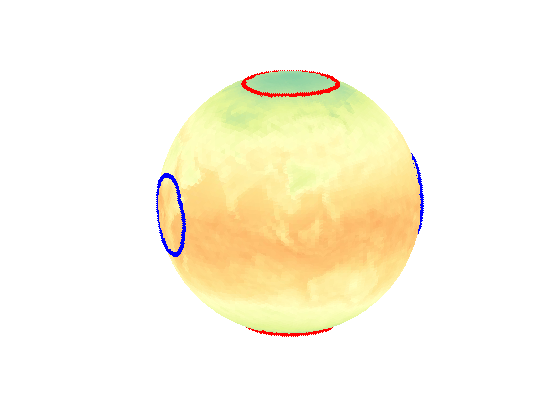}
\includegraphics[trim = 15mm 8mm 15mm 8mm,clip, width = 0.27\textwidth]{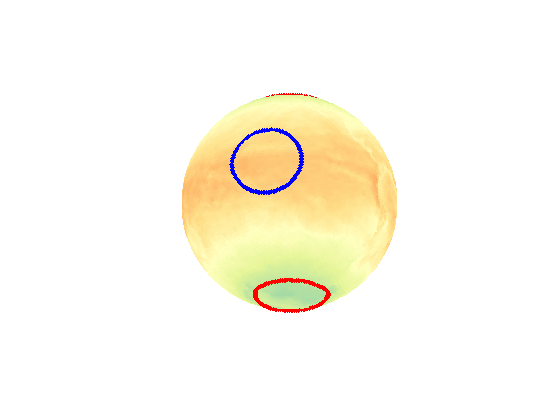}\\[0ex]
\includegraphics[trim = 10mm 0mm 10mm 0mm,clip, width = 0.25\textwidth]{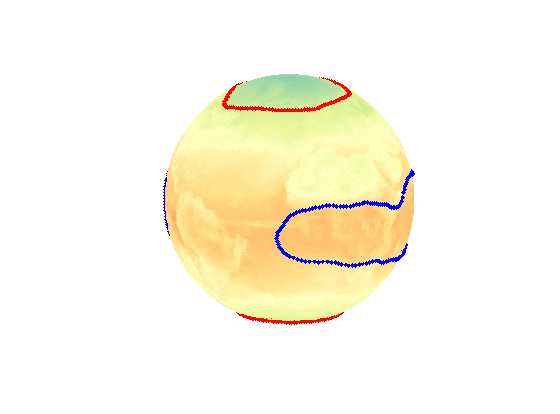}
\includegraphics[trim = 10mm 0mm 10mm 0mm,clip, width = 0.25\textwidth]{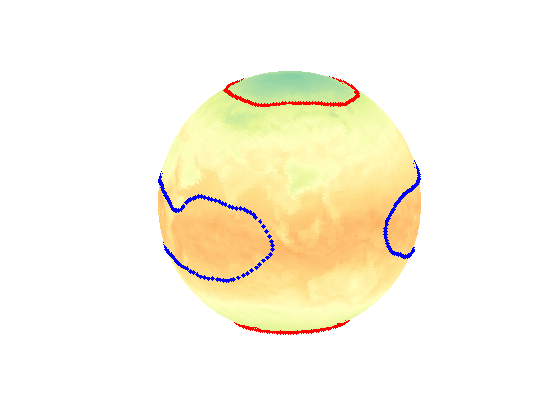}
\includegraphics[trim = 15mm 8mm 15mm 8mm,clip, width = 0.27\textwidth]{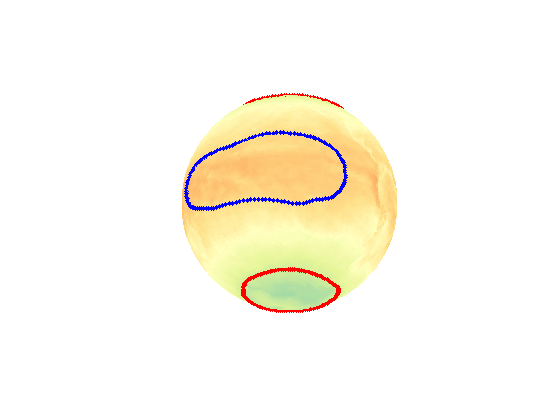}\\[0ex]
\includegraphics[trim = 10mm 0mm 10mm 0mm,clip, width = 0.25\textwidth]{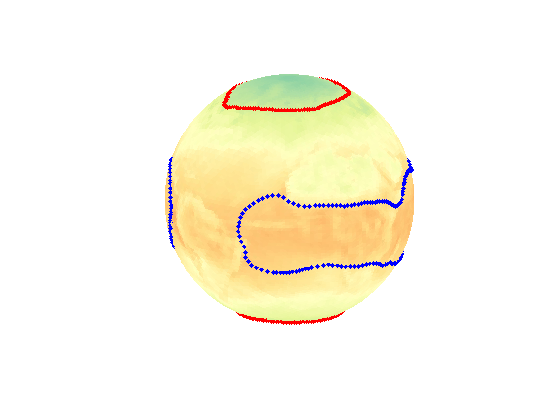}
\includegraphics[trim = 10mm 0mm 10mm 0mm,clip, width = 0.25\textwidth]{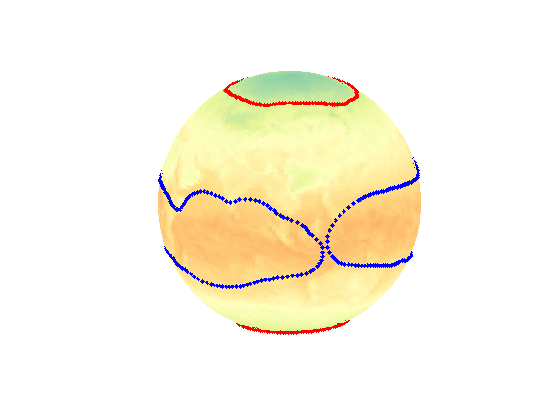}
\includegraphics[trim = 15mm 8mm 15mm 8mm,clip, width = 0.27\textwidth]{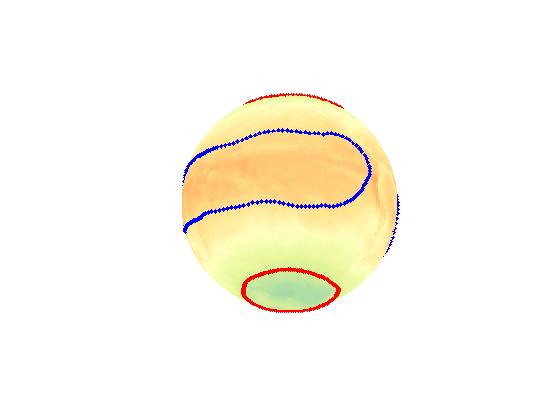}\\[0ex]
\includegraphics[trim = 10mm 0mm 10mm 0mm,clip, width = 0.25\textwidth]{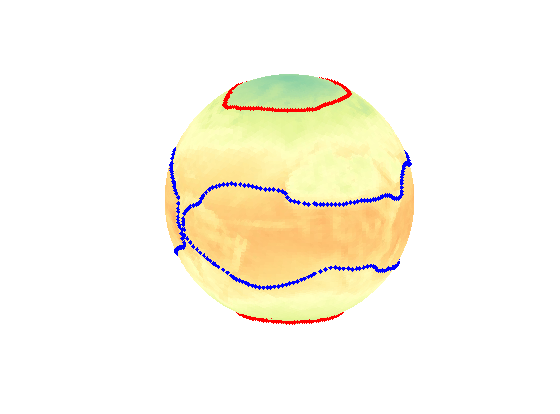}
\includegraphics[trim = 10mm 0mm 10mm 0mm,clip, width = 0.25\textwidth]{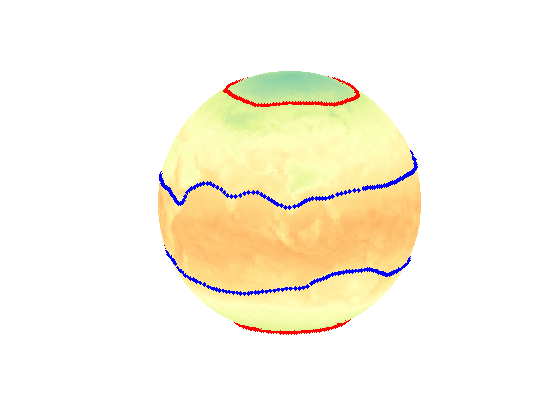}
\includegraphics[trim = 15mm 8mm 15mm 8mm,clip, width = 0.27\textwidth]{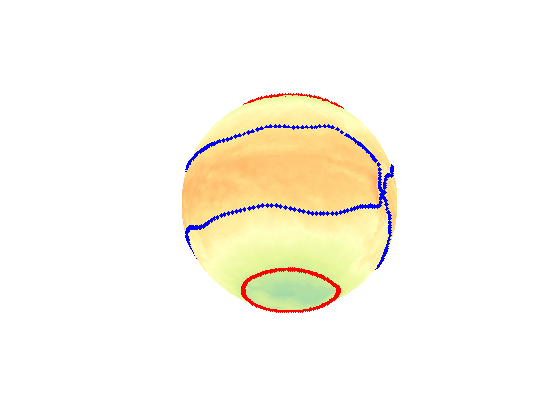}\\[0ex]
\includegraphics[trim = 10mm 0mm 10mm 0mm,clip, width = 0.25\textwidth]{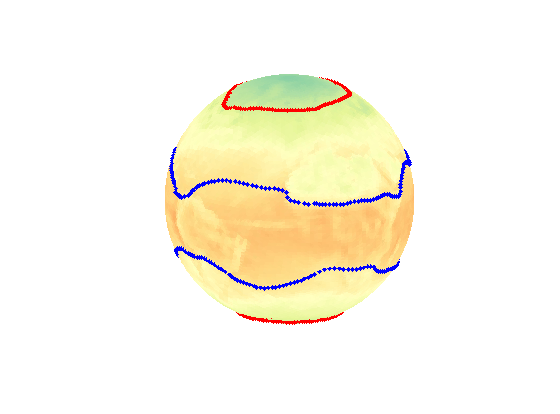}
\includegraphics[trim = 10mm 0mm 10mm 0mm,clip, width = 0.25\textwidth]{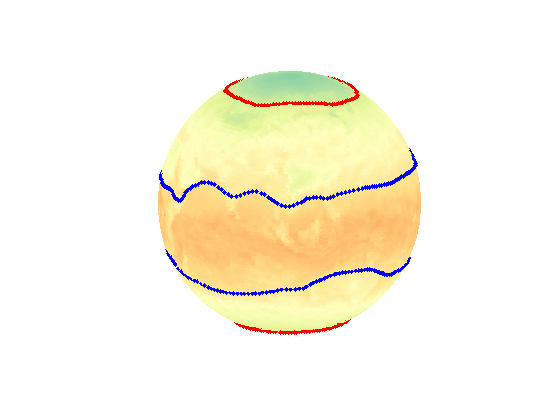}
\includegraphics[trim = 15mm 8mm 15mm 8mm,clip, width = 0.27\textwidth]{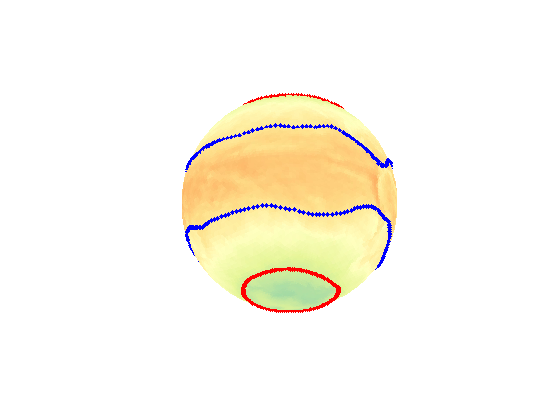}
\caption{Segmentation of net radiation data given on the Earth's surface. First - fifth row: Original image  and contours for $m=1, 50, 71, 149, 180$. First - third column: Different viewing angles. The original image is from the NASA Earth Observation data set, \cite{NEO}.}
\label{fig:earth_net_radiation1}
\end{figure}

\begin{figure}
\centering
\includegraphics[trim = 10mm 0mm 10mm 0mm,clip, width = 0.25\textwidth]{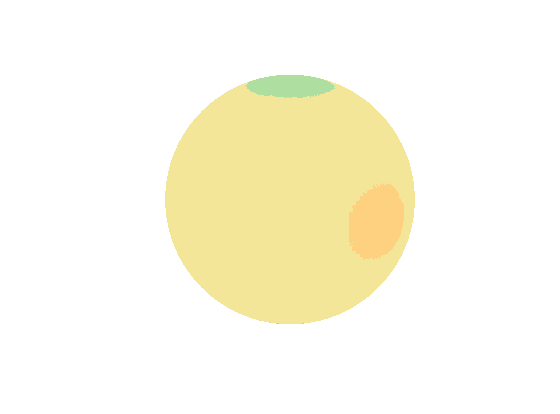}
\includegraphics[trim = 10mm 0mm 10mm 0mm,clip, width = 0.25\textwidth]{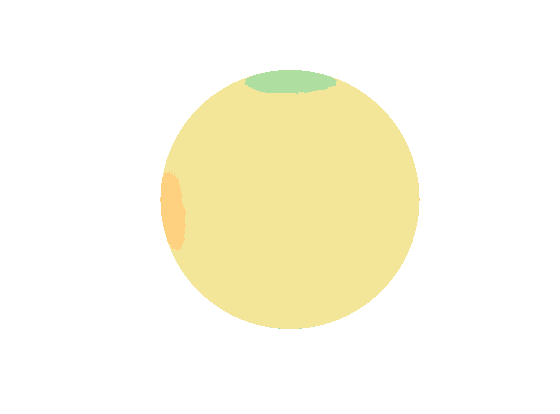}
\includegraphics[trim = 15mm 8mm 15mm 8mm,clip, width = 0.27\textwidth]{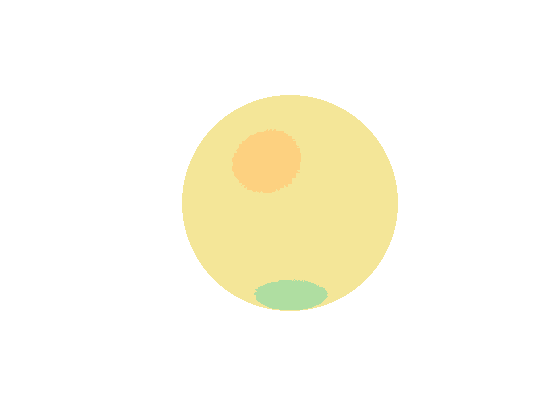}\\[0ex]
\includegraphics[trim = 10mm 0mm 10mm 0mm,clip, width = 0.25\textwidth]{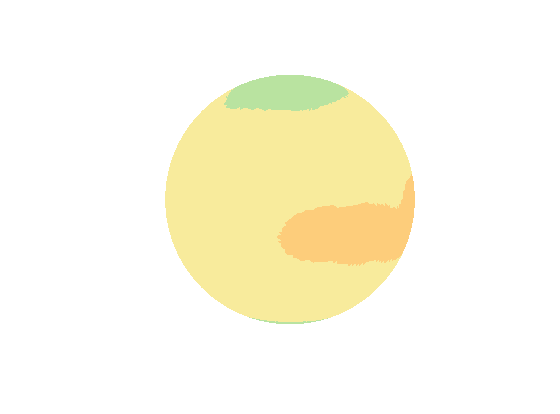}
\includegraphics[trim = 10mm 0mm 10mm 0mm,clip, width = 0.25\textwidth]{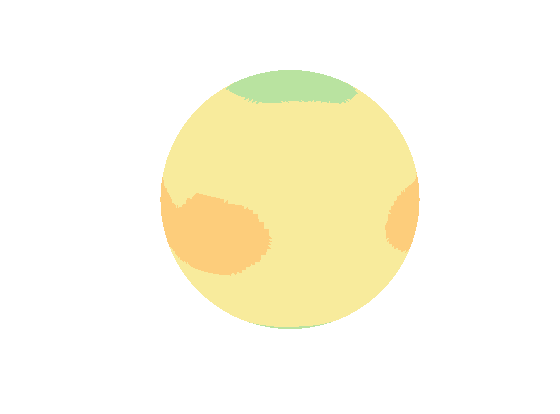}
\includegraphics[trim = 15mm 8mm 15mm 8mm,clip, width = 0.27\textwidth]{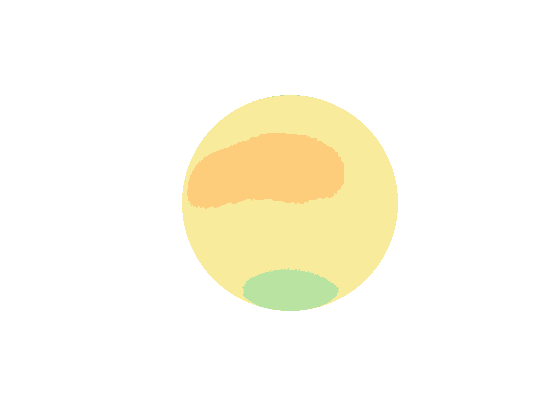}\\[0ex]
\includegraphics[trim = 10mm 0mm 10mm 0mm,clip, width = 0.25\textwidth]{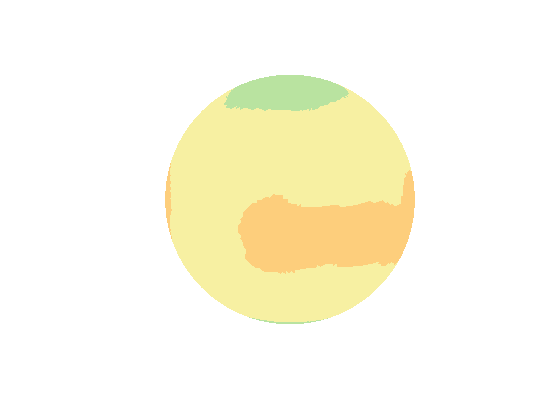}
\includegraphics[trim = 10mm 0mm 10mm 0mm,clip, width = 0.25\textwidth]{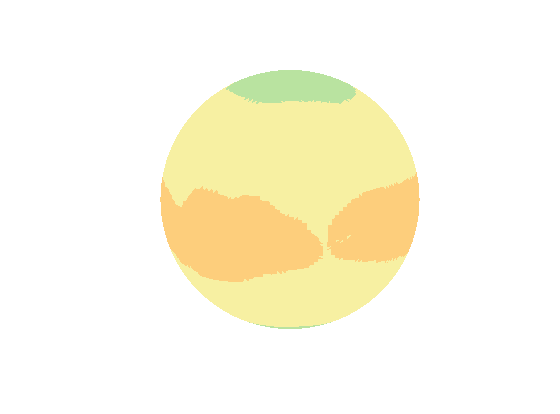}
\includegraphics[trim = 15mm 8mm 15mm 8mm,clip, width = 0.27\textwidth]{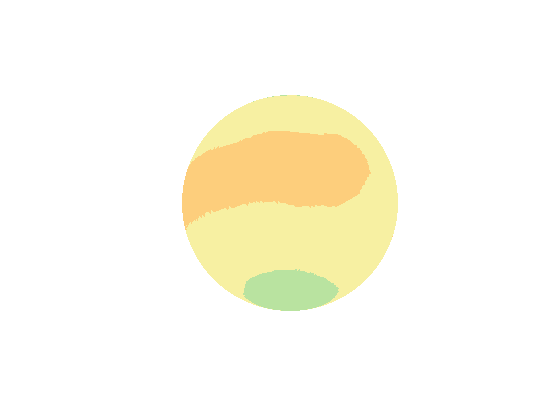}\\[0ex]
\includegraphics[trim = 10mm 0mm 10mm 0mm,clip, width = 0.25\textwidth]{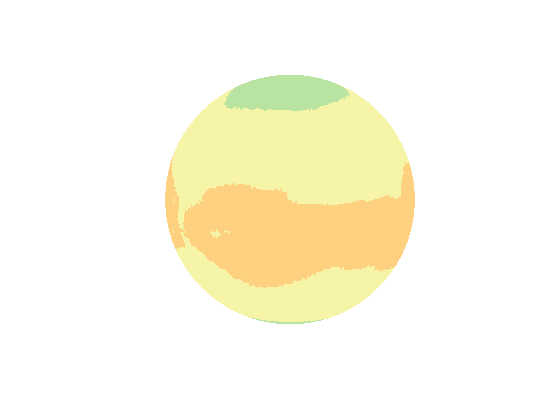}
\includegraphics[trim = 10mm 0mm 10mm 0mm,clip, width = 0.25\textwidth]{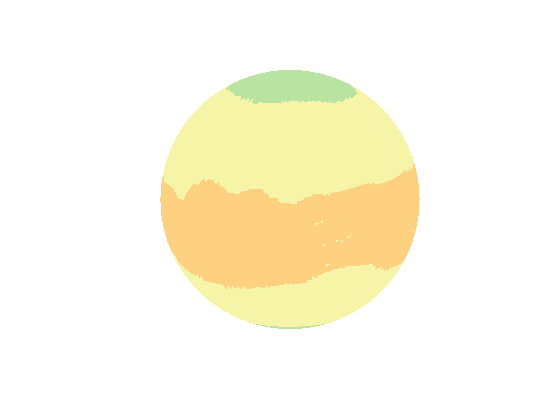}
\includegraphics[trim = 15mm 8mm 15mm 8mm,clip, width = 0.27\textwidth]{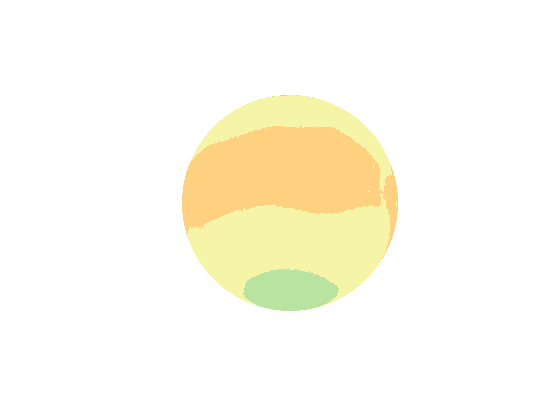}\\[0ex]
\includegraphics[trim = 10mm 0mm 10mm 0mm,clip, width = 0.25\textwidth]{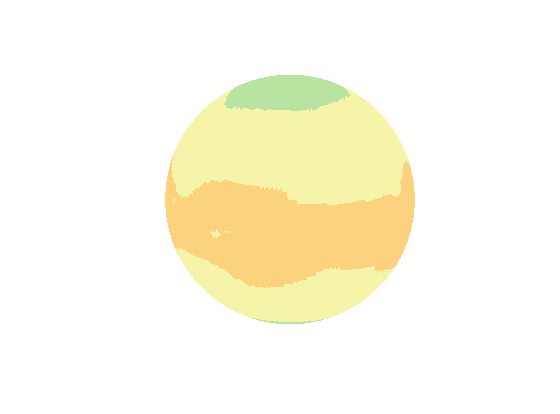}
\includegraphics[trim = 10mm 0mm 10mm 0mm,clip, width = 0.25\textwidth]{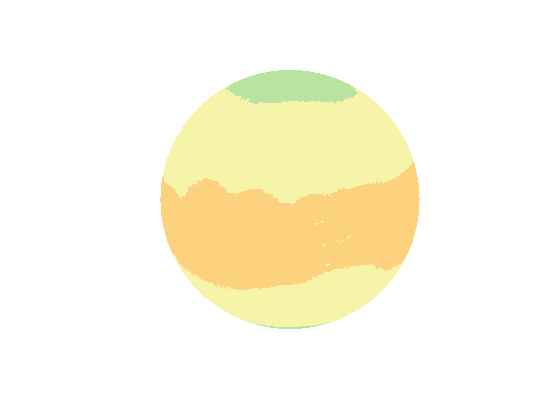}
\includegraphics[trim = 15mm 8mm 15mm 8mm,clip, width = 0.27\textwidth]{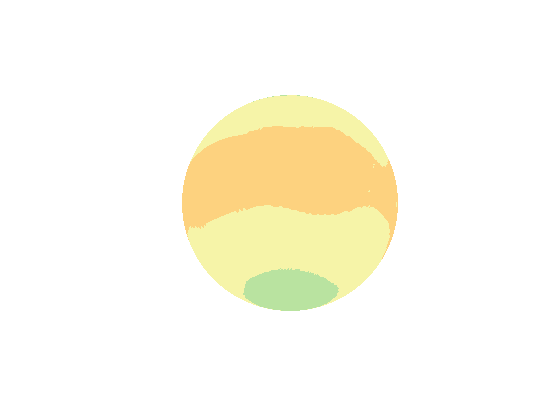}
\caption{Segmentation of net radiation data given on the Earth's surface. First - fifth row: Piecewise constant approximation for $m=1, 50, 71, 149, 180$. First - third column: Different viewing angles. The original image is from the NASA Earth Observation data set, \cite{NEO}.}
\label{fig:earth_net_radiation2}
\end{figure}

\begin{figure}
\centering
\includegraphics[trim = 10mm 0mm 10mm 0mm,clip, width = 0.25\textwidth]{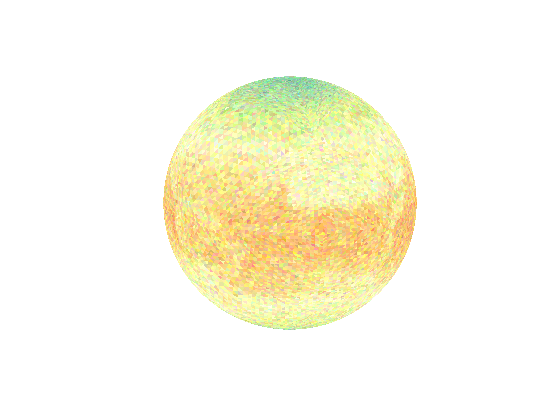}
\includegraphics[trim = 10mm 0mm 10mm 0mm,clip, width = 0.25\textwidth]{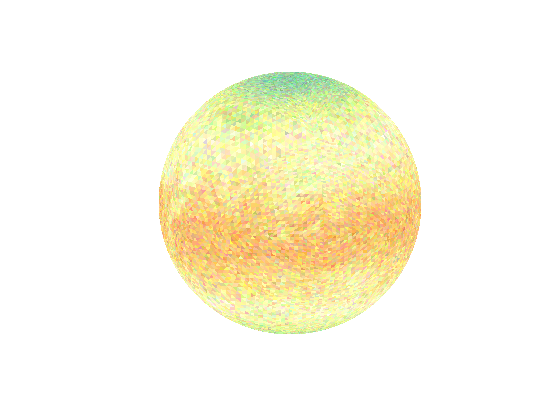}  
\includegraphics[trim = 15mm 8mm 15mm 8mm,clip, width = 0.27\textwidth]{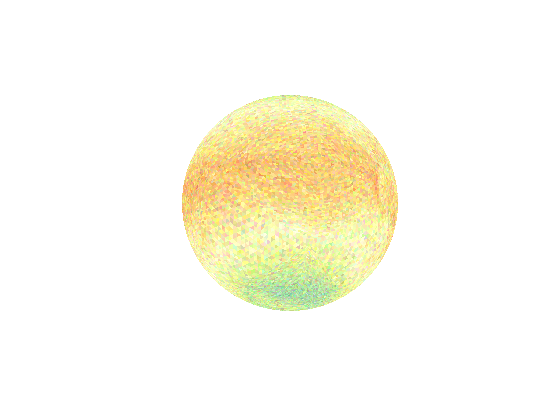}\\[0ex]
\includegraphics[trim = 10mm 0mm 10mm 0mm,clip, width = 0.25\textwidth]{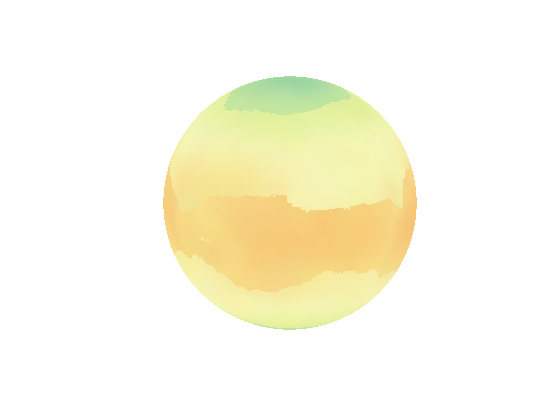}
\includegraphics[trim = 10mm 0mm 10mm 0mm,clip, width = 0.25\textwidth]{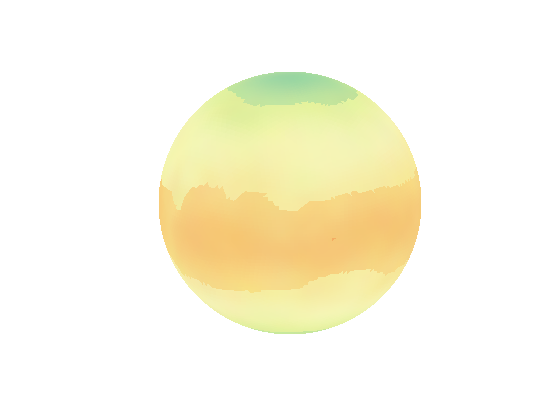}
\includegraphics[trim = 15mm 8mm 15mm 8mm,clip, width = 0.27\textwidth]{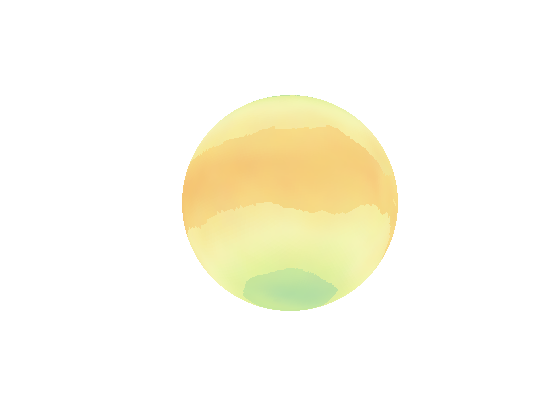}\\[0ex]
\includegraphics[trim = 10mm 0mm 10mm 0mm,clip, width = 0.25\textwidth]{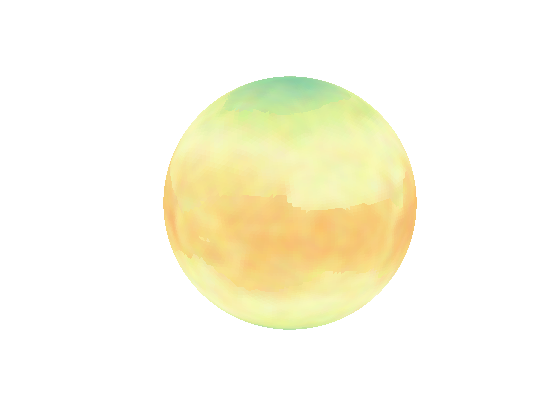}
\includegraphics[trim = 10mm 0mm 10mm 0mm,clip, width = 0.25\textwidth]{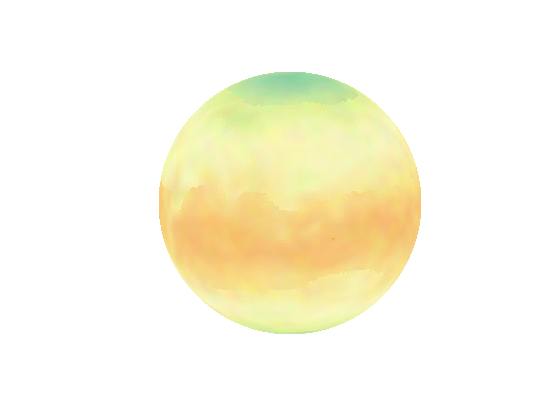}
\includegraphics[trim = 15mm 8mm 15mm 8mm,clip, width = 0.27\textwidth]{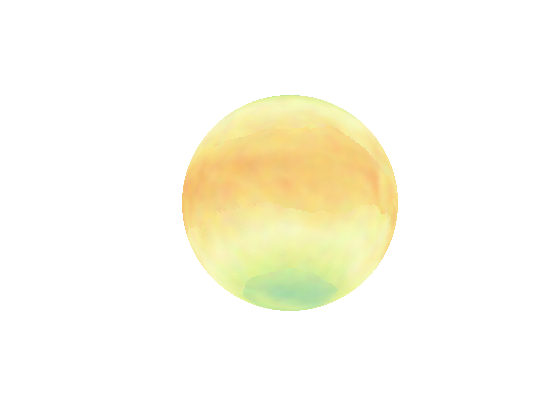}\\[0ex]
\includegraphics[trim = 10mm 0mm 10mm 0mm,clip, width = 0.25\textwidth]{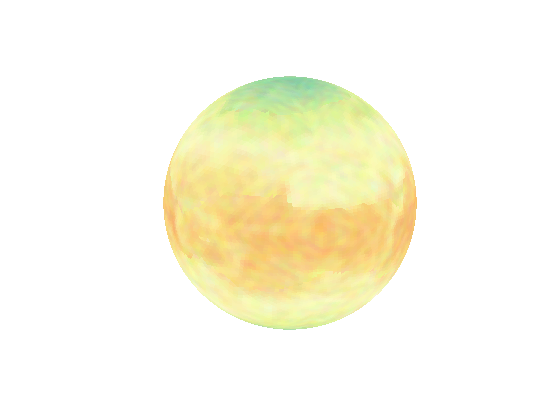}
\includegraphics[trim = 10mm 0mm 10mm 0mm,clip, width = 0.25\textwidth]{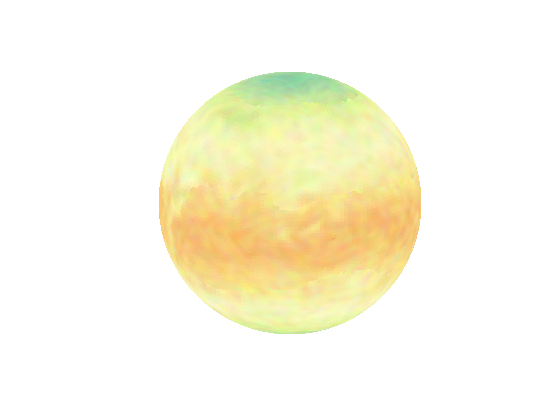}
\includegraphics[trim = 15mm 8mm 15mm 8mm,clip, width = 0.27\textwidth]{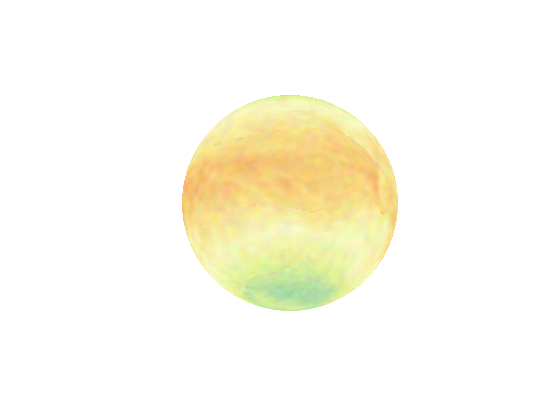}
\caption{Denoising of the net radiation data (image with added noise). First row: Noise added image to be smoothed. Second row: Smoothing result using $\lambda = 100$. Third row: $\lambda = 1000$. Forth row: $\lambda = 10000$.  First - third column: Different viewing angles. The original image is from the NASA Earth Observation data set, \cite{NEO}.}
\label{fig:earth_net_radiation_noise2}
\end{figure}

\begin{figure}
\centering
\includegraphics[width = 0.2\textwidth]{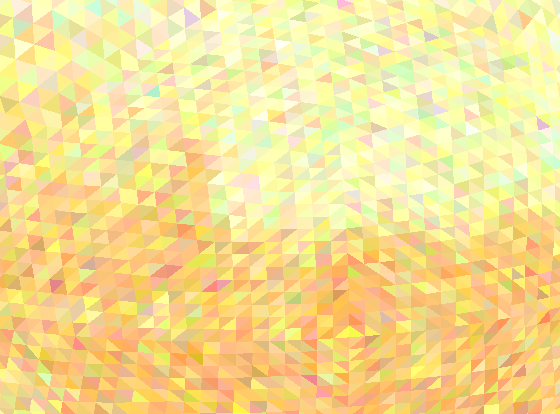} \hspace{1ex}
\includegraphics[width = 0.2\textwidth]{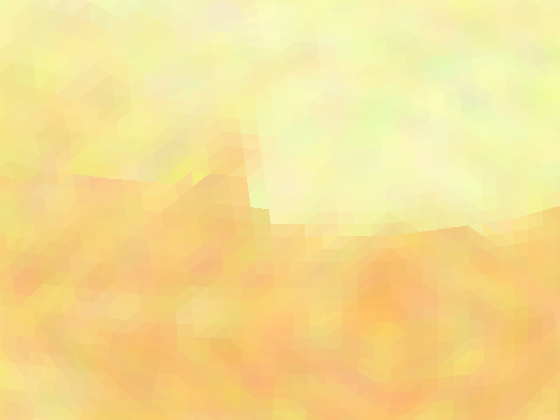}
\caption{Magnification of a part of the image. Left: noisy, original image. Right: smoothed version with $\lambda = 10000$.}
\label{fig:earth_net_radiation_noise3}
\end{figure}

Since the original image has little noise, we add some artificial, Gaussian noise to the image and repeat the segmentation using the generated image. As postprocessing step, we smooth the noisy image applying our restoration scheme, cf. Section \ref{subsec:finitedifference_imagedenoising_geodesic}. Figure \ref{fig:earth_net_radiation_noise2} shows the image with added noise and the denoised image for different parameters $\lambda$ under different viewing angles. We compare the results for several choices of $\lambda$: Using $\lambda=100$, the denoised image is very close to the piecewise constant image. For $\lambda=1000$ and $\lambda=10000$, we obtain images with removed noise, but which still contain sufficient details of the original data set. The data is not smoothed out too strong compared to $\lambda=100$. Figure \ref{fig:earth_net_radiation_noise3} shows a magnification of a part of the surface. It shows the noisy image (left) and the result of the denoising using $\lambda = 10000$ (right). We observe that the noise is well smoothed out. 

We have shown how the developed method can be applied on segmentation of global Earth data. The given data need not be a classical image generated by a camera. The data can be any data defined on the Earth's surface, like radiation data as in the examples presented here. 

One may argue that global Earth data can also be processed on a flat 2D domain, with a rectangular grid given by a discrete set of longitudes and latitudes. In principle, one can apply a two-dimensional image segmentation and restoration method as developed in \cite{Benninghoff2014a} for 2D images. However, performing the image processing using the 2D image has some disadvantages: To map the global Earth data to a rectangular 2D image, image boundaries at the poles and at longitude 0 have to be created. Points which are close to each other on the sphere (like at the two opposite sides of longitude 0) are not close to each other in the 2D image. Also topology changes like boundary intersection will occur in the 2D image. The coordinates of the boundary nodes on the left and the right boundary of the image may not fit, i.e. they lie on longitude 0, but can have different latitude values resulting in two different 3D points. Further, the length of curves and the area of regions near the poles are differently scaled compared to those near the equator. Polar regions always appear larger in 2D images. After a segmentation is performed, the size of the segmented regions are often of interest. Therefore, the area of the regions can be easily computed in a postprocessing step using the sphere data. This is not directly possible from the 2D image due to the different scaling of the polar and equatorial regions.
In summary,  it is beneficial to consider global Earth data directly as images on a surface.   

%% file: conclusion.tex
\section{Conclusion}
We presented how images on a surface can be efficiently segmented by curve evolution with a parametric approach. Furthermore, we showed how restoration with edge enhancement can be performed as a postprocessing step. We considered extensions of the Mumford-Shah \cite{Mumford89} and Chan-Vese \cite{Chan01} models to images on surfaces. The velocity of the parameterized curves was restricted to the tangent space of the surface, which guarantees that all curves  modeled as parametric curves in $\mathbb{R}^3$ stay on the surface during the evolution. We introduced an efficient numerical scheme based on a method for geodesic curvature flow \cite{BGN10}. Topology changes can be detected fast with an effort of $\mathcal{O}(N)$, where $N$ is the number of node points of the discretized curves. 
The applicability of the developed schemes on different images and different surfaces has been demonstrated in several experiments. 

%% file: Paper_ImagesOnSurfaces.bbl
\begin{thebibliography}{10}

\bibitem{Balazovjech12}
M.~Bala\v{z}ovjech, K.~Mikula, M.~Petr\'{a}\v{s}ov\'{a}, and J.~Urb\'{a}n.
\newblock {Lagrangean method with topological changes for numerical modelling
  of forest fire propagation}.
\newblock In {\em {Proceedings of ALGORITMY 2012, 19th Conference on Scientific
  Computing}}, pages 42--52, Vysok\'{e} Tatry, Podbansk'{v}, Slovakia, 2012.

\bibitem{BGN07a}
J.~W. Barrett, H.~Garcke, and R.~N\"urnberg.
\newblock {On the variational approximation of combined second and fourth order
  geometric evolution equations}.
\newblock {\em {SIAM} J. Sci. Comput.}, 29(3):1006--1041, 2007.

\bibitem{BGN10}
J.~W. Barrett, H.~Garcke, and R.~N\"urnberg.
\newblock {Numerical Approximation of Gradient Flows for Closed Curves in
  $\mathbb{R}^d$}.
\newblock {\em IMA J. Numer. Anal.}, 30(1):4--60, 2010.

\bibitem{Benninghoff2014a}
H.~Benninghoff and H.~Garcke.
\newblock {Efficient Image Segmentation and Restoration Using Parametric Curve
  Evolution With Junctions and Topology Changes}.
\newblock {\em SIAM J. Imaging Sci.}, 7(3):1451--1483, 2014.

\bibitem{Bertalmio01}
M.~Bertalmio, L.-T. Cheng, S.~Osher, and G.~Sapiro.
\newblock {Variational Problems and Partial Differential Equations on Implicit
  Surfaces: The Framework and Examples in Image Processing and Pattern
  Formation}.
\newblock {\em J. Comput. Phys.}, 174(2):759--780, 2001.

\bibitem{Caselles97}
V.~Caselles, R.~Kimmel, and G.~Sapiro.
\newblock {Geodesic Active Contours}.
\newblock {\em Int. J. Comput. Vision}, 22(1):61--79, 1997.

\bibitem{Chan06}
T.~F. Chan, S.~Esedoglu, and M.~Nikolova.
\newblock {Algorithms for Finding Global Minimizers of Image Segmentation and
  Denoising Models}.
\newblock {\em SIAM J. Appl. Math.}, 66:1632--1648, 2006.

\bibitem{Chan00}
T.~F. Chan, B.~Y. Sandberg, and L.~A. Vese.
\newblock {Active Contours without Edges for Vector-Valued Images}.
\newblock {\em J. Vis. Commun. Image R.}, 11(2):130--141, 2000.

\bibitem{Chan01}
T.~F. Chan and L.~A. Vese.
\newblock {Active Contours Without Edges}.
\newblock {\em IEEE Trans. Image Process.}, 10(2):266--277, 2001.

\bibitem{Cheng02}
L.-T. Cheng, P.~Burchard, B.~Merriman, and S.~Osher.
\newblock {Motion of Curves Constrained on Surfaces Using a Level-Set
  Approach}.
\newblock {\em J. Comput. Phys.}, 175(2):604--644, 2002.

\bibitem{Cunderlik13}
R.~{\v C}underl\'{i}k, K.~Mikula, and M.~Tunega.
\newblock {Nonlinear diffusion filtering of data on the Earth's surface}.
\newblock {\em J. Geodesy}, 87(2):143--160, 2013.

\bibitem{Davis04}
T.A. Davis.
\newblock {Algorithm 832: UMFPACK, an unsymmetric-pattern multifrontal method}.
\newblock {\em ACM Trans. Math. Software}, 30(2):196--199, 2004.

\bibitem{Dziuk13}
G.~Dziuk and C.~M. Elliott.
\newblock {Finite Element Methods for Surface PDEs}.
\newblock {\em Acta Numer.}, 22:289--396, 2013.

\bibitem{GarckeWieland06}
H.~Garcke and S.~Wieland.
\newblock {Surfactant Spreading on Thin Viscous Films: Nonnegative Solutions of
  a Coupled Degenerate System}.
\newblock {\em {SIAM} J. Math. Anal.}, 37(6):2025--2048, 2006.

\bibitem{Kass88}
M.~Kass, A.~Witkin, and D.~Terzopoulos.
\newblock {Snakes: Active Contour Models}.
\newblock {\em Int. J. Comput. Vision}, 1(4):321--331, 1988.

\bibitem{Kimmel97}
R.~Kimmel.
\newblock {Intrinsic Scale Space for Images on Surfaces: The Geodesic Curvature
  Flow}.
\newblock {\em Graph. Model. Im. Proc.}, 59(5):365--372, 1997.

\bibitem{Krueger08}
M.~Kr\"uger, P.~Delmas, and G.~Gimel’farb.
\newblock {Active contour based segmentation of 3D surfaces}.
\newblock In {\em {Proceedings of the European Conference on Computer Vision}},
  pages 350--363, Marseille, France, 2008.

\bibitem{Lai11}
R.~Lai and T.~F. Chan.
\newblock {A Framework for Intrinsic Image Processing on Surfaces}.
\newblock {\em Comput. Vis. Image Und.}, 115(12):1647--1661, 2011.

\bibitem{Lang02}
S.~Lang.
\newblock {\em {Introduction to Differentiable Manifolds}}.
\newblock Universitext. Springer, New York, Berlin, Heidelberg, 2nd edition,
  2002.

\bibitem{Lee02}
J.~M. Lee.
\newblock {\em {Introduction to Smooth Manifolds}}, volume 218 of {\em Graduate
  Texts in Mathematics}.
\newblock Springer, New York, Heidelberg, Dordrecht, London, 2002.

\bibitem{MikulaUrban12}
K.~Mikula and J.~Urb\'{a}n.
\newblock {New fast and stable Lagrangean method for image segmentation}.
\newblock In {\em {Proceedings of the 5th International Congress on Image and
  Signal Processing (CISP 2012)}}, pages 834--842, Chongquing, China, 2012.

\bibitem{Mikula06}
K.~Mikula and D.~\v{S}ev\v{c}ovi\v{c}.
\newblock {Evolution of Curves on a Surface Driven by the Geodesic Curvature
  and External Force}.
\newblock {\em Appl. Anal.}, 85(4):345--362, 2006.

\bibitem{Mumford89}
D.~Mumford and J.~Shah.
\newblock {Optimal Approximation by Piecewise Smooth Functions and Associated
  Variational Problems}.
\newblock {\em Commun. Pur. Appl. Math.}, 42:577--685, 1989.

\bibitem{NEO}
{NASA}.
\newblock {NASA Earth Observations}, 2 2014.

\bibitem{OsherSethian88}
S.~Osher and J.~A. Sethian.
\newblock {Fronts Propagating with Curvature Dependent Speed: Algorithms Based
  on Hamilton-Jacobi Formulations}.
\newblock {\em J. Comput. Phys.}, 79(1):12--49, 1988.

\bibitem{Paysen09}
P.~Paysan, R.~Knothe, B.~Amberg, S.~Romdhani, and T.~Vetter.
\newblock {A 3D Face Model for Pose and Illumination Invariant Face
  Recognition}.
\newblock In {\em {Proceedings of the 6th IEEE International Conference on
  Advanced Video and Signal Based Surveillance (AVSS) for Security, Safety and
  Monitoring in Smart Environments}}, pages 296--301, Genova, Italy, 2009.
  IEEE.

\bibitem{Rudin1992}
L.~I. Rudin, S.~Osher, and E.~Fatemi.
\newblock {Nonlinear Total Variation Based Noise Removal Algorithms}.
\newblock {\em Physica D}, 60(1-4):259--268, 1992.

\bibitem{Spira07}
A.~Spira and R.~Kimmel.
\newblock {Geometric Curve Flows on Parametric Manifolds}.
\newblock {\em J. Comput. Phys.}, 223:235--249, 2007.

\bibitem{Tian09}
L.~Tian, C.~B. Macdonald, and S.~J. Ruuth.
\newblock {Segmentation on surfaces with the closest point method}.
\newblock In {\em {Proceedings of the 16th IEEE International Conference on
  Image Processing}}, pages 3009--3012, Cairo, Egypt, 2009.

\bibitem{Turk94}
G.~Turk and M.~Levoy.
\newblock {Zippered polygon meshes from range images}.
\newblock In {\em {Proceedings of the 21st annual conference on Computer
  graphics and interactive techniques (SIGGRAPH '94)}}, pages 311--318, New
  York, NY, USA, 1994. ACM.

\bibitem{Wu2010}
C.~Wu and X.-C. Tai.
\newblock {Augmented Lagrangian Method Dual Methods, and Split Bregman
  Iteration for ROF, Vectorial TV, and High Order Models}.
\newblock {\em SIAM J. Imaging Sci.}, 3(3):300--339, 2010.

\bibitem{Wu2012}
C.~Wu, J.~Zhang, Y.~Duan, and X.-C. Tai.
\newblock {Augmented Lagrangian Method for Total Variation Based Image
  Restoration and Segmentation Over Triangulated Surfaces}.
\newblock {\em J. Sci. Comput.}, 50(1):145--166, 2012.

\bibitem{Zhou13}
L.~Zhou and J.~Li.
\newblock {Image Segmentation on Implicit Surface Based on Chan-Vese Model}.
\newblock {\em J. Theor. Appl. Inf. Technol.}, 48(1):206--209, 2013.

\end{thebibliography}
